%% file: colm2025_conference.tex
\documentclass{article} 
\PassOptionsToPackage{table}{xcolor}
\usepackage{colm2025_conference}

\usepackage{microtype}
\usepackage{hyperref}
\usepackage{url}
\usepackage{booktabs}
\usepackage{multirow}
\usepackage{tabularx}

\usepackage{amssymb} %
\usepackage{amsmath} %
\usepackage{amsfonts} %
\usepackage{amsthm} %

\input{math_commands}

\usepackage{lineno}
\usepackage{titlesec}
\usepackage[most]{tcolorbox}
\usepackage{caption}

\definecolor{darkblue}{rgb}{0, 0, 0.5}
\hypersetup{colorlinks=true, citecolor=darkblue, linkcolor=darkblue, urlcolor=darkblue}
\usepackage{graphicx}
\usepackage{subcaption}
\usepackage{colortbl}
\usepackage{amsmath}
\usepackage{siunitx}
\usepackage{pifont}    

\definecolor{richpurple}{RGB}{75,46,131}
\definecolor{beige}{RGB}{245,245,220}

\newcommand{\purplecomic}[1]{%
  {\color{richpurple}\selectfont #1}%
}
\newcommand\eat[1]{}

\newcommand{\blackcomic}[1]{%
  {\color{black}\selectfont #1}%
}
\renewcommand{\thefootnote}{\fnsymbol{footnote}}

\tcbuselibrary{theorems}
\newcounter{theoremcounter}
\newcounter{lemmacounter}
\newcounter{remarkcounter}
\newcounter{definitioncounter}

\newtcbtheorem[use counter=theoremcounter]{theorem}{Theorem}%
{colback=purple!5!white,colframe=richpurple!80!black,fonttitle=\bfseries}{thm}
\newtcbtheorem[use counter=lemmacounter]{lemma}{Lemma}%
{colback=purple!5!white,colframe=richpurple!80!black,fonttitle=\bfseries}{lem}
\newtcbtheorem[use counter=remarkcounter]{remark}{Remark}%
{colback=beige,colframe=brown!50!black,fonttitle=\bfseries}{rem}
\newtcbtheorem[use counter=definitioncounter]{definition}{Definition}%
{colback=purple!5!white,colframe=richpurple!80!black,fonttitle=\bfseries}{def}

\title{
  \purplecomic{\textbf{SeeUPO}}: 
  \blackcomic{Sequence-Level Agentic-RL with Convergence Guarantees}
}

\author{
Tianyi Hu,
~Qingxu Fu, 
~Yanxi Chen, 
~Zhaoyang Liu\footnotemark[2], 
~Bolin Ding\footnotemark[2] \\
[1em]
Tongyi Lab\includegraphics[height=12pt]{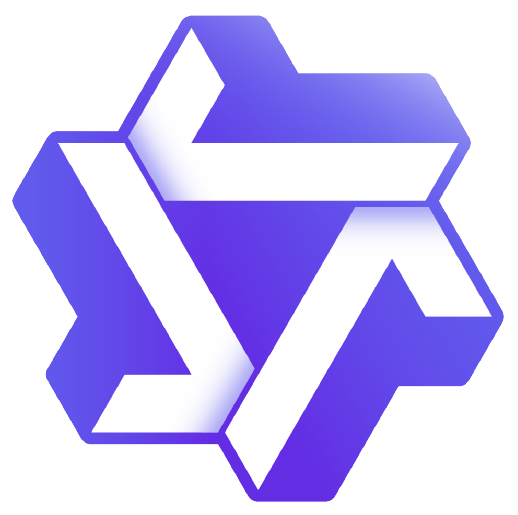}, Alibaba Group \\
[1em]
\texttt{\{wenju.hty, fuqingxu.fqx, chenyanxi.cyx, jingmu.lzy, bolin.ding\}@alibaba-inc.com}
}

\begin{document}


\maketitle

\begingroup
  \renewcommand\thefootnote{}%
  \footnotetext{%
    \begin{tabular}{@{}l@{\hspace{0.4em}}l@{}}
      $^{\dagger}$ & Corresponding authors.
    \end{tabular}%
  }%
\endgroup
\phantomsection

\begin{tcolorbox}[
  colback=blue!5!white,
  colframe=richpurple!80!black,
  boxrule=1.5pt,
  arc=3mm,
  left=4mm,
  right=4mm,
  top=3mm,
  bottom=3mm,
  fonttitle=\bfseries,
  title=\textbf{Abstract}
]

Reinforcement learning (RL) has emerged as the predominant paradigm for training large language model (LLM)-based AI agents. However, existing backbone RL algorithms lack verified convergence guarantees in agentic scenarios, especially in multi-turn settings, which can lead to training instability and failure to converge to optimal policies. 

\vspace{0.8em}
In this paper, we systematically analyze how different combinations of policy update mechanisms and advantage estimation methods affect convergence properties. We find that REINFORCE with Group Relative Advantage Estimation (GRAE) can converge to the globally optimal under undiscounted conditions, but the combination of PPO \& GRAE breaks PPO's original monotonic improvement property. Furthermore, we demonstrate that mainstream backbone RL algorithms cannot simultaneously achieve both critic-free and convergence guarantees in multi-turn scenarios. 

\vspace{0.8em}
To address this, we propose \textcolor{richpurple}{\textbf{SeeUPO}} (\textbf{\underline{Se}}quence-level S\textbf{\underline{e}}quential \textbf{\underline{U}}pdate \textbf{\underline{P}}olicy \textbf{\underline{O}}ptimization), a critic-free approach with convergence guarantees for multi-turn interactions. SeeUPO models multi-turn interaction as sequentially executed multi-agent bandit problems. Through turn-by-turn sequential policy updates in reverse execution order, it ensures monotonic improvement and convergence to global optimal solution via backward induction.

\vspace{0.8em}
Experiments on AppWorld and BFCL v4 demonstrate SeeUPO's substantial improvements over existing backbone algorithms: relative gains of 43.3\%--54.6\% on Qwen3-14B and 24.1\%--41.9\% on Qwen2.5-14B (averaged across benchmarks), along with superior training stability.

\end{tcolorbox}

\begin{figure}[htbp]
  \centering
  \vspace{-0.5em}
  \includegraphics[width=1.0\textwidth]{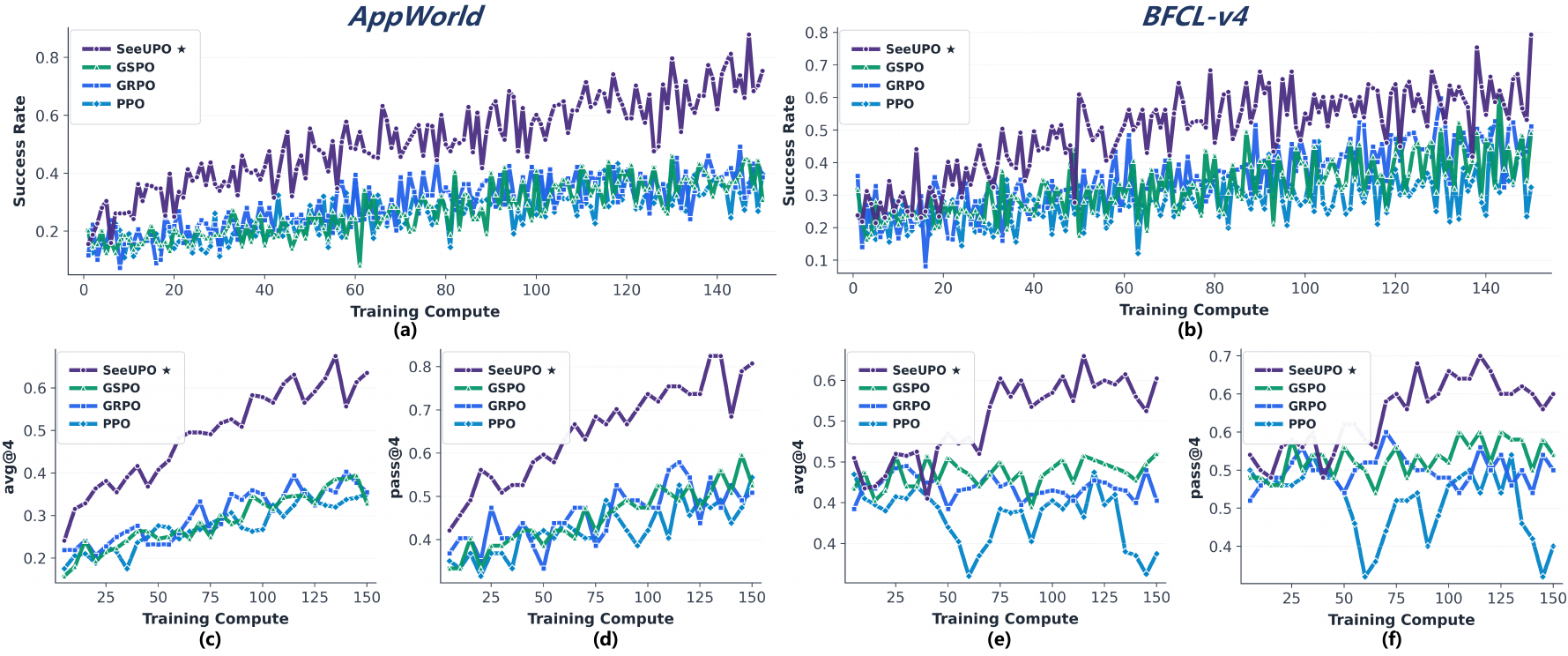}
  \captionsetup{width=0.95\textwidth}
  \caption{Performance comparison of training Qwen3-14B model on the AppWorld and BFCL-v4 benchmarks. (a)-(b) show training curves, (c)-(f) show test curves. SeeUPO algorithm demonstrates significantly stronger training stability and optimal performance compared to other backbone RL algorithms.}
  \label{fig:begin}
\end{figure}

\newpage

\input{sections/SeeUPO_Main/intro}

\input{sections/SeeUPO_Main/preliminaries}

\input{sections/SeeUPO_Main/analysis}

\input{sections/SeeUPO_Main/method}

\input{sections/SeeUPO_Main/experiments}

\input{sections/SeeUPO_Main/conclusion}




\bibliography{colm2025_conference}
\bibliographystyle{colm2025_conference}

\appendix

\input{sections/SeeUPO_Main/appendix}

\end{document}

%% file: math_commands.tex
\usepackage{amsmath,amsfonts,bm}

\def\eqref#1{equation~\ref{#1}}

\def\1{\bm{1}}

\DeclareMathAlphabet{\mathsfit}{\encodingdefault}{\sfdefault}{m}{sl}
\SetMathAlphabet{\mathsfit}{bold}{\encodingdefault}{\sfdefault}{bx}{n}

\DeclareMathOperator*{\argmax}{arg\,max}

%% file: sections/SeeUPO_Main/intro.tex
\section{Introduction}
\label{sec:intro}

The rapid advancement of large language models (LLMs)~\citep{liu2024deepseek, yang2025qwen3} has catalyzed the emergence of autonomous AI agents capable of executing complex tasks through tool use and multi-turn interactions with diverse environments. These AI systems have demonstrated remarkable capabilities across a wide spectrum of real-world applications, including web navigation~\citep{agentbench2023}, software development~\citep{trivedi2024appworld}, and interactive tool-augmented environments~\citep{patil2025bfcl}. The significant potential of agentic AI has motivated extensive research efforts, making the development of robust, scalable training methods a key research focus.

Within this landscape, reinforcement learning (RL) has emerged as the predominant paradigm for training LLM-based agents through interaction and feedback. A growing body of work has explored RL-driven agent training across various dimensions: \textbf{(interaction scope)} from single-turn task reasoning to multi-turn interactive planning~\citep{chai2025rlfactory, wang2025ragen, xi2025agentgym, jiang2025aligning}; \textbf{(capability integration)} from static capability modules to unified policy optimization, where RL transforms planning, tool use, and memory into interdependent, trainable policies~\citep{qian2504toolrl, yan2025memory}; \textbf{(multimodal integration)} from single-modal text generation to embodied and multimodal perception~\citep{qi2025vln, feng2025video, li2025videochat}; and \textbf{(evolutionary mechanism)} from heuristic-based self-correction to internalized self-improvement, allowing agents to iteratively refine knowledge and strategies through RL-driven self-evolution~\citep{wang2025ragen, guan2025rstar, zhai2025agentevolver}.

Notably, while these methods span diverse research directions, they converge on a common foundation: a small set of backbone RL algorithms and their variants~\citep{srivastava2025technical}. For instance, \textbf{Proximal Policy Optimization (PPO)}~\citep{PPO} serves as the foundational algorithm for actor-critic RL~\citep{ouyang2022training}, where the canonical PPO implementation estimates advantage functions through Generalized Advantage Estimation (GAE)~\citep{GAE} and critic networks, and updates policy networks via proximal policy optimization. \textbf{RLOO}~\citep{ahmadian2024back} achieves a critic-free advantage estimation approach by sampling multiple within-group responses for the same query to compute group relative advantages, and performs sequence-level policy updates via REINFORCE. \textbf{Group Relative Policy Optimization (GRPO)}~\citep{grpo2024} combines the Group Relative Advantage Estimation (GRAE) insight of RLOO with PPO-style policy updates, becoming a representative algorithm that combines PPO with critic-free approaches. \textbf{Group Sequence Policy Optimization (GSPO)}~\citep{gspo2024} further refines GRPO with sequence-level mechanisms, becoming a representative of sequence-level RL algorithms.

However, \textit{the theoretical soundness of these backbone algorithms in agentic RL scenarios, particularly in multi-turn settings, remains an open question}~\citep{zhanglandscape}. To address this question, we categorize mainstream backbone RL algorithms from the perspectives of advantage estimation and policy update mechanisms, and analyze the convergence properties of different combinations in both single-turn and multi-turn scenarios. Specifically, we examine two types of advantage estimation methods: \textbf{Generalized Advantage Estimation (GAE)} (critic-dependent), and \textbf{Group Relative Advantage Estimation (GRAE)} (critic-free). We also examine two types of policy update mechanisms: \textbf{REINFORCE} (fully on-policy, based on Vanilla Policy Gradient), and \textbf{Proximal Policy Update (PPU)\footnote{To avoid confusion, in this paper we use \textbf{PPO} to refer to the \textit{complete algorithm} in the LLM-RL-Training domain, and \textbf{PPU} to specifically refer to the proximal \textit{policy update method} characterized by importance sampling and clipping mechanisms.}} (partially on-policy, based on Importance Sampling correction). We prove that the combination of GRAE and REINFORCE converges to the globally optimal solution only under undiscounted settings. We discover that the combination of GRAE and PPU breaks PPO's original monotonic improvement property in most cases, with convergence guarantees only in contextual bandit scenarios. Through comprehensive analysis, we find that \textit{mainstream backbone RL algorithms cannot simultaneously enjoy both critic-free operation and good monotonic improvement / convergence guarantees in multi-turn scenarios}.

To address these limitations, we propose \textbf{SeeUPO} (\textbf{S}equence-level S\textbf{e}quential \textbf{U}pdate \textbf{P}olicy \textbf{O}ptimization), a novel algorithm that \textbf{provides convergence guarantees for multi-turn agentic RL} while \textbf{maintaining the critic-free approach}. SeeUPO models multi-turn interaction problems as sequentially executed multi-agent bandit problems, where each turn is abstracted as a virtual agent. The core mechanism of SeeUPO is \textbf{turn-by-turn sequential policy updates in reverse execution order} ($T \to T{-}1 \to \cdots \to 1$). This reverse-order sequential update mechanism enables \textbf{backward induction}, where each turn optimizes against the already-updated optimal policies of subsequent turns, thereby achieving global optimality. We prove that SeeUPO inherits \textbf{monotonic improvement guarantees} from the HAML framework~\citep{HARL}, and further establish that the \textbf{reverse update order guarantees convergence to the globally optimal policy} in multi-turn contextual bandit settings. Through implicit turn-level credit assignment via advantage function decomposition, SeeUPO correctly attributes each turn's contribution to the global return, achieving stable and theoretically sound training in multi-turn scenarios.

We evaluate SeeUPO on two multi-turn agentic benchmarks, AppWorld~\citep{trivedi2024appworld} and BFCL v4~\citep{patil2025bfcl}, demonstrating substantial improvements over existing backbone RL algorithms. On Qwen3-14B, SeeUPO achieves 60.80\% avg@4 and 72.85\% pass@4, with relative improvements ranging from 43.3\% to 54.6\% compared to baseline methods. On Qwen2.5-14B, SeeUPO achieves 53.07\% avg@4 and 63.59\% pass@4, with relative improvements ranging from 24.1\% to 41.9\% over baselines. SeeUPO maintains stable training without the catastrophic failures observed in baseline methods. Additional comparative experiments validate our theoretical insights: the reverse update order achieves the best performance, confirming our proof that backward induction enables convergence to global optimality; batch-level normalization preserves the drift functional properties required for monotonic improvement guarantees while providing numerical stability.

\paragraph{Contributions} Our main contributions are summarized as follows:
\begin{itemize}
    \item \textbf{Theoretical Analysis of Backbone RL Algorithms.} We provide a comprehensive convergence analysis of mainstream backbone RL algorithms by categorizing them along two dimensions: advantage estimation (GAE vs. GRAE) and policy update mechanisms (REINFORCE vs. PPU). We prove the convergence conditions and limitations of each combination in both single-turn and multi-turn scenarios, revealing a fundamental trade-off: existing algorithms cannot simultaneously achieve critic-free operation and convergence guarantees in multi-turn settings.
    \item \textbf{Novel Algorithm with Theoretical Guarantees.} We propose SeeUPO, a sequence-level sequential update policy optimization algorithm that resolves the above trade-off. By modeling multi-turn interactions as sequentially-executed multi-agent bandit problems, SeeUPO inherits monotonic improvement guarantees and achieves convergence to global optimality through reverse-order sequential updates via backward induction.
    \item \textbf{Comprehensive Experimental Validation.} We evaluate SeeUPO on two challenging multi-turn agentic benchmarks (AppWorld and BFCL v4), demonstrating substantial performance improvements (43.3\%--54.6\% on Qwen3-14B, 24.1\%--41.9\% on Qwen2.5-14B) and superior training stability compared to existing backbone RL algorithms.
\end{itemize}

\paragraph{Organization}
The remainder of this paper is organized as follows. Section~\ref{sec:preliminaries} introduces necessary preliminaries, including token / sequence-level RL modeling, advantage estimation methods, policy update mechanisms, and the convergence theory frameworks. Section~\ref{sec:analysis} provides a systematic convergence analysis of existing backbone RL algorithms, analyzing the properties of different advantage estimation and policy update combinations (Subsection~\ref{subsec:advantage_estimation}) and summarizing the fundamental trade-off between critic-free operation and convergence guarantees in multi-turn scenarios (Subsection~\ref{subsec:summary}). Section~\ref{sec:method} presents SeeUPO, including its theoretical framework(Subsection~\ref{subsec:seeupo_update}) and practical implementation details (Subsection~\ref{subsec:practical_methods}). Section~\ref{sec:experiments} reports experimental results. Section~\ref{sec:conclusion} concludes this paper.

%% file: sections/SeeUPO_Main/preliminaries.tex
\section{Preliminaries}
\label{sec:preliminaries}

\subsection{Token-Level RL and Sequence-Level RL}

From a modeling-level perspective, mainstream RL for LLMs can be roughly categorized into two types: \textbf{token-level} and \textbf{sequence-level}. In \textbf{token-level RL}, each token generation step corresponds to a timestep in the MDP, where the state $s_t$ represents the concatenation of the query and previous tokens, and the action $a_t$ is the next token to be generated. In single-turn scenarios, the state transition is deterministic: $s_{t+1} = \text{concat}(s_t, a_t)$. In multi-turn interactive scenarios, the state transition becomes non-deterministic, as token-level RL models interactions by incorporating turn transitions and environmental feedback into the state through concatenation.

In \textbf{sequence-level RL}, each timestep corresponds to generating a complete sequence (sentence or turn), where the state $\boldsymbol{s}$ represents the environmental state and the action $\boldsymbol{a}$ is the complete response sequence. Under this formulation, a single-turn task degenerates into a \textit{contextual bandit model}, while multi-turn tasks correspond to \textit{standard MDPs} with true environmental state transitions~\citep{zheng2025stabilizing, design2025reinforcing}. Sequence-level RL provides a more natural abstraction for multi-turn interactions, as it directly models the environmental dynamics and enables more straightforward credit assignment across turns. 

\textit{Throughout this paper, we use plain symbols (e.g., $s$, $a$) to denote token-level elements and bold symbols (e.g., $\boldsymbol{s}$, $\boldsymbol{a}$) to denote sequence-level elements. In discussions of conventional RL methods, we use plain symbols.}

\subsection{GAE and GRAE}

Advantage estimation is a key component of reinforcement learning. In traditional RL, \textbf{GAE}~\citep{GAE} estimates advantages using temporal difference (TD) errors computed via a value function network (critic). Specifically, for a trajectory segment, GAE computes the advantage estimate as $\hat{A}^{\text{GAE}}(s_t, a_t) = \sum_{l=0}^{\infty} (\gamma\lambda)^l \delta_{t+l}$, where $\delta_t = r_t + \gamma V_\phi(s_{t+1}) - V_\phi(s_t)$ is the TD error, $V_\phi$ is the value network, $\gamma$ is the discount factor, and $\lambda \in [0,1]$ controls the bias-variance trade-off. Typically, GAE trades off some unbiasedness for lower estimation variance.

The computational cost of training a separate critic network in LLM-targeted RL has motivated the development of critic-free advantage estimation methods, with \textbf{GRAE} being a representative approach. 
Given a query (initial state) $s_0$, GRAE samples $N$ independent responses (trajectories) $\{\tau^{(i)}\}_{i=1}^{N}$ from the policy $\pi_\theta$ and computes the group mean reward $\bar{R} = \frac{1}{N} \sum_{i=1}^{N} R^{(i)}$, where $R^{(i)}$ denotes the cumulative reward of trajectory $\tau^{(i)}$. For $(s_t, a_t) \in \tau^{(i)}$, GRAE assigns the advantage estimate as:
 $\hat{A}^{\text{GRAE}}(s_t, a_t) = R^{(i)} - \bar{R}$. For token-level RL, GRAE further broadcasts the advantage to each token.
We emphasize that GRAE here represents a general (ideal) formulation, and different practical RL algorithms may implement it differently. For instance, algorithms such as GRPO and GSPO further divide by the group variance, which has been shown to introduce bias~\citep{hureinforce++}, and we will show that this practice breaks the monotonic improvement property. For ease of discussion, we omit the leave-one-out operation in GRAE formulations~\citep{ahmadian2024back}.

\subsection{REINFORCE and PPU}

Policy update mechanisms are essential for translating advantage estimates into policy improvements~\citep{RL}. We discuss two mainstream update mechanisms used in backbone RL algorithms, namely \textbf{REINFORCE} and \textbf{PPU}.
\textbf{REINFORCE}~\citep{RL} is a fully on-policy algorithm that directly optimizes the expected return. The policy gradient is estimated by sampling state-action pairs $(s_t, a_t)$ from the current policy $\pi_\theta$:
\begin{equation}
    \nabla_\theta J_{\text{REINFORCE}}(\theta) = \mathbb{E}_{(s_t, a_t) \sim \pi_\theta} \left[ \nabla_\theta \log \pi_\theta(a_t | s_t) \hat{A}(s_t, a_t) \right],
\end{equation}
where $\hat{A}(s_t, a_t)$ is the estimated advantage.
\textbf{PPU} typically refers to the mechanism used in PPO~\citep{PPO}, which allows for multiple updates per data batch (partially on-policy) by constraining the policy shift. PPU maximizes a clipped surrogate objective estimated over samples collected by an old policy $\pi_{\theta_{\text{old}}}$:
\begin{equation}
    \nabla_\theta J_{\text{PPU}}(\theta) = \nabla_\theta \mathbb{E}_{(s_t, a_t) \sim \pi_{\theta_{\text{old}}}} \left[ \min \left( r_t(\theta) \hat{A}_t, \operatorname{clip}(r_t(\theta), 1 - \epsilon, 1 + \epsilon) \hat{A}_t \right) \right],
\end{equation}
where $\hat{A}_t$ denotes $\hat{A}(s_t, a_t)$, $r_t(\theta) = \frac{\pi_\theta(a_t | s_t)}{\pi_{\theta_{\text{old}}}(a_t | s_t)}$ is the importance sampling probability ratio and $\epsilon$ is the clipping parameter.

\subsection{Mirror Learning and Heterogeneous-Agent Mirror Learning}

Convergence guarantees are crucial for ensuring \textit{training stability} and \textit{optimal policy discovery} in reinforcement learning. \textbf{Mirror Learning}~\citep{Mirror} provides a unified theoretical framework for analyzing the convergence properties of policy optimization algorithms. The framework introduces key concepts: (1) the \textbf{drift function} $\mathfrak{D}_\pi(\bar{\pi} | s)$, which quantifies the update cost of a new policy relative to the old policy, (2) the \textbf{neighborhood operator} $\mathcal{N}$, which defines the search space for policy updates, and (3) the \textbf{mirror operator}, which combines the advantage-related objective with the drift function to guide policy updates. 
When the drift function and neighborhood operator satisfy specific conditions, Mirror Learning guarantees \textit{monotonic policy improvement} and \textit{convergence to optimal policies}. This framework has been successfully applied to prove convergence guarantees for algorithms such as GPI~\citep{RL} and PPO. We leverage this work to help analyze the convergence properties of backbone RL algorithms used for training LLMs.

\textbf{Heterogeneous-Agent Mirror Learning (HAML)}~\citep{HARL} extends Mirror Learning to multi-agent settings. The core insight is that by decomposing the joint advantage function into conditional advantages and performing sequential policy updates, policy improvements can be coordinated across agents, avoiding conflicts that could arise from simultaneous updates. Under HAML, algorithms guarantee monotonic improvement of the joint return and convergence to Nash equilibrium. We leverage HAML in our proposed SeeUPO algorithm to achieve convergence guarantees in multi-turn scenarios.

\textit{For a comprehensive introduction to the Mirror Learning and HAML framework, including detailed definitions and formal properties, we refer readers to Appendix~\ref{sec:appendix_mirror_learning}.}

%% file: sections/SeeUPO_Main/analysis.tex
\definecolor{GAEColor}{RGB}{65,105,225}
\definecolor{GRAEColor}{RGB}{75,46,131}
\definecolor{PPOColor}{RGB}{255,105,180}
\definecolor{REINFORCEColor}{RGB}{255,165,0}
\definecolor{HAPPOColor}{RGB}{128,0,128}

\definecolor{myBlue}{RGB}{25, 25, 112}      
\definecolor{myPurple}{RGB}{128, 0, 128}   
\definecolor{myGreen}{RGB}{34, 139, 34}    
\definecolor{myGray}{gray}{0.6}             
\definecolor{noteGray}{gray}{0.3}           
\definecolor{rowGray}{gray}{0.96}           

\newcommand{\cmark}{\textcolor{myGreen}{\ding{51}}} 
\newcommand{\xmark}{\textcolor{myGray}{\ding{55}}}  

\newcommand{\algoFont}[1]{\textsc{#1}}                                    
\newcommand{\advFont}[1]{\textcolor{myBlue}{\textsf{\textbf{#1}}}}        
\newcommand{\polFont}[1]{\textcolor{myPurple}{\textsf{\textbf{#1}}}}      
\newcommand{\noteFont}[1]{\itshape\color{noteGray}\raggedright #1} 

\section{Analysis}
\label{sec:analysis}

In this section, we provide a systematic convergence analysis of mainstream backbone RL algorithms. We first present the algorithm comparison table and key findings, then analyze two fundamental components: advantage estimation methods (\textbf{GAE} and \textbf{GRAE}) and policy update mechanisms (\textbf{REINFORCE} and \textbf{PPU}), followed by analyzing their combined effects, and finally summarize the performance of existing algorithms in single-turn and multi-turn scenarios.

Table~\ref{tab:algorithm_comparison} provides a systematic comparison of mainstream backbone reinforcement learning algorithms used for language model training, analyzing their core components, modeling levels, and convergence properties.
Our systematic analysis reveals a fundamental trade-off: \textbf{mainstream backbone RL algorithms struggle to simultaneously achieve both critic-free operation and convergence guarantees in multi-turn scenarios}.

\begin{table}[htbp]
    \centering
\caption{Systematic analysis of mainstream backbone RL algorithms for language model training. The table compares algorithms based on their advantage estimation methods, policy update mechanisms, modeling level, and convergence guarantees in single-turn and multi-turn scenarios. Detailed theoretical analysis is provided in the corresponding appendix sections.}
\label{tab:algorithm_comparison}
\tiny
\setlength{\tabcolsep}{5pt}
\renewcommand{\arraystretch}{1.4}
\begin{tabularx}{\textwidth}{@{}>{\centering\arraybackslash}p{1.4cm}>{\centering\arraybackslash}p{1.6cm}>{\centering\arraybackslash}p{1.6cm}>{\centering\arraybackslash}c>{\centering\arraybackslash}c>{\centering\arraybackslash}p{2.6cm}>{\centering\arraybackslash}p{3.6cm}@{}}
\toprule
\begin{tabular}[c]{@{}c@{}}\textbf{Advantage} \\ \textbf{Estimation}\end{tabular} & \begin{tabular}[c]{@{}c@{}}\textbf{Policy Update} \\ \textbf{Mechanism}\end{tabular} & \begin{tabular}[c]{@{}c@{}}\textbf{Modeling} \\ \textbf{Level}\end{tabular} & \begin{tabular}[c]{@{}c@{}}\textbf{Single-Turn} \\ \textbf{Convergence}\end{tabular} & \begin{tabular}[c]{@{}c@{}}\textbf{Multi-Turn} \\ \textbf{Convergence}\end{tabular} & \begin{tabular}[c]{@{}c@{}}\textbf{Algorithm} \\ \textbf{Instance}\end{tabular} & \textbf{Notes} \\
\midrule
\advFont{GAE} & \polFont{PPU} & Token-level & \cmark & \cmark & \algoFont{PPO}~\citep{PPO} & \begin{minipage}[c]{3.6cm}\raggedright\noteFont{Critic-dependent; Assumes perfect value function approximation (Appendix~\ref{app:gae_bias},~\ref{app:gae_ppo_global})}\end{minipage} \\
\addlinespace[0.25em]
\rowcolor{rowGray}
\advFont{GRAE} & \polFont{PPU} & Token-level & \xmark & \xmark & \algoFont{REINFORCE++ (w/ Baseline)}~\citep{hureinforce++}, \algoFont{GRPO}~\citep{grpo2024} & \begin{minipage}[c]{3.6cm}\raggedright\noteFont{Structural bias breaks monotonic improvement (Appendix~\ref{app:grae_bias_gradient},~\ref{app:grae_ppo_convergence})}\end{minipage} \\
\addlinespace[0.25em]
\advFont{GRAE} & \polFont{REINFORCE} & Sequence-Level & \cmark & \xmark & \algoFont{RLOO}~\citep{ahmadian2024back} & \begin{minipage}[c]{3.6cm}\raggedright\noteFont{Requires undiscounted settings (Appendix~\ref{app:grae_reinforce_global})}\end{minipage} \\
\addlinespace[0.25em]
\rowcolor{rowGray}
\advFont{GRAE} & \polFont{PPU} & Sequence-Level & \cmark & \xmark & \algoFont{GSPO}~\citep{gspo2024} & \begin{minipage}[c]{3.6cm}\raggedright\noteFont{Requires removal of group-variance normalization (Appendix~\ref{app:grae_bandit_gspo})}\end{minipage} \\
\addlinespace[0.25em]
\advFont{GRAE} & \polFont{HAML} & Sequence-Level & --- & \cmark & \algoFont{SeeUPO} & \begin{minipage}[c]{3.6cm}\raggedright\noteFont{Converts multi-turn to sequential multi-agent single-turn (Appendix~\ref{sec:seeupo_global_convergence})}\end{minipage} \\
\bottomrule
\end{tabularx}
\end{table}

\subsection{Advantage Estimation and Policy Update Analysis}
\label{subsec:advantage_estimation}

We analyze two mainstream advantage estimation methods and their combinations with policy update mechanisms. Detailed theoretical analysis and proofs are provided in Appendix~\ref{app:gae_bias}--\ref{app:grae_bandit_gspo}.

\textbf{GAE} provides unbiased estimates under perfect value function approximation, with bias bounded by $|\text{Bias}| \leq \frac{1 + \gamma - 2\gamma\lambda}{1 - \gamma\lambda} \cdot \epsilon_{\max}$ where $\epsilon_{\max}$ is the maximum value estimation error (Appendix~\ref{app:gae_bias}). \textbf{GRAE} is biased but provides unbiased gradient estimates under undiscounted ($\gamma = 1$) settings; when the conditions are violated, both the estimator and gradient become biased (Appendix~\ref{app:grae_bias_gradient}).

We analyze how different combinations affect convergence, summarized in Table~\ref{tab:algorithm_comparison}:

\begin{itemize}
    \item \textbf{GRAE-REINFORCE} (e.g., RLOO): Guarantees gradient unbiasedness, monotonic improvement, and convergence under undiscounted objective ($\gamma=1$) with bounded rewards in finite-horizon MDPs. REINFORCE can be viewed as an instance of Mirror Learning~\citep{Mirror} with trivial drift ($\mathfrak{D} \equiv 0$) and trivial neighbourhood ($\mathcal{N} = \Pi$), which ensures the convergence guarantee (Appendix~\ref{app:grae_reinforce_global}).
    
    \item \textbf{GAE-PPU} (e.g., PPO): Guarantees monotonic improvement and convergence to globally optimal policy under perfect value function approximation (Appendix~\ref{app:gae_ppo_global}).
    
    \item \textbf{GRAE-PPU} (e.g., GRPO, REINFORCE++ w/ Baseline~\citep{hureinforce++}, GSPO): Does \textit{not} guarantee monotonic improvement or convergence in general, because GRAE's structural bias $\Delta(s_t) = V(s_t) - V(s_0)$ cannot be eliminated in PPU's clipped objective (Appendix~\ref{app:grae_ppo_convergence}). \textbf{Exception}: In Contextual Bandit settings (single-turn \& sequence-level), GRAE becomes unbiased and GRAE-PPU converges; however, group-variance normalization (as in GSPO) breaks this guarantee (Appendix~\ref{app:grae_bandit_gspo}).
    
    \item \textbf{GAE-REINFORCE} (no existing algorithm instance): This combination is \textit{Pareto-dominated} by existing alternatives. It retains GAE's critic dependency while lacking PPU's trust region constraints that bound policy updates. When value function approximation is imperfect, biased GAE estimates can be amplified through REINFORCE's unbounded updates. In contrast, GAE-PPU achieves better sample efficiency with bounded updates, while GRAE-REINFORCE eliminates critic dependency entirely.
\end{itemize}

\subsection{Summary: The Fundamental Trade-off}
\label{subsec:summary}

Based on the analysis above, we summarize the convergence properties. In \textbf{single-turn scenarios}: RLOO converges under undiscounted settings; PPO converges under perfect value function approximation; GSPO (without group-variance normalization) converges in Contextual Bandit settings.

In \textbf{multi-turn scenarios}, our analysis reveals a fundamental trade-off: \textbf{mainstream backbone RL algorithms struggle to simultaneously achieve both critic-free operation and convergence guarantees}. 

\begin{itemize}
    \item \textbf{Critic-dependent methods (GAE-PPU)}: Require accurate token-level value function estimation, which becomes challenging in multi-turn scenarios due to non-stationary state transitions, with additional computational overhead from maintaining the critic network~\citep{zhanglandscape}.
    \item \textbf{Critic-free methods (GRAE-based)}: GRAE-PPU breaks monotonic improvement due to structural bias in PPU's clipped objective. GRAE-REINFORCE requires undiscounted settings that are difficult to satisfy in multi-turn scenarios. Moreover, in sequence-level settings, GRAE becomes biased in multi-turn scenarios due to credit assignment problems and amplified structural bias $\Delta(\boldsymbol{s}_{\boldsymbol{t}}) = V(\boldsymbol{s}_{\boldsymbol{t}}) - V(\boldsymbol{s}_0)$.
\end{itemize}

This fundamental limitation motivates a new algorithmic framework that achieves both critic-free operation and convergence guarantees in multi-turn scenarios, which we address in Section~\ref{sec:method}.

%% file: sections/SeeUPO_Main/method.tex
\definecolor{YellowOrange}{RGB}{255,165,0}      
\definecolor{RoyalBlue}{RGB}{65,105,225}        
\definecolor{Magenta}{RGB}{255,105,180}          
\definecolor{Purple}{RGB}{128,0,128}            

\section{Method}
\label{sec:method}

To address this limitation, we propose a new algorithmic framework: \textbf{\textcolor{richpurple}{\textit{Se}}quence-level S\textcolor{richpurple}{\textit{e}}quential \textcolor{richpurple}{\textit{U}}pdate \textcolor{richpurple}{\textit{P}}olicy \textcolor{richpurple}{\textit{O}}ptimization} (\textcolor{richpurple}{\textbf{SeeUPO}}).
The core mechanism of SeeUPO is \textbf{to model multi-turn interaction problems as sequentially-executed multi-agent contextual bandit problems}, where each turn is abstracted as a \textit{virtual agent}, and the order of turns corresponds to the execution order among agents.

\begin{tcolorbox}[
  colback=purple!5!white,
  colframe=richpurple!80!black,
  boxrule=1pt,
  arc=2mm,
  left=3mm,
  right=3mm,
  top=2mm,
  bottom=2mm,
  fonttitle=\bfseries,
  title=\textbf{Fundamental Principles of SeeUPO}
]
This modeling approach is built upon two fundamental principles:
\begin{enumerate}
    \item \textbf{Transforming convergence analysis at the turn-level into convergence analysis at the agent-level}, which allows leveraging existing multi-agent reinforcement learning (MARL) theories to solve the problem.
    \item \textbf{Transforming multi-timestep MDP problems into bandit problems}, which enables unbiased advantage estimation without requiring value function estimation.
\end{enumerate}
\end{tcolorbox}

\subsection{Theoretical Framework: Sequence-level Sequential Update Policy Optimization}
\label{subsec:seeupo_update}

\paragraph{Multi-Agent Modeling}
\label{par:multi_agent_modeling}

SeeUPO abstracts multi-turn interaction tasks into sequentially-decision multi-agent single-turn bandit problems, where each turn is mapped to a virtual agent $t \in \{1, 2, \ldots, T\}$ (as depicted in Figure~\ref{fig:seeupo-idea}). All agents share a common global state $\boldsymbol{s}_0 \in \mathcal{S}_S$ (the initial task state). The sequence-level action $\boldsymbol{a}^t \in \mathcal{A}_S$ of agent $t$ corresponds to the complete response of the $t$-th turn. The joint action $\boldsymbol{a}^{1:T} = (\boldsymbol{a}^1, \boldsymbol{a}^2, \ldots, \boldsymbol{a}^T)$ denotes the concatenation of all agents' actions. Each agent $t$'s sequence-level policy $\boldsymbol{\pi}^t(\boldsymbol{a}^t | \boldsymbol{s}_0, \boldsymbol{a}^{1:t-1})$ takes as input the global state $\boldsymbol{s}_0$ and the action history from preceding agents $\boldsymbol{a}^{1:t-1}$. The state transition function is implicitly modeled through sequentially executed policies: the evolution of the interaction is determined by agent $t$'s action selected according to policy $\boldsymbol{\pi}^{t}(\cdot|\boldsymbol{s}_0, \boldsymbol{a}^{1:t-1})$. We adopt a shared reward (\textit{Team-Reward}) mechanism $r(\boldsymbol{s}_0, \boldsymbol{a}^{1:T})$, where the team reward equals the final task reward or cumulative return across all turns, ensuring all agents jointly optimize the global objective.

\begin{figure}[htbp]
  \centering
  \includegraphics[width=1.0\textwidth]{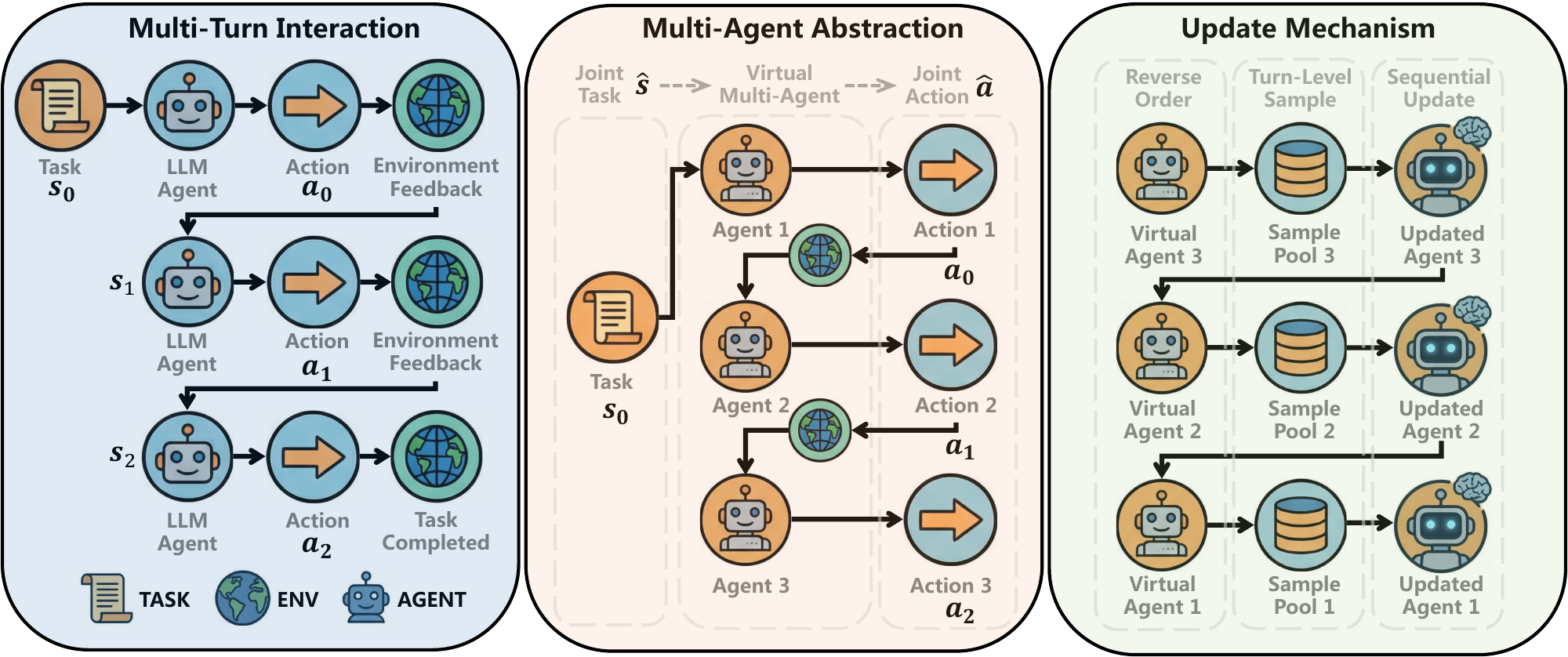}
  \caption{The core idea of SeeUPO is to abstract multi-turn interaction tasks into sequentially-decision multi-agent single-turn tasks, and adopt reverse-order sequential updates to achieve global optimality via backward induction. The figure shows an example scenario with three turns, from left to right showing the \textit{original task scenario}, the \textit{multi-agent modeling} of the scenario, and the \textit{reverse update mechanism} based on MARL theory.}
  \label{fig:seeupo-idea}
\end{figure}

We emphasize that \textbf{this multi-agent modeling only exhibits a \textit{training-phase specificity} and offers no inherent advantages for execution optimization}. This abstraction constitutes a methodological shift in data treatment during the training phase, but does not correspond to a functional mechanism for multi-agent coordination during actual execution.

\paragraph{Policy Update}
\label{par:policy_update} 

After modeling is completed, \textbf{the optimization of Sequence-level policies is transformed into optimizing the joint policy of a multi-agent system.}\footnote{Throughout this section, we use $\bar{\pi}$ to denote a policy to be optimized, and $\hat{\boldsymbol{\pi}}$ to denote the joint policy.} SeeUPO adopts the sequential update mechanism from the HAML framework (see Appendix~\ref{sec:appendix_mirror_learning}), with a crucial design choice: \textbf{the update order is set to the reverse of the execution order} ($T \to T{-}1 \to \cdots \to 1$). This reverse update order not only resolves agent-level update conflicts by updating policies turn by turn, but also enables \textbf{backward induction to achieve global optimality} (see Theorem 2 in Appendix~\ref{sec:seeupo_global_convergence}).

At each iteration $k$, the algorithm updates each agent's policy sequentially following the \textbf{reverse order} of execution ($T \to T{-}1 \to \cdots \to 1$). The policy update process consists of three key components:

\textbf{(1) Policy Update Rule.} For turn $t$ in the reverse update order, by substituting the HAML into the update rule, the policy update can be expressed as:

\begin{equation}
    \hat{\pi}^{t}_{k+1} = \underset{\bar{\pi}^{t} \in \textcolor{YellowOrange}{\mathcal{U}^{t}_{\hat{\boldsymbol{\pi}}_k}(\hat{\pi}^{t}_k)}}{\arg\max} \; \mathbb{E}_{\boldsymbol{s}_0 \sim \beta_{\hat{\boldsymbol{\pi}}_k}} \Bigg[ \textcolor{RoyalBlue}{\mathbb{E}_{\boldsymbol{a}^{t+1:T} \sim \hat{\boldsymbol{\pi}}^{t+1:T}_{k+1}, \boldsymbol{a}^{t} \sim \bar{\pi}^{t}} \left[A^{t}_{\hat{\boldsymbol{\pi}}_k}(\boldsymbol{s}_0, \boldsymbol{a}^{t}, \boldsymbol{a}^{t+1:T})\right]} - \textcolor{Magenta}{\mathfrak{D}^{t}_{\hat{\boldsymbol{\pi}}_k}(\bar{\pi}^{t} \mid \boldsymbol{s}_0, \hat{\boldsymbol{\pi}}^{t+1:T}_{k+1})} \Bigg],
    \label{eq:seeupo-update}
\end{equation}

where $\boldsymbol{s}_0$ is the sequence-level joint state (i.e., the global initial state), $\hat{\boldsymbol{\pi}}_k$ denotes the joint policy at iteration $k$, $\textcolor{YellowOrange}{\mathcal{U}^{t}_{\hat{\boldsymbol{\pi}}_k}(\hat{\pi}^{t}_k)}$ is the \textcolor{YellowOrange}{neighborhood operator} for turn $t$, $\beta_{\hat{\boldsymbol{\pi}}_k}$ is the sampling state distribution, $\textcolor{RoyalBlue}{\mathbb{E}_{\boldsymbol{a}^{t+1:T} \sim \hat{\boldsymbol{\pi}}^{t+1:T}_{k+1}, \boldsymbol{a}^{t} \sim \bar{\pi}^{t}} \left[A^{t}_{\hat{\boldsymbol{\pi}}_k}\right]}$ is the \textcolor{RoyalBlue}{expectation of the local advantage function} for turn $t$, and $\textcolor{Magenta}{\mathfrak{D}^{t}_{\hat{\boldsymbol{\pi}}_k}}$ is the \textcolor{Magenta}{drift functional}. Crucially, under the reverse update order, \textbf{each turn $t$ considers the already-updated policies $\hat{\boldsymbol{\pi}}^{t+1:T}_{k+1}$ of subsequent turns}, ensuring coordination in the update process and enabling backward induction.

\textbf{(2) Local Advantage Function Computation.} To compute the \textcolor{RoyalBlue}{expectation of the local advantage function} $\textcolor{RoyalBlue}{\mathbb{E}_{\boldsymbol{a}^{t+1:T} \sim \hat{\boldsymbol{\pi}}^{t+1:T}_{k+1}, \boldsymbol{a}^{t} \sim \bar{\pi}^{t}} \left[A^{t}_{\hat{\boldsymbol{\pi}}_k}\right]}$, SeeUPO leverages the global advantage function. Given the global advantage function $\textcolor{Purple}{\hat{A}_{\hat{\boldsymbol{\pi}}_k}(\boldsymbol{s}_0, \boldsymbol{a}^{1:T})}$, this expectation can be estimated:

\begin{align}
    &\textcolor{RoyalBlue}{\mathbb{E}_{\boldsymbol{a}^{t+1:T} \sim \hat{\boldsymbol{\pi}}^{t+1:T}_{k+1}, \boldsymbol{a}^{t} \sim \bar{\boldsymbol{\pi}}^{t}} \left[A^{t}_{\hat{\boldsymbol{\pi}}_k}(\boldsymbol{s}_0, \boldsymbol{a}^{t}, \boldsymbol{a}^{t+1:T})\right]} \nonumber \\
    &= \mathbb{E}_{\boldsymbol{a}^{1:T} \sim \hat{\boldsymbol{\pi}}_k} \left[ \left(\frac{\bar{\boldsymbol{\pi}}^{t}(\boldsymbol{a}^{t}|\boldsymbol{s}_0, \boldsymbol{a}^{1:t-1})}{\hat{\boldsymbol{\pi}}^{t}_k(\boldsymbol{a}^{t}|\boldsymbol{s}_0, \boldsymbol{a}^{1:t-1})}-1\right) \cdot \frac{\hat{\boldsymbol{\pi}}^{t+1:T}_{k+1}(\boldsymbol{a}^{t+1:T}|\boldsymbol{s}_0, \boldsymbol{a}^{1:t})}{\hat{\boldsymbol{\pi}}^{t+1:T}_k(\boldsymbol{a}^{t+1:T}|\boldsymbol{s}_0, \boldsymbol{a}^{1:t})} \cdot \textcolor{Purple}{\hat{A}_{\hat{\boldsymbol{\pi}}_k}(\boldsymbol{s}_0, \boldsymbol{a}^{1:T})} \right],
    \label{eq:local-advantage-estimation}
\end{align}

where the first term $\left(\frac{\bar{\boldsymbol{\pi}}^{t}(\boldsymbol{a}^{t}|\boldsymbol{s}_0, \boldsymbol{a}^{1:t-1})}{\hat{\boldsymbol{\pi}}^{t}_k(\boldsymbol{a}^{t}|\boldsymbol{s}_0, \boldsymbol{a}^{1:t-1})}-1\right)$ involves the candidate policy $\bar{\boldsymbol{\pi}}^{t}$ (to be optimized) and the current policy $\hat{\boldsymbol{\pi}}^{t}_k$, while the second ratio $\frac{\hat{\boldsymbol{\pi}}^{t+1:T}_{k+1}(\boldsymbol{a}^{t+1:T}|\boldsymbol{s}_0, \boldsymbol{a}^{1:t})}{\hat{\boldsymbol{\pi}}^{t+1:T}_k(\boldsymbol{a}^{t+1:T}|\boldsymbol{s}_0, \boldsymbol{a}^{1:t})}$ involves the already-updated joint policy of subsequent turns $\hat{\boldsymbol{\pi}}^{t+1:T}_{k+1}$ and the previous policy $\hat{\boldsymbol{\pi}}^{t+1:T}_k$. Note that the $-1$ term in the first factor has zero gradient with respect to $\bar{\boldsymbol{\pi}}^{t}$ and can be omitted in practical gradient computation. \textbf{The computation of the local advantage function effectively performs implicit credit assignment across turns}~\citep{HARL}, as it decomposes the global advantage into turn-specific contributions by incorporating the importance sampling ratios from subsequently updated turns.

\textbf{(3) Global Advantage Function Computation.} In the bandit setting, the global advantage function $\textcolor{Purple}{\hat{A}_{\hat{\boldsymbol{\pi}}_k}(\boldsymbol{s}_0, \boldsymbol{a}^{1:T})}$ can be estimated directly from sampled rewards. Specifically, for a given initial state $\boldsymbol{s}_0$ and sequence-level joint action $\boldsymbol{a}^{1:T}$, the global advantage function degenerates to:

\begin{equation}
    \textcolor{Purple}{\hat{A}_{\hat{\boldsymbol{\pi}}_k}(\boldsymbol{s}_0, \boldsymbol{a}^{1:T})} = r(\boldsymbol{s}_0, \boldsymbol{a}^{1:T}) - \mathbb{E}_{\boldsymbol{a}'^{1:T} \sim \hat{\boldsymbol{\pi}}_k(\cdot|\boldsymbol{s}_0)}[r(\boldsymbol{s}_0, \boldsymbol{a}'^{1:T})],
    \label{eq:global-advantage-bandit}
\end{equation}

where $r(\boldsymbol{s}_0, \boldsymbol{a}^{1:T})$ is the immediate reward and $\mathbb{E}_{\boldsymbol{a}'^{1:T} \sim \hat{\boldsymbol{\pi}}_k(\cdot|\boldsymbol{s}_0)}[r(\boldsymbol{s}_0, \boldsymbol{a}'^{1:T})]$ is the expected reward under the current policy. This formulation provides an unbiased estimate of the advantage function in the bandit setting.

\paragraph{Theoretical Guarantees}
\label{par:theoretical_guarantees}

SeeUPO inherits the monotonic improvement guarantee from the HAML framework (Theorem 1 in Appendix~\ref{sec:appendix_mirror_learning}). Beyond this, we establish a stronger result for the multi-turn contextual bandit setting: \textbf{the reverse update order guarantees convergence to the globally optimal policy} (Theorem 2 in Appendix~\ref{sec:seeupo_global_convergence}). The key insight is that, unlike general cooperative games where random update orders are required for Nash equilibrium convergence, the fixed sequential execution order in our setting enable \textbf{backward induction}. Specifically, the reverse update order ensures that when updating turn $t$, all subsequent turns $t+1, \ldots, T$ have already been updated to their optimal policies given their continuation values. This allows each turn to optimize against the true optimal continuation value $V^*$, yielding global optimality. The complete proof is provided in Appendix~\ref{sec:seeupo_global_convergence}.

\subsection{Practical Methods}
\label{subsec:practical_methods}

\begin{tcolorbox}[
  colback=blue!3!white,
  colframe=richpurple!80!black,
  boxrule=1.5pt,
  arc=3mm,
  left=4mm,
  right=4mm,
  top=3mm,
  bottom=3mm,
  fonttitle=\bfseries,
  title=\textbf{Algorithm 1: SeeUPPO-GRAE}
]
\textbf{Input:} Initial sequence-level joint policy $\boldsymbol{\pi}_{0}$ with parameters $\boldsymbol{\theta}_{0}$, maximum turns $T$, batch size $B$, group size $G$, clipping parameter $\epsilon$, learning rate $\alpha$

\textbf{Output:} Optimized policy $\boldsymbol{\pi}_{K}$ after $K$ iterations

\vspace{0.3em}
\begin{enumerate}
    \item Initialize $\boldsymbol{\pi}_{0}$ with parameters $\boldsymbol{\theta}_{0}$
    \item \textbf{For} $k = 0, 1, \ldots, K-1$:
    \begin{enumerate}
        \item \textbf{Data Collection:}
        \begin{itemize}
            \item Sample dataset $\mathcal{D}_k = \{(\boldsymbol{s}_0, \boldsymbol{a}^{1:T}, r)\}$ by:
            \begin{itemize}
                \item For each of $B$ initial states $\boldsymbol{s}_0$, sample $G$ trajectories $\boldsymbol{a}^{1:T}$ and collect rewards $r(\boldsymbol{s}_0, \boldsymbol{a}^{1:T})$
            \end{itemize}
            \item Organize data into $T$ sample pools: for each turn $t \in \{1, \ldots, T\}$, construct pool $\mathcal{D}_{t} = \{(\boldsymbol{s}_0, \boldsymbol{a}^{1:t-1}, \boldsymbol{a}^{t})\}$
        \end{itemize}
        \item \textbf{Joint Advantage Estimation:}
        \begin{itemize}
            \item \textbf{For} each $(\boldsymbol{s}_0, \boldsymbol{a}^{1:T}) \in \mathcal{D}_T$:
            \begin{itemize}
                \item Compute joint advantage: $\hat{A}_{\hat{\boldsymbol{\pi}}_{k}}(\boldsymbol{s}_0, \boldsymbol{a}^{1:T}) = r(\boldsymbol{s}_0, \boldsymbol{a}^{1:T}) - \bar{r}(\boldsymbol{s}_0)$
                \item Initialize $M^{T+1}(\boldsymbol{s}_0, \boldsymbol{a}^{1:T}) = \hat{A}_{\hat{\boldsymbol{\pi}}_{k}}(\boldsymbol{s}_0, \boldsymbol{a}^{1:T})$
            \end{itemize}
        \end{itemize}
        \item \textbf{Sequential Policy Update (Reverse Order):}
        \begin{itemize}
            \item \textbf{For} $t = T, T{-}1, \ldots, 1$:
            \begin{itemize}
                \item Update policy parameters to obtain $\boldsymbol{\pi}^{t}_{\boldsymbol{\theta}_{k+1}}$ via Equation.~\ref{eq:seeupo-ppo-gradient}:
                \item \textbf{If} $t > 1$:
                \begin{itemize}
                    \item Update $M^{t}$ for next batch via Equation.~\ref{eq:M-update}
                \end{itemize}
                \item \textbf{Else}:
                \begin{itemize}
                    \item Set $\boldsymbol{\theta}_{k+1} = \boldsymbol{\theta}^{1}_{k+1}$
                \end{itemize}
            \end{itemize}
        \end{itemize}
    \end{enumerate}
    \item \textbf{Return} $\boldsymbol{\pi}_{K}$
\end{enumerate}
\end{tcolorbox}

In this part, we present a concrete example of SeeUPO (Algorithm~1 presents the pseudocode). The practical implementation instantiates the theoretical framework in two key aspects: (1) adopting a PPO-style clipping mechanism to implement the mirror operator through gradient-based updates, and (2) using GRAE for joint advantage estimation. This approach essentially combines GRAE with HAPPO \citep{HARL}. We refer to this practical algorithm as \textbf{SeeUPPO-GRAE}, though for convenience, we still refer to it as SeeUPO in the remainder of this paper. We emphasize that SeeUPPO-GRAE is not the only instantiation of SeeUPO---the theoretical framework admits various variants by substituting different components, such as replacing the PPO-style clipping with TRPO-style trust region constraints, or replacing GRAE with other advantage estimators.

\textbf{The algorithm operates iteratively, with each iteration comprising data collection, advantage estimation, and sequential policy updates.} SeeUPO adopts a \textit{turn-oriented} batch construction approach that separately organizes samples from identical turns, as illustrated in Figure~\ref{fig:seeupo-batch}. This approach enables sequential policy updates by maintaining turn-level sample pools, in contrast to methods that construct batches using entire trajectories or concatenated sliced turns. For tasks with fewer than the maximum $T$ turns, placeholder samples (e.g., Sample 6 in the figure) are introduced as \textit{no-op (null action) samples}. Note that the figure demonstrates batch construction patterns using the \textit{React} + \textit{Reasoning-Augmented Template} paradigm \citep{zhai2025agentevolver}.

\begin{figure}[htbp]
  \centering
  \includegraphics[width=1.0\textwidth]{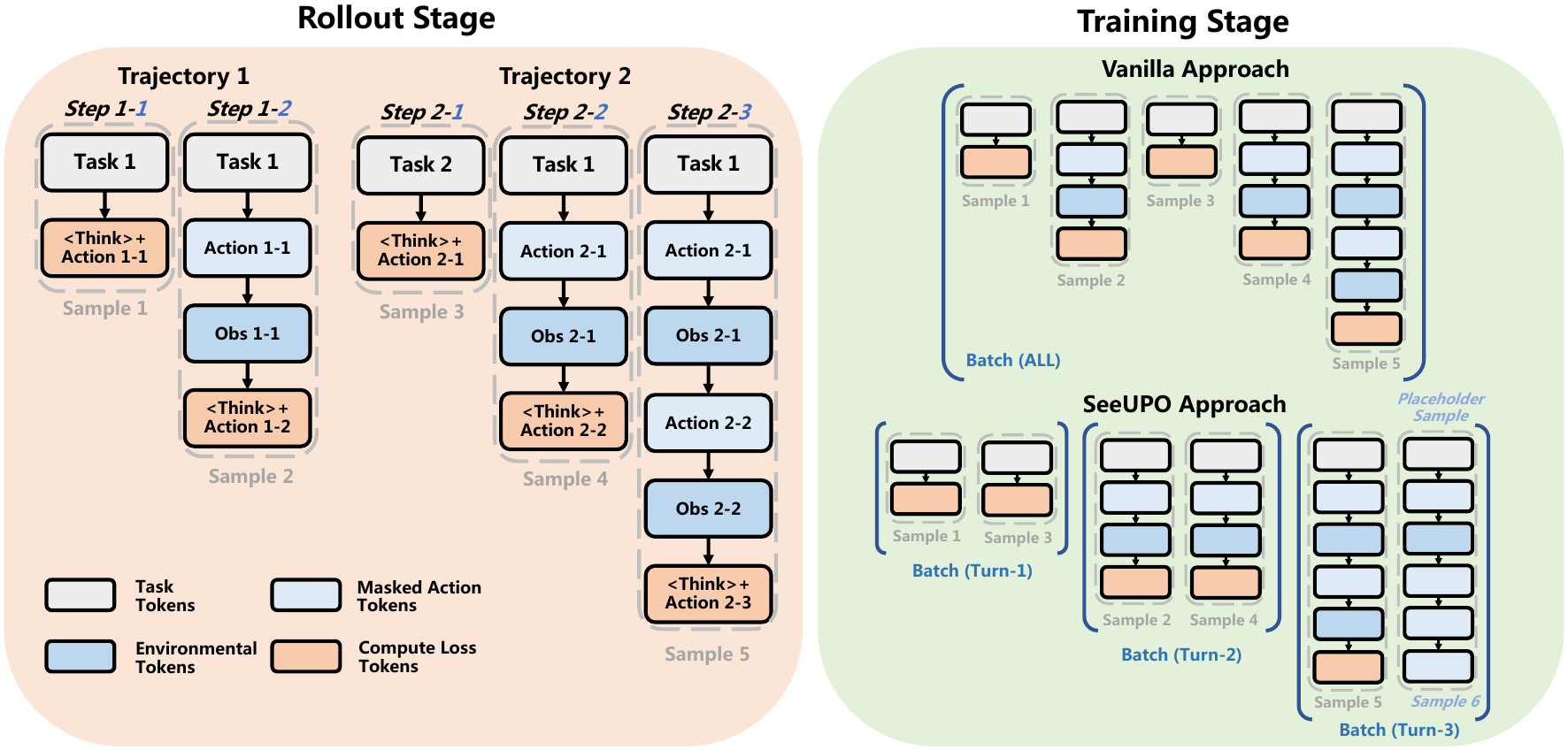}
  \caption{The batch construction approach of SeeUPO. Unlike methods that construct batches using entire trajectories or by concatenating sliced turns, SeeUPO implements a \textit{turn-oriented} approach that separately organizes samples from identical turns. This figure demonstrates the divergent batch construction patterns between SeeUPO and the Vanilla approach under two tasks with maximum three-turn interactions, via \textit{React} + \textit{Reasoning-Augmented Template} paradigm \citep{zhai2025agentevolver}.}
  \label{fig:seeupo-batch}
\end{figure}

\paragraph{PPO-Style Policy Update}
\label{par:ppo_update}

In our multi-turn RL setting, all turns share the same policy parameters $\boldsymbol{\theta}$. We use $\boldsymbol{\pi}_{\boldsymbol{\theta}^{t}_{k+1}}$ to denote the policy after updating turn $t$'s data in iteration $k$. For notational clarity, we denote $(\boldsymbol{s}_0, \boldsymbol{a}^{1:T})$ as a joint trajectory sample, where $\boldsymbol{s}_0$ is the initial state (query) and $\boldsymbol{a}^{1:T}$ is the sequence-level joint action as defined in Section~\ref{par:multi_agent_modeling}.

Specifically, for turn $t$ in the reverse update order ($T \to T{-}1 \to \cdots \to 1$) at iteration $k$, the policy update is performed to obtain $\boldsymbol{\pi}_{\boldsymbol{\theta}^{t}_{k+1}}$ by computing the gradient of the policy parameters $\boldsymbol{\theta}$ with respect to the following expectation:

\begin{equation}
    \nabla_{\boldsymbol{\theta}} \mathbb{E}_{(\boldsymbol{s}_0, \boldsymbol{a}^{1:t-1}, \boldsymbol{a}^{t}) \sim \mathcal{D}_t} \left[ \min \left( r^{t}(\boldsymbol{\theta}) M^{t+1}(\boldsymbol{s}_0, \boldsymbol{a}^{1:T}), \operatorname{clip}(r^{t}(\boldsymbol{\theta}), 1 \pm \epsilon) M^{t+1}(\boldsymbol{s}_0, \boldsymbol{a}^{1:T}) \right) \right],
    \label{eq:seeupo-ppo-gradient}
\end{equation}

where $\mathcal{D}_t = \{(\boldsymbol{s}_0, \boldsymbol{a}^{1:t-1}, \boldsymbol{a}^{t})\}$ is the turn-specific sample pool constructed during data collection (see Algorithm 1), containing samples organized by turn $t$. The sequence-level importance sampling ratio $r^{t}(\boldsymbol{\theta}) = \frac{\boldsymbol{\pi}_{\boldsymbol{\theta}}(\boldsymbol{a}^{t}|\boldsymbol{s}_0, \boldsymbol{a}^{1:t-1})}{\boldsymbol{\pi}_{\boldsymbol{\theta}_k}(\boldsymbol{a}^{t}|\boldsymbol{s}_0, \boldsymbol{a}^{1:t-1})}$ for turn $t$ is computed in a manner similar to GSPO, where $\boldsymbol{a}^{t}$ denotes the action at turn $t$ and $(\boldsymbol{s}_0, \boldsymbol{a}^{1:t-1})$ is the conditioning context (initial state and previous actions). $\epsilon$ is the clipping parameter, and $M^{t+1}(\boldsymbol{s}_0, \boldsymbol{a}^{1:T})$ is a maintained quantity that captures the sequential advantage information from subsequently updated turns.

The quantity $M^{t}(\boldsymbol{s}_0, \boldsymbol{a}^{1:T})$ is initialized and updated sequentially to incorporate the importance sampling ratios from previously updated turns. Specifically, we initialize:

\begin{equation}
    M^{T+1}(\boldsymbol{s}_0, \boldsymbol{a}^{1:T}) = \hat{A}_{\hat{\boldsymbol{\pi}}_{k}}(\boldsymbol{s}_0, \boldsymbol{a}^{1:T}),
    \label{eq:M-initialization}
\end{equation}

where $\hat{A}_{\hat{\boldsymbol{\pi}}_{k}}(\boldsymbol{s}_0, \boldsymbol{a}^{1:T})$ is the global advantage estimate (computed via GRAE as described below). After updating turn $t$, $M^{t}(\boldsymbol{s}_0, \boldsymbol{a}^{1:T})$ is computed recursively:

\begin{equation}
    M^{t}(\boldsymbol{s}_0, \boldsymbol{a}^{1:T}) = \frac{\boldsymbol{\pi}_{\boldsymbol{\theta}^{t}_{k+1}}(\boldsymbol{a}^{t}|\boldsymbol{s}_0, \boldsymbol{a}^{1:t-1})}{\boldsymbol{\pi}_{\boldsymbol{\theta}_k}(\boldsymbol{a}^{t}|\boldsymbol{s}_0, \boldsymbol{a}^{1:t-1})} \cdot M^{t+1}(\boldsymbol{s}_0, \boldsymbol{a}^{1:T}),
    \label{eq:M-update}
\end{equation}

where $\boldsymbol{\theta}^{t}_{k+1}$ denotes the parameters after optimizing turn $t$, and $\boldsymbol{\theta}_k$ denotes the parameters at the beginning of iteration $k$ (i.e., the reference policy used for sampling). Due to parameter sharing, this sequential update mechanism ensures that $M^{t}(\boldsymbol{s}_0, \boldsymbol{a}^{1:T})$ incorporates the importance sampling ratios from all previously updated turns, matching the expectation structure in Equation~\ref{eq:local-advantage-estimation}.

\paragraph{GRAE-based Advantage Estimation}
\label{par:advantage_estimation}

In the bandit setting, the global advantage function $\textcolor{Purple}{\hat{A}_{\hat{\boldsymbol{\pi}}_{k}}(\boldsymbol{s}_0, \boldsymbol{a}^{1:T})}$ can be estimated directly from sampled rewards, as established in Equation~\ref{eq:global-advantage-bandit}. For each initial state in the batch, SeeUPO samples $G$ different joint actions and collects the corresponding Team-Rewards. The global advantage estimate is computed as:

\begin{equation}
    \hat{A}_{\hat{\boldsymbol{\pi}}_{k}}(\boldsymbol{s}_0, \boldsymbol{a}^{1:T}) = r(\boldsymbol{s}_0, \boldsymbol{a}^{1:T}) - \bar{r}(\boldsymbol{s}_0),
    \label{eq:seeupo-advantage-estimation}
\end{equation}

where $\bar{r}(\boldsymbol{s}_0)$ is the mean reward over $G$ trajectories sampled from the same initial state $\boldsymbol{s}_0$, serving as a Monte Carlo estimator of $V_{\hat{\boldsymbol{\pi}}_{k}}(\boldsymbol{s}_0) = \mathbb{E}_{\boldsymbol{a}^{1:T} \sim \hat{\boldsymbol{\pi}}_{k}(\cdot|\boldsymbol{s}_0)}[r(\boldsymbol{s}_0, \boldsymbol{a}^{1:T})]$. This approach provides an unbiased estimate of the advantage function in the bandit setting, without requiring a separate critic (see Appendix~\ref{app:grae_bandit_gspo} for detailed analysis).

In practice, we apply \textbf{batch-level normalization} to the advantage estimates for numerical stability. Specifically, we normalize all advantage estimates in a batch as: $\tilde{A} = (\hat{A} - \mu_{\mathcal{B}}) / \sigma_{\mathcal{B}}$, where $\mu_{\mathcal{B}}$ is the batch mean and $\sigma_{\mathcal{B}}$ is the batch standard deviation. This normalization approach maintains theoretical convergence guarantees while improving training stability: since $\mu_{\mathcal{B}}$ and $\sigma_{\mathcal{B}}$ are constants independent of the candidate policy, the argmax of the optimization problem remains unchanged, leaving the drift functional completely unaffected (see Appendix~\ref{app:batch_normalization} for detailed analysis). This is in contrast to group normalization which applies state-dependent scaling factors that can violate these properties (see Appendix~\ref{app:variance_normalization}). Moreover, experimental results in Section~\ref{subsec:ablation_norm} demonstrate that batch-level normalization performs comparably to group normalization and no normalization, while preserving the theoretical convergence properties.

%% file: sections/SeeUPO_Main/experiments.tex
\section{Experiments}
\label{sec:experiments}

In this section, we conduct a series of comprehensive experiments to systematically evaluate the performance of our proposed SeeUPO framework. We begin by detailing the experimental settings (Section~\ref{sec:exp_settings}), followed by a presentation of the main results that demonstrate the superior performance of SeeUPO compared to existing selected backbone RL algorithms (Section~\ref{sec:main_results}). Finally, we perform additional comparative experiments to investigate the impact of update order and advantage normalization strategies on overall training performance (Section~\ref{sec:ablation_studies}).

\subsection{Experimental Settings}
\label{sec:exp_settings}

\subsubsection{Benchmarks and Evaluations}
\label{subsubsec:benchmarks}

We evaluate SeeUPO on two tool-augmented, agentic benchmarks: \textbf{AppWorld}~\citep{trivedi2024appworld} and \textbf{BFCL~v4}~\citep{patil2025bfcl}. Both benchmarks expose multi-step API/tool interactions under sparse terminal rewards, making them ideal testbeds for evaluating multi-turn agent training algorithms.

\begin{itemize}
    \item \textbf{AppWorld}: A controllable world of apps and people for benchmarking interactive coding agents. AppWorld exposes multi-step API interactions where agents must navigate complex application workflows to accomplish user-specified goals. The environment provides programmatic evaluation through task goal completion tests.
    \item \textbf{BFCL v4}: The Berkeley Function Calling Leaderboard multi-turn benchmark. We use the \emph{multi-turn} split and follow the official evaluation protocol: at the end of each turn, an example is marked correct only if it simultaneously passes \emph{state-based} checks (final backend state matches ground truth on non-private attributes) and \emph{response-based} checks (subset-matched execution path); force-terminated runs are counted as incorrect.
\end{itemize}

For evaluation metrics, we report \textbf{avg@4} (averaging task completion scores over 4 independent rollouts per instance) and \textbf{pass@4} (the success rate when sampling 4 independent rollouts per instance). Trajectories are truncated at a maximum of 10 steps to ensure computational efficiency while maintaining sufficient horizon for multi-turn interactions.

\subsubsection{Baselines and Backbone Models}
\label{subsubsec:baselines}

Our experiments leverage two backbone models: \textbf{Qwen2.5-14B-Instruct}~\citep{yang2025qwen25} and \textbf{Qwen3-14B}~\citep{yang2025qwen3}. These models serve as the agent's policy network.

To rigorously evaluate the efficacy of our method, we compare against three mainstream backbone RL algorithms for agent training:
\begin{itemize}
    \item \textbf{PPO}: A representative critic-dependent algorithm with theoretical convergence guarantees in multi-turn scenarios.
    \item \textbf{GRPO}: A representative critic-free algorithm at the token level, which employs group-relative advantage estimation and PPO-style update mechanism.
    \item \textbf{GSPO}: A representative critic-free algorithm at the sequence level, which extends GRPO with sequence-level mechanisms. 
\end{itemize}

\subsubsection{Implementation Details}
\label{subsubsec:impl_details}

To ensure fair comparison, all baselines and our SeeUPO use the same training configuration. We train our agent policies using SeeUPO as described in Section~\ref{sec:method}. The general training configuration is as follows:

\textbf{Training hyperparameters:} We use a learning rate of $1 \times 10^{-6}$ for the actor network and a batch size of 32. The clipping parameter $\epsilon$ in the PPO-style objective is all set to 0.2 for both lower and upper bounds. The KL penalty coefficient is 0.002. We sample 8 rollouts per instance during training. We train for 75 epochs on AppWorld and 50 epochs on BFCL v4. To ensure fairness, all algorithms use the same number of samples, and samples are utilized the same number of times within each update step. All experiments on both base models adopt the \textit{React} + \textit{Reasoning-Augmented Template} paradigm~\citep{zhai2025agentevolver}.

\textbf{SeeUPO-specific settings:} For the sequential update mechanism, we use reverse order for both AppWorld and BFCL v4 benchmarks, as our subsequent experiments demonstrate that reverse order achieves the best performance. The advantage estimation follows a group-relative approach with batch-level normalization, where we normalize advantages by subtracting the batch mean and dividing by the batch standard deviation.

\textbf{Infrastructure:} All experiments were conducted on clusters of NVIDIA H20 (96GB) GPUs, where each cluster consists of 8 GPUs. For computational resource allocation, PPO uses 2 clusters (16 GPUs total), while GRPO, GSPO, and SeeUPO each use 1 cluster (8 GPUs total). Our implementation is built on PyTorch and leverages the \texttt{veRL} library\footnote{\url{https://github.com/volcengine/verl}} for distributed training infrastructure. The agent-related infrastructure is built upon AgentEvolver\footnote{\url{https://github.com/modelscope/AgentEvolver}}~\citep{zhai2025agentevolver}.

\subsection{Main Results}
\label{sec:main_results}

We conduct comprehensive experiments to evaluate the overall performance of the proposed SeeUPO framework and compare it against selected backbone RL algorithms. As shown in Table~\ref{tab:main_results} and Figure~\ref{fig:baseline_compare}, SeeUPO demonstrates substantial performance improvements over existing backbone RL algorithms across both benchmarks and models, with the training dynamics revealing several critical insights into SeeUPO's superior performance and robustness.

\begin{table}[h]
\centering
\caption{Performance comparison on two benchmark environments. Columns show \textit{avg@4} and \textit{pass@4} for each benchmark, plus their averages (Avg.). All values are in percent (\%). \textbf{Bolded numbers} highlight the best results. Improvement percentages (shown as \textcolor{blue}{blue superscripts}) indicate SeeUPO's relative improvement over each baseline.}
\label{tab:main_results}
\small
\setlength{\tabcolsep}{6pt}
\begin{tabular}{@{}l cc cc cc@{}}
\toprule
\textbf{Model \& Method} &
\multicolumn{2}{c}{\textbf{AppWorld}} &
\multicolumn{2}{c}{\textbf{BFCL v4}} &
\multicolumn{2}{c}{\textbf{Avg.}} \\
\cmidrule(lr){2-3}\cmidrule(lr){4-5}\cmidrule(l){6-7}
 & avg@4 & pass@4 &
 avg@4 & pass@4 &
 avg@4 & pass@4 \\
\midrule
\textit{Qwen2.5-14B}   & 4.825 & 7.018 & 25.50 & 35.75 & 15.16 & 21.39  \\
+PPO    & 40.79\textcolor{blue}{\tiny$^{+24.7\%}$} & 57.89\textcolor{blue}{\tiny$^{+21.2\%}$} & 44.75\textcolor{blue}{\tiny$^{+23.5\%}$} & 56.00\textcolor{blue}{\tiny$^{+1.8\%}$} & 42.77\textcolor{blue}{\tiny$^{+24.1\%}$} & 56.95\textcolor{blue}{\tiny$^{+11.7\%}$}  \\
+GRPO    & 35.53\textcolor{blue}{\tiny$^{+43.3\%}$} & 49.12\textcolor{blue}{\tiny$^{+42.9\%}$} & 46.25\textcolor{blue}{\tiny$^{+19.5\%}$} & 55.00\textcolor{blue}{\tiny$^{+3.6\%}$} & 40.89\textcolor{blue}{\tiny$^{+29.8\%}$} & 52.06\textcolor{blue}{\tiny$^{+22.1\%}$}  \\
+GSPO   & 24.12\textcolor{blue}{\tiny$^{+111.0\%}$} & 38.60\textcolor{blue}{\tiny$^{+81.8\%}$} & 50.75\textcolor{blue}{\tiny$^{+8.9\%}$} & 56.00\textcolor{blue}{\tiny$^{+1.8\%}$} & 37.44\textcolor{blue}{\tiny$^{+41.9\%}$} & 47.30\textcolor{blue}{\tiny$^{+34.4\%}$}  \\
\rowcolor{gray!20}
\textbf{SeeUPO (ours)} &
\textbf{50.88} & \textbf{70.18} &
\textbf{55.25} & \textbf{57.00} &
\textbf{53.07} & \textbf{63.59} \\
\midrule
\textit{Qwen3-14B}   & 20.18 & 35.09 & 40.25 & 47.00 & 30.22 & 41.05  \\
+PPO    & 35.09\textcolor{blue}{\tiny$^{+81.2\%}$} & 54.39\textcolor{blue}{\tiny$^{+48.4\%}$} & 43.75\textcolor{blue}{\tiny$^{+32.6\%}$} & 52.00\textcolor{blue}{\tiny$^{+25.0\%}$} & 39.42\textcolor{blue}{\tiny$^{+54.3\%}$} & 53.20\textcolor{blue}{\tiny$^{+36.9\%}$}  \\
+GRPO    & 40.35\textcolor{blue}{\tiny$^{+57.6\%}$} & 57.89\textcolor{blue}{\tiny$^{+39.4\%}$} & 44.50\textcolor{blue}{\tiny$^{+30.3\%}$} & 53.00\textcolor{blue}{\tiny$^{+22.6\%}$} & 42.43\textcolor{blue}{\tiny$^{+43.3\%}$} & 55.45\textcolor{blue}{\tiny$^{+31.4\%}$}  \\
+GSPO   & 32.89\textcolor{blue}{\tiny$^{+93.4\%}$} & 52.63\textcolor{blue}{\tiny$^{+53.3\%}$} & 45.75\textcolor{blue}{\tiny$^{+26.8\%}$} & 55.00\textcolor{blue}{\tiny$^{+18.2\%}$} & 39.32\textcolor{blue}{\tiny$^{+54.6\%}$} & 53.82\textcolor{blue}{\tiny$^{+35.4\%}$}  \\
\rowcolor{gray!20}
\textbf{SeeUPO (ours)} &
\textbf{63.60} & \textbf{80.70} &
\textbf{58.00} & \textbf{65.00} &
\textbf{60.80} & \textbf{72.85} \\
\bottomrule
\end{tabular}
\end{table}

\subsubsection{Performance and Training Dynamics Analysis}
\label{subsubsec:performance_analysis}

\begin{figure}[h]
\centering
\includegraphics[width=\textwidth]{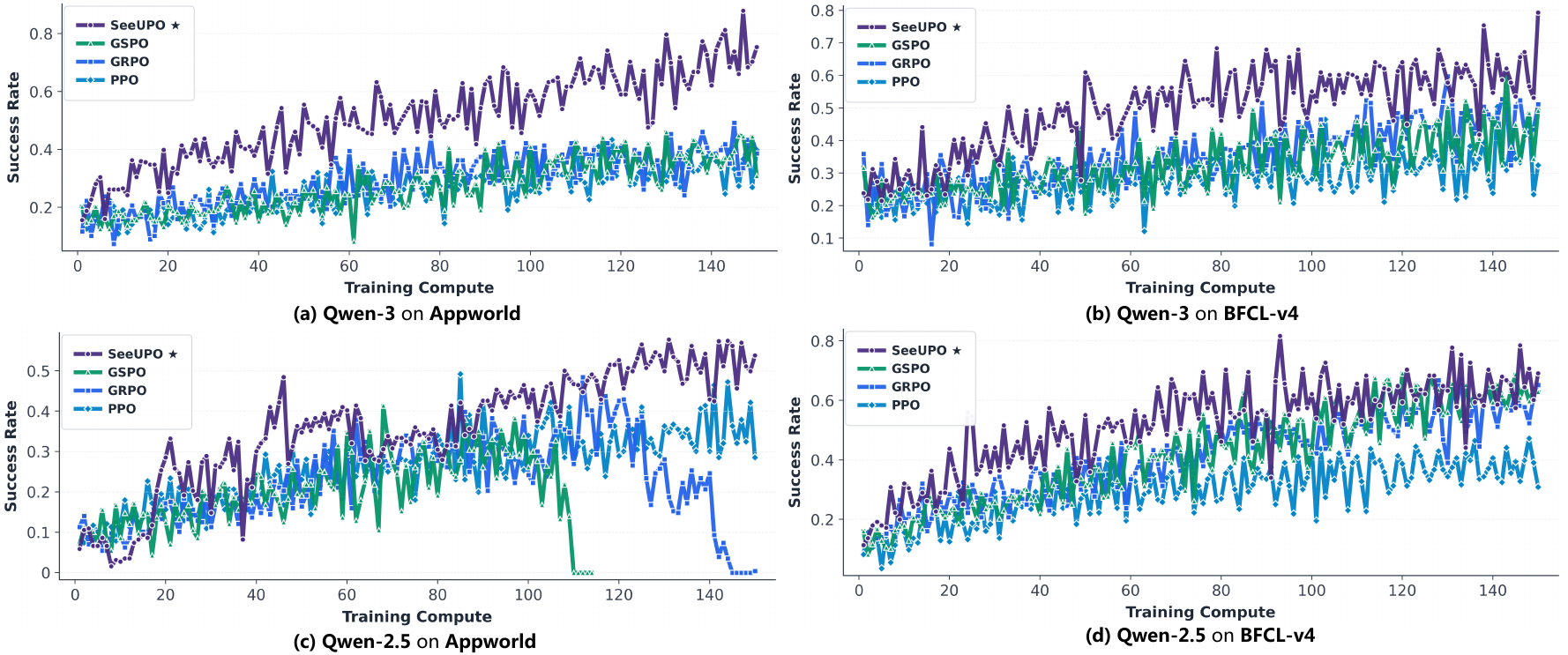}
\caption{Training success rate comparison of SeeUPO and baselines. Subplots (a)-(d) show results for Qwen-3 model on Appworld and BFCL, and Qwen-2.5 model on Appworld and BFCL, respectively.}
\label{fig:baseline_compare}
\end{figure} 

\textbf{Training Stability:} Across all four scenarios, SeeUPO maintains stable training curves without catastrophic failures. In contrast, both GRPO and GSPO exhibit significant performance collapse in the Qwen-2.5 + AppWorld setting (subplot (c)). This catastrophic failure reflects a fundamental limitation of existing critic-free methods in multi-turn scenarios: they suffer from \textbf{advantage estimation bias} when applied to multi-turn tasks, and \textbf{lack monotonic improvement guarantees}. In contrast, SeeUPO addresses these issues through its theoretically-grounded sequential update mechanism, which provides monotonic improvement guarantees inherited from the HAML framework (Theorem 1 in Appendix~\ref{sec:appendix_mirror_learning}) and enables implicit turn-level credit assignment through advantage function decomposition.

\textbf{Final Performance:} The training curves consistently show SeeUPO achieving the highest final success rates across all configurations. The performance gaps are particularly pronounced in the Qwen-3 scenarios, where SeeUPO maintains a clear 20-30 percentage point advantage in pass@4 over baselines throughout training. As shown in Table~\ref{tab:main_results}, SeeUPO achieves substantial improvements on Qwen3-14B, a model with built-in reasoning capabilities, with avg@4 improvements ranging from 43.3\% to 54.6\% across different baselines. On the weaker Qwen2.5-14B model, SeeUPO still demonstrates significant gains, with avg@4 improvements ranging from 24.1\% to 41.9\%. These substantial gains validate that training stability translates directly to superior generalization, consistent with the theoretical guarantee that the reverse-order sequential update mechanism enables convergence to globally optimal policies through backward induction (Theorem 2 in Appendix~\ref{sec:seeupo_global_convergence}).

\subsubsection{Computational Efficiency Analysis}
\label{subsubsec:efficiency}

SeeUPO introduces additional computational overhead compared to baseline methods due to its turn-oriented sequential update mechanism and advantage correction term computation. Therefore, it is necessary to analyze its computational efficiency. We evaluate two key efficiency metrics: \textbf{Training Step Time} (time per training step) and \textbf{Computational Resources} (number of H20 GPUs required). All experiments were conducted on clusters of NVIDIA H20 (96GB) GPUs, with each cluster consisting of 8 GPUs. The efficiency results are presented in Table~\ref{tab:efficiency}.

\begin{table}[h]
\centering
\caption{Computational efficiency comparison across methods. Columns show \textit{Training Step Time} (seconds per training step) for each benchmark and \textit{GPUs} (number of H20 GPUs required). Multiplier factors (shown as \textcolor{red}{red superscripts}) indicate SeeUPO's computation time relative to each baseline method.}
\label{tab:efficiency}
\small
\setlength{\tabcolsep}{6pt}
\begin{tabular}{@{}l ccc c@{}}
\toprule
\textbf{Method} &
\textbf{AppWorld (s)} &
\textbf{BFCL v4 (s)} &
\textbf{Avg. (s)} &
\textbf{GPUs} \\
\midrule
\multicolumn{5}{l}{\textit{Qwen2.5-14B}} \\
\midrule
PPO    & 3026.74\textcolor{red}{\tiny$^{\times 1.39}$} & 2602.88\textcolor{red}{\tiny$^{\times 1.38}$} & 2814.81\textcolor{red}{\tiny$^{\times 1.38}$} & 16 \\
GRPO    & 2377.59\textcolor{red}{\tiny$^{\times 1.76}$} & 2739.73\textcolor{red}{\tiny$^{\times 1.31}$} & 2558.66\textcolor{red}{\tiny$^{\times 1.52}$} & 8 \\
GSPO   & 2263.56\textcolor{red}{\tiny$^{\times 1.85}$} & 2466.14\textcolor{red}{\tiny$^{\times 1.46}$} & 2364.85\textcolor{red}{\tiny$^{\times 1.65}$} & 8 \\
\rowcolor{gray!20}
\textbf{SeeUPO (ours)} & \textbf{4194.70} & \textbf{3595.08} & \textbf{3894.89} & \textbf{8} \\
\midrule
\multicolumn{5}{l}{\textit{Qwen3-14B}} \\
\midrule
PPO    & 3002.29\textcolor{red}{\tiny$^{\times 1.82}$} & 2579.04\textcolor{red}{\tiny$^{\times 1.22}$} & 2790.67\textcolor{red}{\tiny$^{\times 1.54}$} & 16 \\
GRPO    & 2938.37\textcolor{red}{\tiny$^{\times 1.86}$} & 2667.63\textcolor{red}{\tiny$^{\times 1.18}$} & 2802.99\textcolor{red}{\tiny$^{\times 1.54}$} & 8 \\
GSPO   & 3066.94\textcolor{red}{\tiny$^{\times 1.78}$} & 2710.48\textcolor{red}{\tiny$^{\times 1.16}$} & 2888.71\textcolor{red}{\tiny$^{\times 1.49}$} & 8 \\
\rowcolor{gray!20}
\textbf{SeeUPO (ours)} & \textbf{5462.35} & \textbf{3145.06} & \textbf{4303.71} & \textbf{8} \\
\bottomrule
\end{tabular}
\end{table}

As shown in Table~\ref{tab:efficiency}, SeeUPO demonstrates significantly superior performance compared to other algorithms, but its turn-by-turn sequential update mechanism inevitably incurs longer training time. We find that this computational overhead is within an acceptable range: with approximately 1.5$\times$ the training time, SeeUPO achieves notably faster convergence speed and reaches final performance levels far exceeding those of other algorithms. Moreover, in terms of computational resource consumption, \textbf{SeeUPO uses the same amount of resources as the two critic-free algorithms (GRPO and GSPO), requiring only 8 GPUs (1 cluster) compared to PPO's 16 GPUs (2 clusters)}.

\subsection{Additional Comparative Experiments}
\label{sec:ablation_studies}

To investigate the impact of different design choices on overall training performance, we conduct additional comparative experiments focusing on two critical aspects of SeeUPO: \textbf{the update order}, which corresponds to our novel sequential update mechanism that enables backward induction for global optimality (Section~\ref{par:policy_update}), and \textbf{the normalization strategy}, which corresponds to our bandit-based advantage estimation mechanism (Section~\ref{par:advantage_estimation}). All comparative experiments are conducted using Qwen3-14B on both AppWorld and BFCL v4 benchmarks.

\subsubsection{Comparison on Update Order}
\label{subsec:ablation_update_order}

The sequential update mechanism is a core component of SeeUPO (Section~\ref{par:policy_update}). As discussed in Section~\ref{par:theoretical_guarantees}, the reverse update order is theoretically motivated by backward induction, which enables convergence to globally optimal policies (Theorem 2 in Appendix~\ref{sec:seeupo_global_convergence}). However, the specific order in which agents are updated may influence training dynamics and final performance in practice. To empirically validate this theoretical insight, we compare three update order strategies:

\begin{itemize}
    \item \textbf{Natural order}: Agents are updated in the natural execution order (turn 1, turn 2, ..., turn $T$).
    \item \textbf{Reverse order}: Agents are updated in reverse execution order (turn $T$, turn $T-1$, ..., turn 1) (our default setting).
    \item \textbf{Random order}: Agents are updated in a randomly sampled permutation at each iteration.
\end{itemize}

\begin{figure}[h]
\centering
\includegraphics[width=1.0\textwidth]{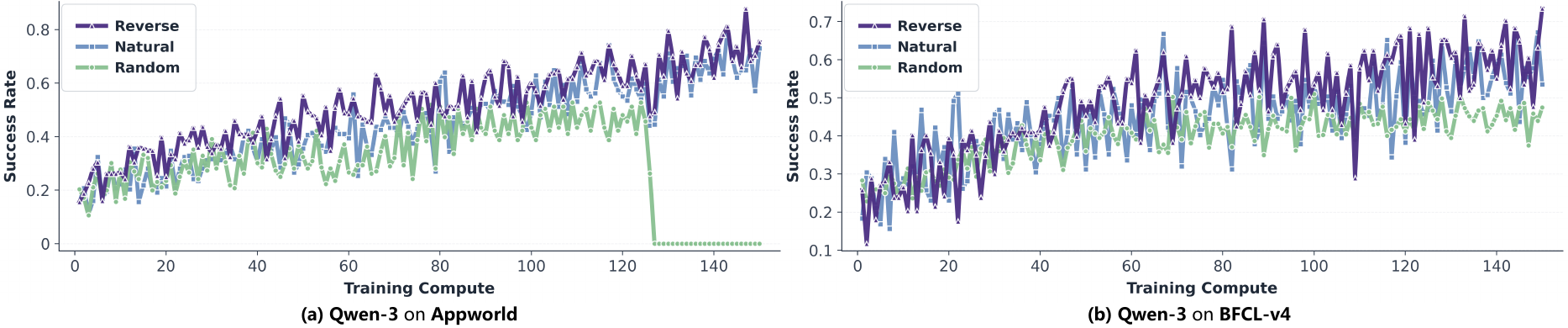}
\caption{Training dynamics comparison of different update order strategies: (a) Qwen-3 evaluated on AppWorld and (b) Qwen-3 evaluated on BFCL-v4.}
\label{fig:ablation_update_order}
\end{figure}

\begin{table}[h]
\centering
\caption{Comparison of update order strategies. Columns show \textit{avg@4} and \textit{pass@4} for each benchmark, plus their averages (Avg.). All values are in percent (\%). The backbone model is Qwen3-14B. \textbf{Bolded numbers} highlight the best results.}
\label{tab:ablation_update_order}
\small
\setlength{\tabcolsep}{8pt}
\begin{tabular}{@{}l cc cc cc@{}}
\toprule
\textbf{Update Order} &
\multicolumn{2}{c}{\textbf{AppWorld}} &
\multicolumn{2}{c}{\textbf{BFCL v4}} &
\multicolumn{2}{c}{\textbf{Avg.}} \\
\cmidrule(lr){2-3}\cmidrule(lr){4-5}\cmidrule(l){6-7}
 & avg@4 & pass@4 & avg@4 & pass@4 & avg@4 & pass@4 \\
\midrule
Natural order & 56.14 & 75.44 & 54.25 & 59.00 & 55.20 & 67.22 \\
Random order & 33.33 & 45.61 & 51.25 & 61.00 & 42.29 & 53.31 \\
\rowcolor{gray!20}
\textbf{Reverse order (ours)} & \textbf{63.60} & \textbf{80.70} & \textbf{58.00} & \textbf{65.00} & \textbf{60.80} & \textbf{72.85} \\
\bottomrule
\end{tabular}
\end{table}

The results are presented in Table~\ref{tab:ablation_update_order} and Figure~\ref{fig:ablation_update_order}. As expected, the reverse order strategy achieves the best performance across both benchmarks, confirming the theoretical insight that updating agents in reverse execution order enables backward induction to achieve global optimality. On AppWorld, reverse order achieves 63.60\% avg@4 and 80.70\% pass@4, outperforming natural order (56.14\% / 75.44\%) and random order (33.33\% / 45.61\%). On BFCL v4, reverse order achieves 58.00\% avg@4 and 65.00\% pass@4, outperforming natural order (54.25\% / 59.00\%) and random order (51.25\% / 61.00\%). Natural order achieves competitive performance, as it still maintains the monotonic improvement guarantee inherited from the HAML framework. However, each update in natural order only achieves local optimality at that moment, preventing direct optimization toward the global optimum. In contrast, random order performs substantially worse, particularly on AppWorld. Unlike natural and reverse orders, which maintain fixed update logic that respects the sequential execution structure, random order completely disrupts the logical chain inherent to the execution sequence. While random update orders have been proven effective in synchronous MARL settings~\citep{HARL}, they fail in sequential multi-agent systems where the update order must be strongly coupled with the execution order. Specifically, random ordering breaks the backward induction mechanism by preventing the guarantee that subsequent turns have been updated to their optimal policies when updating an earlier turn, thereby destroying the logical dependency structure required for global optimality.

\subsubsection{Comparison on Normalization}
\label{subsec:ablation_norm}

The advantage estimation in SeeUPO uses a group-relative approach. While without any subsequent normalization is theoretically grounded, different normalization strategies may impact training stability and performance. We investigate three normalization variants:

\begin{itemize}
    \item \textbf{No normalization}: Use raw advantage estimates without any normalization, preserving drift function properties required for convergence guarantees.
    \item \textbf{Group-level normalization}: Normalize advantages within each group of joint actions sampled from the same initial state, which violates convergence guarantees similar to GSPO.
    \item \textbf{Batch-level normalization}: Normalize advantages by subtracting the batch mean and dividing by the batch standard deviation (our default setting), balancing theoretical correctness and empirical performance.
\end{itemize}

\begin{table}[h]
\centering
\caption{Comparison of normalization strategies. Columns show \textit{avg@4} and \textit{pass@4} for each benchmark, plus their averages (Avg.). All values are in percent (\%). The backbone model is Qwen3-14B. \textbf{Bolded numbers} highlight the best results.}
\label{tab:ablation_norm}
\small
\setlength{\tabcolsep}{6pt}
\begin{tabular}{@{}l cc cc cc@{}}
\toprule
\textbf{Normalization} &
\multicolumn{2}{c}{\textbf{AppWorld}} &
\multicolumn{2}{c}{\textbf{BFCL v4}} &
\multicolumn{2}{c}{\textbf{Avg.}} \\
\cmidrule(lr){2-3}\cmidrule(lr){4-5}\cmidrule(l){6-7}
 & avg@4 & pass@4 & avg@4 & pass@4 & avg@4 & pass@4 \\
\midrule
No normalization & 39.91 & 57.89 & 54.25 & 59.00 & 47.08 & 58.45 \\
Group-level normalization & 62.35 & 79.21 & \textbf{58.20} & 61.00 & 60.28 & 70.10 \\
\rowcolor{gray!20}
\textbf{Batch-level normalization (ours)} & \textbf{63.60} & \textbf{80.70} & 58.00 & \textbf{65.00} & \textbf{60.80} & \textbf{72.85} \\
\bottomrule
\end{tabular}
\end{table}

The results are presented in Table~\ref{tab:ablation_norm}. No normalization is theoretically grounded, as it preserves the drift function properties required for convergence guarantees. However, it achieves the lowest empirical performance (39.91\% avg@4 and 57.89\% pass@4 on AppWorld; 54.25\% avg@4 and 59.00\% pass@4 on BFCL v4), indicating that some form of normalization is essential for numerical stability during training. Group normalization achieves competitive test performance (62.35\% avg@4 and 79.21\% pass@4 on AppWorld; 58.20\% avg@4 and 61.00\% pass@4 on BFCL v4) but violates convergence guarantees similar to GSPO. To balance theoretical correctness and numerical stability, we 
use batch-level normalization, which achieves the best overall performance (63.60\% avg@4 and 80.70\% pass@4 on AppWorld; 58.00\% avg@4 and 65.00\% pass@4 on BFCL v4) while maintaining convergence properties. The group-relative mean subtraction (used in our advantage estimation) provides sufficient variance reduction, and batch normalization further stabilizes training without breaking theoretical guarantees.

%% file: sections/SeeUPO_Main/conclusion.tex
\section{Conclusions}
\label{sec:conclusion}

\textit{In this work, we provide a systematic theoretical analysis of mainstream backbone RL algorithms for LLM training and propose SeeUPO, a novel critic-free algorithm with convergence guarantees in multi-turn scenarios.} Specifically, we categorize existing algorithms along two dimensions---advantage estimation (GAE vs. GRAE) and policy update mechanisms (REINFORCE vs. PPU)---and analyze their convergence properties in both single-turn and multi-turn scenarios. Our analysis reveals a fundamental trade-off: mainstream algorithms cannot simultaneously achieve both critic-free operation and convergence guarantees in multi-turn settings. To address this limitation, we propose \textit{SeeUPO}, which models multi-turn interactions as sequentially-executed multi-agent bandit problems and employs reverse-order sequential policy updates. SeeUPO inherits monotonic improvement guarantees from the HAML framework and achieves convergence to global optimality through backward induction. Experimental results on AppWorld and BFCL v4 demonstrate SeeUPO's substantial improvements: relative gains of \textbf{43.3\%--54.6\%} on Qwen3-14B and \textbf{24.1\%--41.9\%} on Qwen2.5-14B (averaged across benchmarks), along with superior training stability that avoids the catastrophic failures.

\paragraph{Limitations and Future Work}
For limitations, the HAML framework theoretically requires heterogeneous policies across agents, which in traditional RL corresponds to non-shared network parameters. In the LLM context, we argue that the sufficiently large parameter space provides ample representational capacity for different turns or roles to develop functionally distinct behaviors, making the policies non-homogeneous from a turn-level perspective despite parameter sharing. This assumption becomes increasingly justified with larger-scale architectures such as Mixture-of-Experts (MoE) models, where different experts can specialize for different turns. Future work could explore explicit turn-specific parameterization to better satisfy the heterogeneity assumption.

What's more, our discussion of token-level and sequence-level RL modeling is grounded in the current mainstream paradigm of \textit{next-token prediction} for LLMs. While SeeUPO is designed within this framework, the sequence-level RL approach is essentially a meta turn-level modeling and optimization method. This approach may have potential applications in alternative paradigms such as \textit{multi-token/sentence prediction}~\citep{barrault2024large} or \textit{rectified flow diffusion} architectures~\citep{zhang2025survey}, which could be explored in future research.

%% file: sections/SeeUPO_Main/appendix.tex
\section{Mirror Learning and Multi-Agent Mirror Learning}
\label{sec:appendix_mirror_learning}

This section provides a comprehensive introduction to the Mirror Learning framework and its extension to multi-agent settings, Heterogeneous-Agent Mirror Learning (HAML). These frameworks serve as the theoretical foundation for analyzing the convergence properties of various reinforcement learning algorithms discussed in this paper.

\subsection{Mirror Learning Framework}

The Mirror Learning framework~\citep{Mirror} provides a unified theoretical foundation for analyzing the convergence of reinforcement learning algorithms. It introduces three key concepts: \textbf{drift-function}, \textbf{neighbourhood operator}, and \textbf{mirror operator}.

\paragraph{Drift Function}
The drift-function is defined as a mapping $\mathfrak{D} : \Pi \times S \;\longrightarrow\;
\bigl\{\mathfrak{D}_\pi(\,\cdot | s)\colon \mathcal{P}(A)\to\mathbb{R}\bigr\}$,
which captures the \textit{update cost} of the new policy relative to the old policy, where $\Pi$ is the policy space.

\paragraph{Neighbourhood Operator}
The neighbourhood operator is defined as $\mathcal{N} : \Pi \to \mathbb{P}(\Pi)$, where $\mathbb{P}(\Pi)$ is the power set of $\Pi$, reflecting the \textit{search space} for optimizing the new policy.

\paragraph{Mirror Operator}
The mirror operator is a mapping applied to the value function under a certain policy, expressed as:
\begin{equation}
    \left[ \mathcal{M}_{\mathfrak{D}}^{\bar{\pi}} V_\pi \right](s) = \mathbb{E}_{a \sim \bar{\pi}} \left[ A_\pi(s, a) \right] - \frac{\nu_\pi^{\bar{\pi}}(s)}{\beta_\pi(s)} \mathfrak{D}_\pi(\bar{\pi} | s),
    \label{eq:mirror operator}
\end{equation}
where $\pi$ and $\bar{\pi}$ represent the old and new policies, respectively, $A_\pi(s, a)$ denotes the advantage function, which can equivalently be replaced by $Q_\pi(s, a)$. The term $\nu_\pi^{\bar{\pi}}(s)$ is a functional that ensures continuity of expectations between the two policies under the drift context, and $\beta_\pi(s)$ represents the sampling state distribution. Lastly, $\mathfrak{D}_\pi(\bar{\pi} | s)$ quantifies the drift-function of the new policy relative to the old policy at state $s$.

\paragraph{Policy Update Rule.}
The mirror learning framework establishes theoretical convergence guarantees\footnote{In this context, convergence refers to properties such as a \textit{monotonic increase} in the expected return and \textit{eventual convergence} to an optimal policy.} for algorithms utilizing the mirror operator. These guarantees are contingent upon both the drift-function and the neighbourhood operator satisfying certain properties. The update rule is formally defined as:
\begin{equation}
    \pi_{\text{new}} = \underset{\bar{\pi} \in \mathcal{N}(\pi_{\text{old}})}{\arg\max} \mathbb{E}_{s \sim \beta_{\pi}} \left[ \left[ \mathcal{M}^{\bar{\pi}}_{\mathfrak{D}} V_{\pi_{\text{old}}} \right](s) \right].    
    \label{eq:mirror learning}
\end{equation}

\paragraph{Required Conditions for Convergence.}
Specifically, the conditions that the drift-function must satisfy are:
\begin{enumerate}
  \item \emph{(Nonnegativity)}
  For all $s \in S$, $\pi$ and $\bar{\pi} \in \Pi$,
    $\mathfrak{D}_{\pi}(\bar{\pi}\mid s) \ge 
    \mathfrak{D}_{\pi}(\pi\mid s) = 0.$
  \item \emph{(Zero gradient)}
  For all $s \in S$, $\pi$ and $\bar{\pi} \in \Pi$,
  $\nabla_{\bar{\pi}(\cdot\mid s)}
    \mathfrak{D}_{\pi}(\bar{\pi}\mid s)
    \bigl\vert_{\bar{\pi}=\pi} = 0.$
  (All its Gâteaux derivatives are zero).
\end{enumerate}

The conditions that the neighbourhood operator must satisfy are:
\begin{enumerate}
   \item \textit{(Continuity)} $\mathcal{N}$ is a continuous map.
    \item \textit{(Compactness)} Every $\mathcal{N}(\pi)$ is a compact set.
    \item \textit{(Closed ball)} There exists a metric $\chi : \Pi \times \Pi \to \mathbb{R}$, such that for all $\pi \in \Pi$, there exists $\zeta > 0$, such that $\chi(\pi, \bar{\pi}) \leqslant \zeta$ implies $\bar{\pi} \in \mathcal{N}(\pi)$.
\end{enumerate}

The mirror learning in Equation~\ref{eq:mirror learning} represents a generalized form of policy update, which is compatible with commonly used convergent algorithms such as TRPO and PPU. It is worth noting that \textbf{mirror learning assumes an accurate estimation of the value function or advantage function}. This assumption is readily satisfied in classical reinforcement learning~\citep{RL}. However, this assumption may be violated in some backbone RL algorithms used for training LLMs as we show in the above sections.

\subsection{Heterogeneous-Agent Mirror Learning}

Building upon the mirror learning framework, \citep{HARL} propose the Heterogeneous-Agent Mirror Learning (HAML) framework to address cooperative multi-agent reinforcement learning problems. HAML extends the mirror learning paradigm to the multi-agent setting while preserving theoretical guarantees of monotonic improvement and convergence to Nash Equilibrium. In our work, we leverage this theory to address convergence issues in multi-turn policy updates.

\paragraph{Multi-Agent Advantage Decomposition.}
The key challenge in the multi-agent setting lies in coordinating policy updates among multiple agents. HAML addresses this challenge by \textit{virtually} decomposing the joint advantage function into conditional advantage functions for agent subsets (multi-agent advantage decomposition lemma), and sequentially updating each agent's policy based on its contribution to the global advantage. Specifically, the multi-agent advantage decomposition lemma states that for any agent subset $i_{1:m}$:
\begin{equation}
    A^{i_{1:m}}_\pi(s, a^{i_{1:m}}) = \sum_{j=1}^{m} A^{i_j}_\pi(s, a^{i_{1:j-1}}, a^{i_j}),
    \label{eq:multi-agent advantage decomposition}
\end{equation}
where $A^{i_j}_\pi(s, a^{i_{1:j-1}}, a^{i_j}) = Q^{i_{1:j}}_\pi(s, a^{i_{1:j}}) - Q^{i_{1:j-1}}_\pi(s, a^{i_{1:j-1}})$ represents the multi-agent advantage function that evaluates the contribution of agent $i_j$ given previous agents' actions $a^{i_{1:j-1}}$. Note that this advantage decomposition is performed virtually rather than in actual execution; whether agents execute actions sequentially or simultaneously does not affect the decomposition.

\paragraph{Heterogeneous-Agent Drift Functional}
Beyond the advantage decomposition, HAML extends other concepts from mirror learning in a conditional manner. Correspondingly, HAML introduces three key components: (1) The \textbf{heterogeneous-agent drift functional} (HADF) for agent $i \in \mathcal{N}$ is defined as:
\begin{equation}
    \mathfrak{D}^i : \boldsymbol{\Pi} \times \boldsymbol{\Pi} \times \mathbb{P}(-i) \times \mathcal{S} \to \{\mathfrak{D}^i_\pi(\cdot|s, \bar{\pi}^{j_{1:m}}) : \mathcal{P}(\mathcal{A}^i) \to \mathbb{R}\},
    \label{eq:hadf}
\end{equation}
where $\boldsymbol{\Pi}$ denotes the joint policy space, $\mathbb{P}(-i)$ is the power set of agents excluding $i$, and $\mathcal{A}^i$ is the action space of agent $i$. The drift $\mathfrak{D}^i_\pi(\hat{\pi}^i|s, \bar{\pi}^{j_{1:m}})$ captures the update cost between $\pi^i$ and $\hat{\pi}^i$, given that agents $j_{1:m}$ just updated to $\bar{\pi}^{j_{1:m}}$.

\paragraph{Neighbourhood Operator}
(2) The \textbf{neighbourhood operator} $\mathcal{U}^i : \boldsymbol{\Pi} \times \Pi^i \to \mathbb{P}(\Pi^i)$ for agent $i$ is defined such that $\forall \pi^i \in \Pi^i$, $\mathcal{U}^i_\pi(\pi^i)$ contains a closed ball, \emph{i.e.}, there exists a state-wise monotonically non-decreasing metric $\chi : \Pi^i \times \Pi^i \to \mathbb{R}$ such that $\forall \pi^i \in \Pi^i$ there exists $\delta^i > 0$ such that $\chi(\pi^i, \bar{\pi}^i) \leq \delta^i \implies \bar{\pi}^i \in \mathcal{U}^i_\pi(\pi^i)$.

\paragraph{Heterogeneous-Agent Mirror Operator}
(3) The \textbf{heterogeneous-agent mirror operator} (HAMO) integrates the multi-agent advantage function with the HADF:
\begin{equation}
    \left[\mathcal{M}^{(\hat{\pi}^i)}_{\mathfrak{D}^i, \bar{\pi}^{j_{1:m}}} A_\pi\right](s) \triangleq \mathbb{E}_{a^{j_{1:m}} \sim \bar{\pi}^{j_{1:m}}, a^i \sim \hat{\pi}^i} \left[A^i_\pi(s, a^{j_{1:m}}, a^i)\right] - \mathfrak{D}^i_\pi(\hat{\pi}^i \mid s, \bar{\pi}^{j_{1:m}}),
    \label{eq:hamo}
\end{equation}
where the expectation is taken over the joint action of agents $j_{1:m}$ following their updated policies and agent $i$ following the candidate policy $\hat{\pi}^i$.

\paragraph{HAML Update Rule.}
The HAML update rule follows a sequential structure: at each iteration $k$, a random permutation $i_{1:n}$ of agents is drawn, and each agent $i_m$ updates by solving:
\begin{equation}
    \pi^{i_m}_{k+1} = \underset{\bar{\pi}^{i_m} \in \mathcal{U}^{i_m}_{\pi_k}(\pi^{i_m}_k)}{\arg\max} \; \mathbb{E}_{s \sim \beta_{\pi_k}} \left[\left[\mathcal{M}^{(\bar{\pi}^{i_m})}_{\mathfrak{D}^{i_m}, \pi^{i_{1:m-1}}_{k+1}} A_{\pi_k}\right](s)\right],
    \label{eq:haml_update_rule}
\end{equation}
where $\mathcal{U}^{i_m}_{\pi_k}(\pi^{i_m}_k)$ is the neighbourhood operator for agent $i_m$, and $\beta_{\pi_k}$ is the sampling distribution.

\paragraph{Required Conditions for HAML Convergence.}
The HADF must satisfy analogous properties to the single-agent case:
\begin{enumerate}
    \item \emph{(Nonnegativity)} For all $s \in \mathcal{S}$, joint policy $\boldsymbol{\pi} \in \boldsymbol{\Pi}$, agent $i$'s policies $\pi^i, \bar{\pi}^i \in \Pi^i$, and joint policy $\bar{\boldsymbol{\pi}}^{j_{1:m}}$ of agents $j_{1:m} \subseteq \mathcal{N} \setminus \{i\}$:
    \begin{equation}
        \mathfrak{D}^i_{\boldsymbol{\pi}}(\bar{\pi}^i \mid s, \bar{\boldsymbol{\pi}}^{j_{1:m}}) \ge \mathfrak{D}^i_{\boldsymbol{\pi}}(\pi^i \mid s, \bar{\boldsymbol{\pi}}^{j_{1:m}}) = 0.
    \end{equation}
    
    \item \emph{(Zero gradient)} For all $s \in \mathcal{S}$, joint policy $\boldsymbol{\pi} \in \boldsymbol{\Pi}$, and joint policy $\bar{\boldsymbol{\pi}}^{j_{1:m}}$ of agents $j_{1:m} \subseteq \mathcal{N} \setminus \{i\}$, all Gâteaux derivatives vanish at the reference policy:
    \begin{equation}
        \nabla_{\bar{\pi}^i(\cdot\mid s)} \mathfrak{D}^i_{\boldsymbol{\pi}}(\bar{\pi}^i \mid s, \bar{\boldsymbol{\pi}}^{j_{1:m}}) \Big|_{\bar{\pi}^i=\pi^i} = 0.
    \end{equation}
\end{enumerate}

The neighbourhood operator $\mathcal{U}^i : \boldsymbol{\Pi} \times \Pi^i \to \mathbb{P}(\Pi^i)$ must satisfy: (i) \emph{continuity} as a set-valued map, (ii) \emph{compactness} of $\mathcal{U}^i_{\boldsymbol{\pi}}(\pi^i)$ for all $\boldsymbol{\pi}, \pi^i$, and (iii) the \emph{closed ball property}: there exists a metric $\chi: \Pi^i \times \Pi^i \to \mathbb{R}$ such that for all $\pi^i \in \Pi^i$, there exists $\delta^i > 0$ with $\chi(\pi^i, \bar{\pi}^i) \leq \delta^i \Rightarrow \bar{\pi}^i \in \mathcal{U}^i_{\boldsymbol{\pi}}(\pi^i)$.

\begin{theorem}{Monotonic Improvement under HAML}{haml_monotonic}
\label{thm:haml_monotonic}
\textup{\textit{(Adapted from the Fundamental Theorem of HAML~\citep{HARL})}}
Let $\mathcal{N} = \{1, \ldots, n\}$ be the set of agents. For each agent $i \in \mathcal{N}$, let $\mathfrak{D}^i$ be a heterogeneous-agent drift functional satisfying the nonnegativity and zero gradient properties, and let $\mathcal{U}^i$ be a neighbourhood operator satisfying the closed ball property. Consider the policy sequence $\{\boldsymbol{\pi}_k\}_{k=0}^{\infty}$ generated by the HAML update rule: at each iteration $k$, agents update sequentially via:
\begin{equation}
    \pi^{i_m}_{k+1} = \argmax_{\bar{\pi}^{i_m} \in \mathcal{U}^{i_m}_{\boldsymbol{\pi}_k}(\pi^{i_m}_k)} \mathbb{E}_{s \sim \beta_{\boldsymbol{\pi}_k}} \left[\left[\mathcal{M}^{(\bar{\pi}^{i_m})}_{\mathfrak{D}^{i_m}, \boldsymbol{\pi}^{i_{1:m-1}}_{k+1}} A_{\boldsymbol{\pi}_k}\right](s)\right],
    \label{eq:haml}
\end{equation}
where $\beta_{\boldsymbol{\pi}} \in \mathcal{P}(\mathcal{S})$ is a positive sampling distribution. Then, the expected return is monotonically non-decreasing:
\begin{equation}
    J(\boldsymbol{\pi}_{k+1}) \geq J(\boldsymbol{\pi}_k), \quad \forall k \in \mathbb{N}.
\end{equation}
\end{theorem}

\paragraph{Convergence Properties.}
Under additional regularity conditions (continuity and compactness of the neighbourhood operators, continuous dependence of the sampling distribution on $\boldsymbol{\pi}$), the HAML framework~\citep{HARL} further establishes: (i) convergence of value functions $\lim_{k \to \infty} V_{\boldsymbol{\pi}_k}(s)$, and (ii) when the update order is \textbf{randomly sampled at each iteration with every permutation having positive probability}, the limit points are Nash equilibria. This random ordering requirement is crucial for the Nash equilibrium guarantee in general cooperative games. However, in our multi-turn contextual bandit setting, the problem structure is fundamentally different: the fixed sequential execution order and the absence of state transitions enable backward induction. This special structure allows us to establish a result---convergence to \textbf{global optimum} under a \textit{fixed reverse update order}, without requiring random permutations (see Section~\ref{sec:seeupo_global_convergence}).

\paragraph{Assumption on Advantage Estimation.}
HAML assumes accurate estimation of the joint advantage function $A_{\boldsymbol{\pi}}(s, \boldsymbol{a})$. In contextual bandits, the advantage equals the centered reward: $A_{\boldsymbol{\pi}}(\boldsymbol{s}_0, \boldsymbol{a}) = r(\boldsymbol{s}_0, \boldsymbol{a}) - V_{\boldsymbol{\pi}}(\boldsymbol{s}_0)$. The group relative advantage estimator in our practical method provides unbiased estimates by using the group mean reward as baseline, eliminating the structural bias that arises in MDP settings (see Section~\ref{subsec:advantage_estimation}).

\section{Global Optimality of SeeUPO}
\label{sec:seeupo_global_convergence}

Building upon the HAML framework introduced in Section~\ref{sec:appendix_mirror_learning}, we now establish the global optimality guarantee for SeeUPO. While HAML provides monotonic improvement and convergence to Nash equilibria under random update orders, the multi-turn contextual bandit structure of our problem---with its fixed execution order and absence of state transitions---enables a different result. This section proves that the \textbf{reverse update order} guarantees convergence to the \textbf{global optimum} by exploiting backward induction.

\paragraph{Motivation for the optimal continuation value.}
To formalize backward induction in the multi-turn contextual bandit, we introduce the optimal continuation value $V^*$. It converts the remaining decision horizon into a single scalar value for each history, which allows us to express the global optimal policy in a per-turn form and to link the fixed-point condition in Step 2 to global optimality in Steps 3--4.
\begin{definition}{Optimal Continuation Value}{def:optimal_continuation}
\label{def:optimal_continuation}
For reward function $r(\boldsymbol{s}_0, \boldsymbol{a}^{1:T})$, define recursively:
\begin{align}
    V^*(\boldsymbol{s}_0, \boldsymbol{a}^{1:T}) &= r(\boldsymbol{s}_0, \boldsymbol{a}^{1:T}), \\
    V^*(\boldsymbol{s}_0, \boldsymbol{a}^{1:t-1}) &= \max_{\boldsymbol{a}^t} V^*(\boldsymbol{s}_0, \boldsymbol{a}^{1:t}), \quad t = T, \ldots, 1.
\end{align}
Here $V^*$ denotes the unique function defined by the terminal condition and the backward maximization recursion. The globally optimal policy satisfies, for all $t \in \{1,\ldots,T\}$, $\pi^{*,t}(\boldsymbol{s}_0, \boldsymbol{a}^{1:t-1}) = \argmax_{\boldsymbol{a}^t} V^*(\boldsymbol{s}_0, \boldsymbol{a}^{1:t})$.
\end{definition}

\begin{theorem}{Global Optimality under Reverse Update Order}{seeupo_global_optimal}
\label{thm:seeupo_global_optimal}
Consider a multi-turn contextual bandit with execution order $1 \to 2 \to \cdots \to T$, where each turn $t$ has policy $\pi^t(\boldsymbol{a}^t | \boldsymbol{s}_0, \boldsymbol{a}^{1:t-1})$ over a compact joint policy space $\hat{\boldsymbol{\Pi}}$. Let $r(\boldsymbol{s}_0, \boldsymbol{a}^{1:T})$ be a bounded reward function with $|r(\boldsymbol{s}_0, \boldsymbol{a}^{1:T})| \leq R_{\max}$ for all $\boldsymbol{s}_0, \boldsymbol{a}^{1:T}$. For each turn $t$, let $\mathfrak{D}^{t}$ be a drift functional satisfying the nonnegativity and zero gradient properties, and let $\mathcal{U}^{t}$ be a neighbourhood operator satisfying the continuity, compactness, and closed ball properties. Suppose the advantage function is accurately estimated. Consider the policy sequence $\{\hat{\boldsymbol{\pi}}_k\}_{k=0}^{\infty}$ generated by the HAML update rule under reverse update order ($T \to T{-}1 \to \cdots \to 1$): at each iteration $k$, turn $t$ updates via
\begin{equation}
    \hat{\pi}^t_{k+1} = \argmax_{\pi^t \in \mathcal{U}^t_{\hat{\boldsymbol{\pi}}_k}(\hat{\pi}^t_k)} \mathbb{E}_{\boldsymbol{s}_0 \sim \beta_{\hat{\boldsymbol{\pi}}_k}} \left[\left[\mathcal{M}^{(\pi^t)}_{\mathfrak{D}^t, \hat{\boldsymbol{\pi}}^{t+1:T}_{k+1}} A_{\hat{\boldsymbol{\pi}}_k}\right](\boldsymbol{s}_0)\right],
\end{equation}
where $\beta_{\hat{\boldsymbol{\pi}}} \in \mathcal{P}(\mathcal{S})$ is a positive sampling distribution. Then:
\begin{enumerate}

    \item (\textbf{Value convergence}) The value sequence converges: $\exists\, V(\boldsymbol{s}_0) \text{ such that } \lim_{k \to \infty} V_{\hat{\boldsymbol{\pi}}_k}(\boldsymbol{s}_0) = V(\boldsymbol{s}_0), \quad \forall \boldsymbol{s}_0.$
    \item (\textbf{Global optimality}) Any limit point $\bar{\boldsymbol{\pi}}$ of the policy sequence is globally optimal:
    \begin{equation}
        \bar{\boldsymbol{\pi}} = \boldsymbol{\pi}^*, \quad J(\bar{\boldsymbol{\pi}}) = J^* = \max_{\hat{\boldsymbol{\pi}} \in \hat{\boldsymbol{\Pi}}} J(\hat{\boldsymbol{\pi}}).
    \end{equation}
\end{enumerate}
\end{theorem}

\begin{proof}
\textbf{Step 1 (Value Convergence).}
By Theorem 1, the policy sequence satisfies the monotonic improvement property:
\begin{equation}
    J(\hat{\boldsymbol{\pi}}_{k+1}) \geq J(\hat{\boldsymbol{\pi}}_k), \quad \forall k \in \mathbb{N}.
\end{equation}
Since the reward function is bounded by $R_{\max}$, we have $V_{\hat{\boldsymbol{\pi}}_k}(\boldsymbol{s}_0) \leq R_{\max}$ for all $\boldsymbol{s}_0$ and $k$. The sequence $\{V_{\hat{\boldsymbol{\pi}}_k}(\boldsymbol{s}_0)\}_{k=0}^{\infty}$ is monotonically non-decreasing and bounded above, hence converges:
\begin{equation}
    \exists\, V(\boldsymbol{s}_0) \text{ such that } \lim_{k \to \infty} V_{\hat{\boldsymbol{\pi}}_k}(\boldsymbol{s}_0) = V(\boldsymbol{s}_0), \quad \forall \boldsymbol{s}_0.
    \label{eq:value_convergence}
\end{equation}

\textbf{Step 2 (Characterization of Limit Points).}
By compactness of $\hat{\boldsymbol{\Pi}}$ and the Bolzano-Weierstrass theorem, the policy sequence $\{\hat{\boldsymbol{\pi}}_k\}_{k=0}^{\infty}$ has at least one limit point $\bar{\boldsymbol{\pi}}$. We now characterize the properties of such limit points by establishing that they satisfy a fixed-point condition. 

Consider a subsequence $\{\hat{\boldsymbol{\pi}}_{k_i}\}_{i=1}^{\infty}$ that converges to $\bar{\boldsymbol{\pi}}$. At each iteration $k_i$, under the reverse update order, turn $t$ updates by solving the HAML optimization problem:
\begin{equation}
    \hat{\pi}^t_{k_i+1} = \argmax_{\pi^t \in \mathcal{U}^t_{\hat{\boldsymbol{\pi}}_{k_i}}(\hat{\pi}^t_{k_i})} \mathbb{E}_{\boldsymbol{s}_0 \sim \beta_{k_i}}\left[\mathcal{M}^{(\pi^t)}_{\mathfrak{D}^t, \hat{\boldsymbol{\pi}}^{t+1:T}_{k_i+1}} A_{\hat{\boldsymbol{\pi}}_{k_i}}(\boldsymbol{s}_0)\right],
    \label{eq:haml_update_ki}
\end{equation}
where $\beta_{k_i}$ is the sampling distribution over initial states $\boldsymbol{s}_0$ at iteration $k_i$, and $\hat{\boldsymbol{\pi}}^{t+1:T}_{k_i+1}$ are the already-updated subsequent policies (from turns $t+1$ to $T$).
We now apply Berge's Maximum Theorem~\citep{ausubel1993generalized} to establish convergence of the optimization problem. The objective function in Equation~\ref{eq:haml_update_ki} is continuous in $\hat{\boldsymbol{\pi}}_{k_i}$ because:
\begin{enumerate}

    \item The advantage function $A_{\hat{\boldsymbol{\pi}}_{k_i}}(\boldsymbol{s}_0)$ is continuous in $\hat{\boldsymbol{\pi}}_{k_i}$~\citep{Mirror}.
    \item The drift functional $\mathfrak{D}^t$ is continuous by assumption.
    \item The neighbourhood operator $\mathcal{U}^t$ is continuous and compact-valued by assumption.
    \item The sampling distribution $\beta_{k_i}$ depends continuously on the policy.
\end{enumerate}

Therefore, as $i \to \infty$, by Berge's Maximum Theorem, the optimization problem in Equation~\ref{eq:haml_update_ki} converges to:
\begin{equation}
    \max_{\pi^t \in \mathcal{U}^t_{\bar{\boldsymbol{\pi}}}(\bar{\pi}^t)} \mathbb{E}_{\boldsymbol{s}_0 \sim \bar{\beta}}\left[\mathcal{M}^{(\pi^t)}_{\mathfrak{D}^t, \bar{\boldsymbol{\pi}}^{t+1:T}} A_{\bar{\boldsymbol{\pi}}}(\boldsymbol{s}_0)\right],
    \label{eq:limit_optimization}
\end{equation}
where $\bar{\beta}$ is the limiting sampling distribution over initial states. Furthermore, since $\hat{\pi}^t_{k_i+1}$ is a solution to Equation~\ref{eq:haml_update_ki} for all $i$, there exists a subsequence of $\{\hat{\pi}^t_{k_i+1}\}$ that converges to some policy $\pi'$, which is a solution to the limiting problem in Equation~\ref{eq:limit_optimization}.

\textbf{Claim:} The limit policy $\bar{\pi}^t$ is a fixed point of the HAML operator, i.e., the solution to the limiting problem satisfies $\pi' = \bar{\pi}^t$.

\textit{Proof of Claim.} We prove this by contradiction, following the approach of~\citep{Mirror}. For notational convenience, let $\bar{\boldsymbol{\pi}}_{t \leftarrow \pi'}$ denote the joint policy obtained by replacing the turn-$t$ component of $\bar{\boldsymbol{\pi}}$ with $\pi'$:
\begin{equation}
    \bar{\boldsymbol{\pi}}_{t \leftarrow \pi'} \triangleq (\bar{\pi}^1, \ldots, \bar{\pi}^{t-1}, \pi', \bar{\pi}^{t+1}, \ldots, \bar{\pi}^T).
\end{equation}

Suppose $\pi' \neq \bar{\pi}^t$. Since $\pi'$ solves the limiting optimization problem in Equation~\ref{eq:limit_optimization} and the argmax is unique on a compact set (otherwise we could take $\pi' = \bar{\pi}^t$), the objective value achieved by $\pi'$ must be strictly greater than that achieved by $\bar{\pi}^t$:
\begin{equation}
    \mathbb{E}_{\boldsymbol{s}_0 \sim \bar{\beta}}\left[\mathcal{M}^{(\pi')}_{\mathfrak{D}^t, \bar{\boldsymbol{\pi}}^{t+1:T}} A_{\bar{\boldsymbol{\pi}}}(\boldsymbol{s}_0)\right] > \mathbb{E}_{\boldsymbol{s}_0 \sim \bar{\beta}}\left[\mathcal{M}^{(\bar{\pi}^t)}_{\mathfrak{D}^t, \bar{\boldsymbol{\pi}}^{t+1:T}} A_{\bar{\boldsymbol{\pi}}}(\boldsymbol{s}_0)\right].
    \label{eq:strict_objective_improvement}
\end{equation}
By the monotonic improvement property of HAML (Theorem 1), maximizing the mirror operator guarantees policy improvement. Specifically, since $\pi'$ achieves a strictly higher mirror objective than $\bar{\pi}^t$, there exists some $\boldsymbol{s}_0$ such that:
\begin{equation}
    V_{\bar{\boldsymbol{\pi}}_{t \leftarrow \pi'}}(\boldsymbol{s}_0) > V_{\bar{\boldsymbol{\pi}}}(\boldsymbol{s}_0),
    \label{eq:strict_improvement}
\end{equation}
where the global value function is $V_{\boldsymbol{\pi}}(\boldsymbol{s}_0) = \mathbb{E}_{\boldsymbol{a}^{1:T} \sim \boldsymbol{\pi}}[r(\boldsymbol{s}_0, \boldsymbol{a}^{1:T})]$.

However, by Step 1, $V_{\bar{\boldsymbol{\pi}}}(\boldsymbol{s}_0) = \lim_{i \to \infty} V_{\hat{\boldsymbol{\pi}}_{k_i}}(\boldsymbol{s}_0) = V(\boldsymbol{s}_0)$. Furthermore, since $\pi'$ is the limit of the updated policies $\hat{\pi}^t_{k_i+1}$ (from a further subsequence), and the value sequence is monotonically non-decreasing and bounded, we have $V_{\bar{\boldsymbol{\pi}}_{t \leftarrow \pi'}}(\boldsymbol{s}_0) = \lim_{i \to \infty} V_{\hat{\boldsymbol{\pi}}_{k_i+1}}(\boldsymbol{s}_0) = V(\boldsymbol{s}_0)$. Hence $V_{\bar{\boldsymbol{\pi}}_{t \leftarrow \pi'}}(\boldsymbol{s}_0) = V_{\bar{\boldsymbol{\pi}}}(\boldsymbol{s}_0)$, which contradicts Inequality~\ref{eq:strict_improvement}. Therefore, $\pi' = \bar{\pi}^t$. \hfill $\square$

This establishes that at the limit point $\bar{\boldsymbol{\pi}}$, for each turn $t$, the policy $\bar{\pi}^t$ satisfies the \textbf{fixed-point condition}:
\begin{equation}
    \bar{\pi}^t = \argmax_{\pi^t \in \mathcal{U}^t_{\bar{\boldsymbol{\pi}}}(\bar{\pi}^t)} \mathbb{E}_{\boldsymbol{s}_0 \sim \bar{\beta}}\left[\mathcal{M}^{(\pi^t)}_{\mathfrak{D}^t, \bar{\boldsymbol{\pi}}^{t+1:T}} A_{\bar{\boldsymbol{\pi}}}(\boldsymbol{s}_0)\right].
    \label{eq:fixed_point_condition}
\end{equation}

This fixed-point condition reflects \textbf{conditional optimality}: $\bar{\pi}^t$ maximizes the HAML mirror operator given the fixed subsequent policies $\bar{\boldsymbol{\pi}}^{t+1:T}$. However, this does not immediately imply \textbf{global optimality} (i.e., $\bar{\boldsymbol{\pi}} = \boldsymbol{\pi}^*$). The key insight is that reverse update order combined with the contextual bandit structure enables us to strengthen this conditional optimality to global optimality via backward induction, which we establish in the following steps.

\textbf{Step 3 (Turn $T$ Global Optimality via Dropping the Drift).}
We first prove that the terminal turn achieves global optimality: $\bar{\pi}^T = \pi^{*,T}$. By Step 2, for any history $(\boldsymbol{s}_0, \boldsymbol{a}^{1:T-1})$, the policy $\bar{\pi}^T$ satisfies the fixed-point condition:
\begin{equation}
    \bar{\pi}^T = \argmax_{\pi^T \in \mathcal{U}^T} \left\{\mathbb{E}_{\boldsymbol{a}^T \sim \pi^T}\left[r(\boldsymbol{s}_0, \boldsymbol{a}^{1:T-1}, \boldsymbol{a}^T)\right] - \mathfrak{D}^T_{\bar{\boldsymbol{\pi}}}(\pi^T \mid \boldsymbol{s}_0)\right\}.
\end{equation}

We now apply the \textbf{dropping-the-drift argument}~\citep{Mirror} to show that this fixed point is globally optimal. Suppose for contradiction that there exists some history $(\boldsymbol{s}_0^*, \boldsymbol{a}^{1:T-1}_*)$ and action $\boldsymbol{a}^{T}_*$ such that:
\begin{equation}
    r(\boldsymbol{s}_0^*, \boldsymbol{a}^{1:T-1}_*, \boldsymbol{a}^{T}_*) > \mathbb{E}_{\boldsymbol{a}^T \sim \bar{\pi}^T}\left[r(\boldsymbol{s}_0^*, \boldsymbol{a}^{1:T-1}_*, \boldsymbol{a}^{T})\right].
    \label{eq:contradiction_assumption}
\end{equation}

The expected reward $\mathbb{E}_{\boldsymbol{a}^T \sim \pi^T}[r(\boldsymbol{s}_0^*, \boldsymbol{a}^{1:T-1}_*, \boldsymbol{a}^{T})]$ is an \textbf{affine function} of $\pi^T(\cdot | \boldsymbol{s}_0^*, \boldsymbol{a}^{1:T-1}_*)$, since it can be written as $\sum_{\boldsymbol{a}^T} \pi^T(\boldsymbol{a}^T | \cdot) \cdot r(\boldsymbol{s}_0^*, \boldsymbol{a}^{1:T-1}_*, \boldsymbol{a}^T)$. Therefore, the Gâteaux derivative in the direction toward the Dirac delta distribution $\delta_{\boldsymbol{a}^{T}_*}$ (which concentrates all probability mass on $\boldsymbol{a}^{T}_*$) is strictly positive at $\bar{\pi}^T$.

By the \textbf{zero gradient property} of the drift functional $\mathfrak{D}^{T}$, at the limit point we have:
\begin{equation}
    \nabla_{\pi^T(\cdot|\boldsymbol{s}_0^*, \boldsymbol{a}^{1:T-1}_*)} \mathfrak{D}^{T}_{\bar{\boldsymbol{\pi}}}(\pi^T \mid \boldsymbol{s}_0^*) \Big|_{\pi^T = \bar{\pi}^T} = 0.
\end{equation}

Therefore, the Gâteaux derivative of the \textbf{complete HAML objective} (expected reward minus drift) in the direction toward $\delta_{\boldsymbol{a}^{T}_*}$ is also strictly positive. By the \textbf{closed ball property} of the neighbourhood operator $\mathcal{U}^T$~\citep{HARL}, we can construct a policy $\tilde{\pi}^T \in \mathcal{U}^T_{\bar{\boldsymbol{\pi}}}(\bar{\pi}^T)$ by moving a small step from $\bar{\pi}^T$ toward $\delta_{\boldsymbol{a}^{T}_*}$ at state $(\boldsymbol{s}_0^*, \boldsymbol{a}^{1:T-1}_*)$, such that $\tilde{\pi}^T$ achieves a strictly higher objective value than $\bar{\pi}^T$. This contradicts the fixed-point condition that $\bar{\pi}^T$ maximizes the HAML objective.

Hence, the assumption in Equation~\ref{eq:contradiction_assumption} must be false. Therefore, for all histories $(\boldsymbol{s}_0, \boldsymbol{a}^{1:T-1})$:
\begin{equation}
    \bar{\pi}^T(\boldsymbol{s}_0, \boldsymbol{a}^{1:T-1}) = \argmax_{\boldsymbol{a}^T} r(\boldsymbol{s}_0, \boldsymbol{a}^{1:T}),
\end{equation}
which means $\bar{\pi}^T = \pi^{*,T}$ (the globally optimal policy for turn $T$).

\textbf{Step 4 (Backward Induction to Earlier Turns).}
We now use backward induction to extend global optimality from turn $T$ to all earlier turns. 

\emph{Base case:} Step 3 establishes $\bar{\pi}^T = \pi^{*,T}$.

\emph{Inductive hypothesis:} Assume $\bar{\pi}^{t+1:T} = \pi^{*,t+1:T}$ for some $t \in \{1, \ldots, T-1\}$.

\emph{Inductive step:} We prove $\bar{\pi}^t = \pi^{*,t}$. By the inductive hypothesis, for any history $(\boldsymbol{s}_0, \boldsymbol{a}^{1:t})$, the continuation value under $\bar{\boldsymbol{\pi}}^{t+1:T}$ equals the optimal continuation value:
\begin{equation}
    \mathbb{E}_{\boldsymbol{a}^{t+1:T} \sim \bar{\boldsymbol{\pi}}^{t+1:T}}\left[r(\boldsymbol{s}_0, \boldsymbol{a}^{1:T})\right] = V^*(\boldsymbol{s}_0, \boldsymbol{a}^{1:t}).
\end{equation}

By Step 2, $\bar{\pi}^t$ satisfies the fixed-point condition for turn $t$ given $\bar{\boldsymbol{\pi}}^{t+1:T}$. Substituting the above equality, the HAML objective for turn $t$ becomes:
\begin{equation}
    \bar{\pi}^t = \argmax_{\pi^t \in \mathcal{U}^t} \left\{\mathbb{E}_{\boldsymbol{a}^t \sim \pi^t}\left[V^*(\boldsymbol{s}_0, \boldsymbol{a}^{1:t})\right] - \mathfrak{D}^t_{\bar{\boldsymbol{\pi}}}(\pi^t \mid \boldsymbol{s}_0)\right\}.
\end{equation}

We now apply the \textbf{dropping-the-drift argument} (as in Step 3) to this optimization problem. Suppose for contradiction that there exists history $(\boldsymbol{s}_0^*, \boldsymbol{a}^{1:t-1}_*)$ and action $\boldsymbol{a}^t_*$ such that:
\begin{equation}
    V^*(\boldsymbol{s}_0^*, \boldsymbol{a}^{1:t-1}_*, \boldsymbol{a}^t_*) > \mathbb{E}_{\boldsymbol{a}^t \sim \bar{\pi}^t}\left[V^*(\boldsymbol{s}_0^*, \boldsymbol{a}^{1:t-1}_*, \boldsymbol{a}^t)\right].
\end{equation}

By the same affinity argument, zero gradient property of $\mathfrak{D}^t$, and closed ball property of $\mathcal{U}^t$, we can construct $\tilde{\pi}^t \in \mathcal{U}^t$ achieving higher objective value, contradicting the fixed-point condition. Therefore:
\begin{equation}
    \bar{\pi}^t(\boldsymbol{s}_0, \boldsymbol{a}^{1:t-1}) = \argmax_{\boldsymbol{a}^t} V^*(\boldsymbol{s}_0, \boldsymbol{a}^{1:t}) = \pi^{*,t}.
\end{equation}

By backward induction from $t = T$ down to $t = 1$, we conclude $\bar{\boldsymbol{\pi}} = \boldsymbol{\pi}^*$.

\textbf{Step 5 (Conclusion).}
Since any limit point $\bar{\boldsymbol{\pi}}$ of the policy sequence satisfies $\bar{\boldsymbol{\pi}} = \boldsymbol{\pi}^*$:
\begin{equation}
    J(\bar{\boldsymbol{\pi}}) = \mathbb{E}_{\boldsymbol{s}_0 \sim d}\left[V_{\bar{\boldsymbol{\pi}}}(\boldsymbol{s}_0)\right] = \mathbb{E}_{\boldsymbol{s}_0 \sim d}[V^*(\boldsymbol{s}_0)] = J^* = \max_{\hat{\boldsymbol{\pi}} \in \hat{\boldsymbol{\Pi}}} J(\hat{\boldsymbol{\pi}}).
\end{equation}
Therefore, all limit points are globally optimal, completing the proof.
\end{proof}

\paragraph{Why non-reverse update orders lack a global-optimality guarantee.}
Step 2 yields a fixed-point condition that is \emph{conditional} on the subsequent policies $\bar{\boldsymbol{\pi}}^{t+1:T}$. Reverse order is special because the backward-induction hypothesis ensures $\bar{\boldsymbol{\pi}}^{t+1:T}=\boldsymbol{\pi}^{*,t+1:T}$, which makes the continuation value equal to $V^*$ and allows the dropping-the-drift argument to certify global optimality. Under any other fixed order (e.g., forward order), the subsequent policies conditioned on during the update are generally not optimal, so the fixed-point condition only guarantees optimality with respect to the current continuation value rather than $V^*$. Hence the proof does not extend to global optimality for non-reverse orders, even though monotonic improvement may still hold.

\paragraph{Summary: Convergence guarantee of SeeUPO.}
Combining the results above with the advantage estimation analysis in subsequent sections, we summarize the convergence guarantee of SeeUPO as follows. Theorem 2 establishes global optimality under the assumption that the advantage function is \emph{accurately estimated}. In the \textbf{contextual bandit setting} of SeeUPO, the Advantage Estimator provides \emph{unbiased} advantage estimates by using the mean reward as baseline (see Section~\ref{subsec:advantage_estimation}). This unbiasedness, combined with the reverse update order and the multi-turn contextual bandit structure, ensures that SeeUPO satisfies all conditions of Theorem 2, thereby guaranteeing convergence to the global optimum.

For the practical instantiation \textbf{SeeUPPO-GRAE} (Section~\ref{subsec:practical_methods}), we verify that both components satisfy the required conditions. First, the SeeUPPO component adopts PPU-style clipping for policy updates: as proven in~\citep{Mirror, HARL}, the clipping mechanism corresponds to a valid drift functional (satisfying nonnegativity and zero gradient properties) and the gradient-based update defines a valid neighbourhood operator (satisfying continuity, compactness, and closed ball properties). HAPPO~\citep{HARL} has established that such PPU-style sequential updates constitute a valid HAML instantiation. Second, the GRAE component provides \emph{unbiased} advantage estimation in the contextual bandit setting: since the advantage function degenerates to $A_{\boldsymbol{\pi}}(\boldsymbol{s}_0, \boldsymbol{a}) = r(\boldsymbol{s}_0, \boldsymbol{a}) - \mathbb{E}_{\boldsymbol{a}'}[r(\boldsymbol{s}_0, \boldsymbol{a}')]$, GRAE estimates this by using the group mean reward as baseline without requiring value function approximation (see Section~\ref{subsec:advantage_estimation}). Combining the HAML-compliant SeeUPPO updates with the unbiased GRAE estimation, \textbf{SeeUPPO-GRAE satisfies all conditions of Theorem~\ref{thm:seeupo_global_optimal}}, thereby guaranteeing convergence to the global optimum.

\section{Bias of GAE}
\label{app:gae_bias}

\begin{theorem}{GAE Bias Bound}{gae_bias_bound}
\label{thm:gae_bias_bound}
Let $V^\pi: \mathcal{S} \to \mathbb{R}$ be the true value function under policy $\pi$, and $V_\phi: \mathcal{S} \to \mathbb{R}$ be the estimated value function. Define:
\begin{itemize}
    \item State-action value function: $Q^\pi(s, a) \triangleq \mathbb{E}\left[\sum_{k=0}^{\infty} \gamma^k r_{t+k} \mid s_t = s, a_t = a\right]$
    \item True advantage function: $A^\pi(s, a) \triangleq Q^\pi(s, a) - V^\pi(s)$
    \item Estimation error: $\epsilon(s) \triangleq V_\phi(s) - V^\pi(s)$
    \item Maximum error: $\epsilon_{\max} \triangleq \max_{s \in \mathcal{S}} |\epsilon(s)|$
    \item TD error: $\delta_t^{V_\phi} \triangleq r_t + \gamma V_\phi(s_{t+1}) - V_\phi(s_t)$
    \item GAE estimator: $\hat{A}_t^{\textup{GAE}} \triangleq \sum_{l=0}^{\infty} (\gamma\lambda)^l \delta_{t+l}^{V_\phi}$, where $\gamma, \lambda \in [0,1)$
\end{itemize}

Under on-policy sampling, the bias of GAE satisfies:
\begin{equation}
    \left| \mathbb{E}\left[\hat{A}_t^{\textup{GAE}} \mid s_t, a_t\right] - A^\pi(s_t, a_t) \right| \leq \frac{1 + \gamma - 2\gamma\lambda}{1 - \gamma\lambda} \cdot \epsilon_{\max}.
    \label{eq:gae_bias_bound}
\end{equation}
\end{theorem}

This appendix establishes the bias bound of Generalized Advantage Estimation (GAE). The main result (Theorem 3) shows that the bias of GAE is bounded by $\frac{1 + \gamma - 2\gamma\lambda}{1 - \gamma\lambda} \cdot \epsilon_{\max}$, where $\epsilon_{\max}$ is the maximum value function estimation error. Consequently, under perfect value function approximation ($\epsilon_{\max} = 0$), GAE provides unbiased advantage estimates. This unbiasedness condition is crucial for the convergence guarantees of GAE-based algorithms such as PPU, as analyzed in Section~\ref{app:gae_ppo_global}.

\begin{proof}
\textbf{Step 1 (GAE Expectation with True Value Function).}
Define the TD error using the true value function: $\delta_t^{V^\pi} \triangleq r_t + \gamma V^\pi(s_{t+1}) - V^\pi(s_t)$. For $l = 0$, conditioning on $(s_t, a_t)$:
\begin{align}
    \mathbb{E}[\delta_t^{V^\pi} \mid s_t, a_t] &= \mathbb{E}[r_t + \gamma V^\pi(s_{t+1}) \mid s_t, a_t] - V^\pi(s_t) \nonumber \\
    &= Q^\pi(s_t, a_t) - V^\pi(s_t) = A^\pi(s_t, a_t),
\end{align}
where the second equality follows from the Bellman equation $Q^\pi(s, a) = \mathbb{E}[r + \gamma V^\pi(s') \mid s, a]$.

For $l \geq 1$, under on-policy sampling where $a_{t+l} \sim \pi(\cdot | s_{t+l})$, applying the law of iterated expectations:
\begin{equation}
    \mathbb{E}[\delta_{t+l}^{V^\pi} \mid s_t, a_t] = \mathbb{E}_{s_{t+l}}\left[\mathbb{E}_{a_{t+l} \sim \pi}[A^\pi(s_{t+l}, a_{t+l}) \mid s_{t+l}] \mid s_t, a_t\right].
\end{equation}
Since $\mathbb{E}_{a \sim \pi(\cdot|s)}[A^\pi(s, a)] = \mathbb{E}_{a \sim \pi}[Q^\pi(s, a) - V^\pi(s)] = V^\pi(s) - V^\pi(s) = 0$, we have $\mathbb{E}[\delta_{t+l}^{V^\pi} \mid s_t, a_t] = 0$ for all $l \geq 1$. Therefore:
\begin{equation}
    \mathbb{E}\left[\sum_{l=0}^{\infty} (\gamma\lambda)^l \delta_{t+l}^{V^\pi} \mid s_t, a_t\right] = (\gamma\lambda)^0 \cdot A^\pi(s_t, a_t) + \sum_{l=1}^{\infty} (\gamma\lambda)^l \cdot 0 = A^\pi(s_t, a_t).
    \label{eq:gae_true_expectation}
\end{equation}

\textbf{Step 2 (Bias Computation).}
Substituting $V_\phi(s) = V^\pi(s) + \epsilon(s)$ into the TD error definition:
\begin{align}
    \delta_t^{V_\phi} &= r_t + \gamma V_\phi(s_{t+1}) - V_\phi(s_t) \nonumber \\
    &= r_t + \gamma [V^\pi(s_{t+1}) + \epsilon(s_{t+1})] - [V^\pi(s_t) + \epsilon(s_t)] \nonumber \\
    &= \underbrace{[r_t + \gamma V^\pi(s_{t+1}) - V^\pi(s_t)]}_{\delta_t^{V^\pi}} + \gamma\epsilon(s_{t+1}) - \epsilon(s_t).
\end{align}
Substituting into the GAE definition and taking conditional expectation:
\begin{align}
    \mathbb{E}[\hat{A}_t^{\textup{GAE}} \mid s_t, a_t] &= \mathbb{E}\left[\sum_{l=0}^{\infty} (\gamma\lambda)^l \delta_{t+l}^{V_\phi} \mid s_t, a_t\right] \nonumber \\
    &= \mathbb{E}\left[\sum_{l=0}^{\infty} (\gamma\lambda)^l \delta_{t+l}^{V^\pi} \mid s_t, a_t\right] + \sum_{l=0}^{\infty} (\gamma\lambda)^l \mathbb{E}[\gamma\epsilon(s_{t+l+1}) - \epsilon(s_{t+l}) \mid s_t, a_t].
\end{align}
By Equation~\ref{eq:gae_true_expectation}, the first term equals $A^\pi(s_t, a_t)$. We analyze the second term involving estimation errors:
\begin{equation}
    S \triangleq \sum_{l=0}^{\infty} (\gamma\lambda)^l \mathbb{E}[\gamma\epsilon(s_{t+l+1}) - \epsilon(s_{t+l}) \mid s_t, a_t].
\end{equation}
Expanding this sum and regrouping by coefficients of each $\epsilon(s_{t+m})$ term:
\begin{align}
    S &= \underbrace{-\epsilon(s_t)}_{l=0} + \sum_{m=1}^{\infty} \left[ (\gamma\lambda)^{m-1}\gamma - (\gamma\lambda)^m \right] \mathbb{E}[\epsilon(s_{t+m}) \mid s_t, a_t] \nonumber \\
    &= -\epsilon(s_t) + \sum_{m=1}^{\infty} (\gamma\lambda)^{m-1} \gamma (1 - \lambda) \mathbb{E}[\epsilon(s_{t+m}) \mid s_t, a_t] \nonumber \\
    &= -\epsilon(s_t) + (1-\lambda) \sum_{m=1}^{\infty} \gamma^m \lambda^{m-1} \mathbb{E}[\epsilon(s_{t+m}) \mid s_t, a_t].
\end{align}
Note that the coefficient for the $m$-th term is $\gamma^m \lambda^{m-1} (1-\lambda)$, which is well-defined for all $\lambda \in [0, 1)$. For $\lambda = 0$, the series reduces to the single term $\gamma \epsilon(s_{t+1})$, consistent with GAE$(\gamma, 0)$ being the one-step TD error.

Therefore, the bias is:
\begin{equation}
    \mathbb{E}[\hat{A}_t^{\textup{GAE}} \mid s_t, a_t] - A^\pi(s_t, a_t) = -\epsilon(s_t) + (1-\lambda) \sum_{m=1}^{\infty} \gamma^m \lambda^{m-1} \mathbb{E}[\epsilon(s_{t+m}) \mid s_t, a_t].
\end{equation}

\textbf{Step 3 (Bound).}
Applying the triangle inequality and $|\epsilon(s)| \leq \epsilon_{\max}$:
\begin{align}
    |\textup{Bias}| &\leq |\epsilon(s_t)| + (1-\lambda) \sum_{m=1}^{\infty} \gamma^m \lambda^{m-1} |\mathbb{E}[\epsilon(s_{t+m}) \mid s_t, a_t]| \nonumber \\
    &\leq \epsilon_{\max} + \epsilon_{\max} (1-\lambda) \gamma \sum_{k=0}^{\infty} (\gamma\lambda)^k \quad (\text{let } k = m - 1) \nonumber \\
    &= \epsilon_{\max} + \epsilon_{\max} \cdot \frac{\gamma(1-\lambda)}{1 - \gamma\lambda} \nonumber \\
    &= \epsilon_{\max} \left( 1 + \frac{\gamma - \gamma\lambda}{1 - \gamma\lambda} \right) = \frac{1 + \gamma - 2\gamma\lambda}{1 - \gamma\lambda} \cdot \epsilon_{\max}.
\end{align}
\end{proof}

\paragraph{Discussion.}
Equation~\ref{eq:gae_bias_bound} reveals several important properties of GAE bias:

\begin{enumerate}
    \item \textbf{Unbiasedness under perfect estimation.} When the value function is perfectly estimated, i.e., $\epsilon_{\max} = 0$, the bound becomes zero, implying that GAE provides unbiased advantage estimates. This is the key assumption underlying the convergence guarantees of GAE-based algorithms such as PPU.
    
    \item \textbf{Linear dependence on estimation error.} The bias scales linearly with $\epsilon_{\max}$, indicating that reducing value function approximation error directly reduces advantage estimation bias.
    
    \item \textbf{Effect of $\lambda$.} Define $C(\gamma, \lambda) \triangleq \frac{1 + \gamma - 2\gamma\lambda}{1 - \gamma\lambda}$. As $\lambda \to 1$, $C(\gamma, \lambda) \to 1$, minimizing the bias amplification. As $\lambda \to 0$, $C(\gamma, \lambda) \to 1 + \gamma$, increasing the bound. Thus, larger $\lambda$ values reduce bias sensitivity to value estimation errors.
    
    \item \textbf{Effect of $\gamma$.} For fixed $\lambda$, larger $\gamma$ generally increases $C(\gamma, \lambda)$, making long-horizon tasks more susceptible to bias from value estimation errors.
\end{enumerate}

\section{Bias of GRAE}
\label{app:grae_bias_gradient}

This appendix provides detailed proofs for Theorem 4 regarding the bias and gradient (un)biasedness of GRAE in general MDPs. The key messages are:
\begin{itemize}
    \item The GRAE advantage estimator is structurally biased because it uses the group mean $\bar{R}$ (an $s_0$-level baseline) for all states.
    \item The corresponding policy-gradient estimator is \emph{unbiased} under the undiscounted objective ($\gamma=1$) with total return.
    \item When the objective uses discounting ($\gamma<1$), the GRAE gradient becomes biased.
\end{itemize}
The contextual bandit case is special (GRAE becomes unbiased) and is analyzed separately in Section~\ref{app:grae_bandit_gspo}.

\begin{theorem}{GRAE Bias and Gradient (Un)biasedness in MDPs}{grae_bias_gradient}
\label{thm:grae_bias_gradient}
Consider a finite-horizon MDP and a policy $\pi_\theta$. For a fixed initial state $s_0$, sample $N$ i.i.d.\ trajectories and define the total return for trajectory $\tau^{(i)}$ as $R^{(i)}$. Define the GRAE estimator $\hat{A}^{\textup{GRAE}}(s_t,a_t)=R^{(i)}-\bar{R}$ with $\bar{R}=\frac{1}{N}\sum_{i=1}^N R^{(i)}$. Then:
\begin{enumerate}
    \item (\textbf{Structural bias}) For any $(s_t,a_t)$, $\mathbb{E}[\hat{A}^{\textup{GRAE}} \mid s_t,a_t]=Q(s_t,a_t)-V(s_0)$, so the bias is $V(s_t)-V(s_0)$, which is nonzero in general.
    \item (\textbf{Gradient unbiasedness under undiscounted objective}) When $\gamma=1$, the policy-gradient estimator $g_{\textup{GRAE}}=\sum_{t=0}^{T}\nabla_\theta \log \pi_\theta(a_t\mid s_t)\,\hat{A}^{\textup{GRAE}}$ satisfies $\mathbb{E}[g_{\textup{GRAE}}]=\nabla_\theta J(\theta)$ for the undiscounted objective.
\end{enumerate}
\end{theorem}

\subsection{Baseline Invariance Lemma}

\begin{lemma}{Baseline Invariance}{baseline}
\label{lem:baseline}
\textup{\textit{(Adapted from the Theorem of RL~\citep{RL})}}
For any function $b(s_t)$ that depends only on the state $s_t$ and not on the current action $a_t$:
\begin{equation}
\mathbb{E}_{a_t \sim \pi_\theta(\cdot \mid s_t)}\left[\nabla_\theta \log \pi_\theta(a_t \mid s_t) \cdot b(s_t)\right] = 0.
\end{equation}
\end{lemma}

\begin{proof}
\begin{align}
\mathbb{E}_{a_t \sim \pi_\theta(\cdot \mid s_t)}\left[\nabla_\theta \log \pi_\theta(a_t \mid s_t) \cdot b(s_t)\right]
&= b(s_t) \sum_{a_t} \pi_\theta(a_t \mid s_t) \cdot \frac{\nabla_\theta \pi_\theta(a_t \mid s_t)}{\pi_\theta(a_t \mid s_t)} \\
&= b(s_t) \sum_{a_t} \nabla_\theta \pi_\theta(a_t \mid s_t) \\
&= b(s_t) \cdot \nabla_\theta \underbrace{\sum_{a_t} \pi_\theta(a_t \mid s_t)}_{=1} \\
&= 0.
\end{align}
\end{proof}

\subsection{Proof of GRAE Bias and Gradient Unbiasedness}

\begin{proof}
\textbf{Part 1: Bias of GRAE.}

Consider a response $\tau^{(j)}$ containing state-action pair $(s_t, a_t)$ with reward $R^{(j)}$. For this trajectory, taking the conditional expectation over future trajectories given $(s_t, a_t)$:
\begin{align}
\mathbb{E}_{\tau_{>t} \sim \pi_\theta}[\hat{A}_{\text{GRAE}}(s_t, a_t) \mid s_t, a_t] 
&= \mathbb{E}_{\tau_{>t} \sim \pi_\theta}[R^{(j)} \mid s_t, a_t] - \bar{R} \\
&= Q(s_t, a_t) - \bar{R}.
\end{align}
Note that $\bar{R}$ is computed from other responses in the group and is independent of the future actions in the current response given $s_t$.

As $N \to \infty$, by the law of large numbers:
\begin{equation}
\bar{R} \xrightarrow{p} \mathbb{E}_{\tau \sim \pi_\theta(\cdot \mid s_0)}[R] = V(s_0),
\end{equation}
where $R$ denotes the reward of a response starting from $s_0$.

Thus, in the large sample limit:
\begin{equation}
\mathbb{E}_{\tau_{>t} \sim \pi_\theta}[\hat{A}_{\text{GRAE}}(s_t, a_t) \mid s_t, a_t] \to Q(s_t, a_t) - V(s_0).
\end{equation}

The bias is computed as:
\begin{align}
\text{Bias}(s_t, a_t) 
&= \mathbb{E}_{\tau_{>t} \sim \pi_\theta}[\hat{A}_{\text{GRAE}} \mid s_t, a_t] - A_{\text{true}}(s_t, a_t) \\
&= [Q(s_t, a_t) - V(s_0)] - [Q(s_t, a_t) - V(s_t)] \\
&= V(s_t) - V(s_0) = \Delta(s_t).
\end{align}

In general, $V(s_t) \neq V(s_0)$ since the value function evolves as the trajectory unfolds. Therefore, $\hat{A}_{\text{GRAE}}$ is biased.

\textbf{Part 2: Unbiasedness of Policy Gradient.}

For a response $\tau^{(j)}$ with reward $R^{(j)}$, define the gradient estimators:
\begin{align}
g_{\text{true}} &= \sum_{t=0}^{T} \nabla_\theta \log \pi_\theta(a_t \mid s_t) \cdot (R^{(j)} - V(s_t)), \\
g_{\text{GRAE}} &= \sum_{t=0}^{T} \nabla_\theta \log \pi_\theta(a_t \mid s_t) \cdot (R^{(j)} - \bar{R}),
\end{align}
where the sum is over all state-action pairs $(s_t, a_t)$ in response $\tau^{(j)}$, and all pairs share the same reward $R^{(j)}$ and group mean $\bar{R}$.

As $N \to \infty$, $\bar{R} \to V(s_0)$. Their difference becomes:
\begin{align}
g_{\text{GRAE}} - g_{\text{true}} 
&= \sum_{t=0}^{T} \nabla_\theta \log \pi_\theta(a_t \mid s_t) \cdot (V(s_t) - V(s_0)).
\end{align}

The term $V(s_t) - V(s_0)$ depends on the state history $(s_0, a_0, \ldots, a_{t-1})$ but not on the current action $a_t$. Thus, it can be treated as $b(s_t)$.

Taking expectations using the law of iterated expectations:
\begin{align}
\mathbb{E}_{\tau \sim \pi_\theta}[g_{\text{GRAE}} - g_{\text{true}}] 
&= \sum_{t=0}^{T} \mathbb{E}_{s_t \sim \pi_\theta}\left[\mathbb{E}_{a_t \sim \pi_\theta(\cdot \mid s_t)}\left[\nabla_\theta \log \pi_\theta(a_t \mid s_t) \cdot (V(s_t) - V(s_0))\right]\right] \\
&= \sum_{t=0}^{T} \mathbb{E}_{s_t \sim \pi_\theta}[0] \quad \text{(by Lemma 1)} \\
&= 0.
\end{align}

Since $\mathbb{E}_{\tau \sim \pi_\theta}[g_{\text{true}}] = \nabla_\theta J(\theta)$ by the policy gradient theorem, we conclude:
\begin{equation}
\mathbb{E}_{\tau \sim \pi_\theta}[g_{\text{GRAE}}] = \nabla_\theta J(\theta).
\end{equation}
\end{proof}

\subsection{Proof of GRAE Gradient Bias When $\gamma \neq 1$}

As mentioned in the key messages at the beginning of this section, when the objective uses discounting ($\gamma < 1$), the GRAE gradient becomes biased. This subsection provides a detailed analysis of why this bias arises.

When the objective uses discount factor $\gamma \in (0, 1)$, the GRAE gradient estimator based on total return is biased in general. We now provide a detailed analysis.

Let the discounted return from time $t$ be $G_t^\gamma = \sum_{k=0}^{T-t-1} \gamma^k r_{t+k+1}$, and let $V_\gamma(s) = \mathbb{E}[G_0^\gamma \mid s_0=s]$. The true gradient estimator for the discounted objective is
\begin{equation}
g_{\text{true}} = \sum_{t=0}^{T} \nabla_\theta \log \pi_\theta(a_t \mid s_t) \cdot \left(G_t^\gamma - V_\gamma(s_t)\right).
\end{equation}
GRAE instead uses the total return $R=\sum_{k=0}^{T-1} r_{k+1}$ and the group mean $\bar{R}$:
\begin{equation}
g_{\text{GRAE}} = \sum_{t=0}^{T} \nabla_\theta \log \pi_\theta(a_t \mid s_t) \cdot (R - \bar{R}).
\end{equation}
In the large-sample limit, $\bar{R} \to \mathbb{E}[R \mid s_0]$. Their difference is
\begin{align}
g_{\text{GRAE}} - g_{\text{true}} 
&= \sum_{t=0}^{T} \nabla_\theta \log \pi_\theta(a_t \mid s_t) \cdot \left(R - G_t^\gamma + V_\gamma(s_t) - \mathbb{E}[R \mid s_0]\right).
\end{align}
By Lemma 1, the term $V_\gamma(s_t) - \mathbb{E}[R \mid s_0]$ is a baseline and vanishes in expectation. The remaining term can be expanded as
\begin{equation}
R - G_t^\gamma = \sum_{k=0}^{t-1} r_{k+1} + \sum_{k=1}^{T-t-1} (1 - \gamma^k) r_{t+k+1}.
\end{equation}
The first sum depends only on the past and can be treated as a baseline given $s_t$. The second sum depends on future rewards (and thus on $a_t$) with positive coefficients $(1-\gamma^k)$, so its contribution to the expected gradient is nonzero in general. Hence $g_{\text{GRAE}}$ is biased when $\gamma \neq 1$.

\paragraph{Discussion.}
In multi-turn scenarios, GRAE faces additional challenges:
\begin{enumerate}
    \item \textbf{Credit assignment problems:} Using group mean reward computed from initial state $s_0$ to estimate advantages for states $s_t$ at later turns creates credit assignment issues.
    \item \textbf{Amplified structural bias:} The structural bias $|V(s_t) - V(s_0)|$ grows with the number of turns.
    \item \textbf{Gradient bias:} When $\gamma < 1$ (which is the more common setting in practical MDPs), the gradient estimator becomes biased, compounding the problems above.
\end{enumerate}

Therefore, in multi-turn settings, directly applying GRAE, especially in token-level RL, is not suitable.

\section{Detailed Proofs for GRAE-REINFORCE Convergence}
\label{app:grae_reinforce_global}

This appendix provides detailed proofs for Theorem 5 regarding the convergence of GRAE-REINFORCE. The key insight is that REINFORCE can be viewed as an instance of the Mirror Learning framework~\citep{Mirror}, specifically as a parameterized implementation of Generalized Policy Iteration (GPI). Combined with the gradient unbiasedness of GRAE under the undiscounted objective ($\gamma=1$, Theorem 4), we establish the convergence guarantee.

\subsection{REINFORCE as an Instance of Mirror Learning}

We first establish that REINFORCE fits within the Mirror Learning framework. Recall from Section~\ref{sec:appendix_mirror_learning} that Mirror Learning requires: (1) a drift function $\mathfrak{D}$, (2) a neighbourhood operator $\mathcal{N}$, and (3) a sampling distribution $\beta_\pi$.

\begin{lemma}{REINFORCE as Mirror Learning}{reinforce_mirror}
\label{lem:reinforce_mirror}
REINFORCE is an instance of Mirror Learning with the following components:
\begin{enumerate}
    \item \textbf{Drift function:} $\mathfrak{D} \equiv 0$ (trivial drift).
    \item \textbf{Neighbourhood operator:} $\mathcal{N} = \Pi$ (trivial neighbourhood).
    \item \textbf{Sampling distribution:} $\beta_{\pi} = \rho_{\pi}$, the state visitation distribution under the current policy.
\end{enumerate}
\end{lemma}

\begin{proof}
By~\citep{Mirror}, when $\mathfrak{D} \equiv 0$ and $\mathcal{N} = \Pi$, the Mirror Learning update reduces to GPI:
\begin{equation}
    \pi_{\text{new}}(\cdot|s) = \argmax_{\bar{\pi}(\cdot|s) \in \mathcal{P}(A)} \mathbb{E}_{a \sim \bar{\pi}}[Q_{\pi_{\text{old}}}(s, a)], \quad \forall s \in S.
\end{equation}
REINFORCE is the parameterized, stochastic gradient implementation of GPI. The policy gradient theorem~\citep{RL} shows that:
\begin{equation}
    \nabla_\theta J(\theta) = \mathbb{E}_{s \sim \rho_{\pi_\theta}, a \sim \pi_\theta} \left[ \nabla_\theta \log \pi_\theta(a|s) \cdot Q_{\pi_\theta}(s, a) \right],
\end{equation}
which is exactly the gradient of the GPI objective weighted by the state visitation distribution $\rho_{\pi_\theta}$. Hence, REINFORCE approximately solves the GPI step via gradient ascent.
\end{proof}

\subsection{Convergence Theorem}

\begin{theorem}{Convergence of GRAE-REINFORCE}{grae_reinforce_global}
\label{thm:grae_reinforce_global}
Consider a finite-horizon MDP with $\gamma=1$ and $|r(s,a)| \leq R_{\max}$ for all $(s,a)$. Let $\pi_\theta$ be a parameterized policy over a compact parameter space $\Theta$. Suppose the policy is updated by REINFORCE using the GRAE advantage estimator. Then:
\begin{enumerate}
    \item (\textbf{Gradient unbiasedness}) $\mathbb{E}_{\tau \sim \pi_\theta}[g_{\textup{GRAE}}] = \nabla_\theta J(\theta)$.
    \item (\textbf{Monotonic improvement}) $J(\pi_{k+1}) \geq J(\pi_k), \quad \forall k \in \mathbb{N}$.
    \item (\textbf{Convergence}) $\exists\, J^* \text{ such that } \lim_{k \to \infty} J(\pi_k) = J^*$.
\end{enumerate}
\end{theorem}

\begin{proof}
\textbf{Step 1 (Gradient Unbiasedness).} By Theorem 4, under $\gamma=1$, the GRAE-based gradient estimator satisfies:
\begin{equation}
    \mathbb{E}_{\tau \sim \pi_\theta}[g_{\textup{GRAE}}] = \nabla_\theta J(\theta).
\end{equation}
This follows from the baseline invariance property (Lemma 1): adding or subtracting any state-dependent baseline from the return does not change the expected gradient.

\textbf{Step 2 (Monotonic Improvement and Convergence).} By Lemma 2, REINFORCE is an instance of Mirror Learning with trivial drift $\mathfrak{D} \equiv 0$ and trivial neighbourhood $\mathcal{N} = \Pi$. The trivial drift satisfies nonnegativity ($\mathfrak{D}_\pi(\bar{\pi}|s) = 0 \geq 0$) and zero gradient ($\nabla_{\bar{\pi}} \mathfrak{D}_\pi(\bar{\pi}|s)|_{\bar{\pi}=\pi} = 0$). By the Fundamental Theorem of Mirror Learning~\citep{Mirror}, the policy sequence $\{\pi_k\}_{k=0}^{\infty}$ satisfies:
\begin{equation}
    J(\pi_{k+1}) \geq J(\pi_k), \quad \forall k \in \mathbb{N},
\end{equation}
and
\begin{equation}
    \lim_{k \to \infty} J(\pi_k) = J^*.
\end{equation}
\end{proof}

\section{Detailed Proofs for GAE-PPU Convergence}
\label{app:gae_ppo_global}

This appendix provides detailed proofs for Theorem 6 regarding the convergence of GAE-PPU. The key insight is that PPU can be viewed as an instance of the Mirror Learning framework~\citep{Mirror} with a non-trivial drift function derived from the clipping objective. Combined with the unbiasedness of GAE under perfect value function approximation (Theorem 3), we establish the convergence guarantee.

\subsection{PPU as an Instance of Mirror Learning}

We first establish that PPU fits within the Mirror Learning framework. Recall from Section~\ref{sec:appendix_mirror_learning} that Mirror Learning requires: (1) a drift function $\mathfrak{D}$, (2) a neighbourhood operator $\mathcal{N}$, and (3) a sampling distribution $\beta_\pi$.

\begin{lemma}{PPU as Mirror Learning}{ppo_mirror}
\label{lem:ppo_mirror}
PPU is an instance of Mirror Learning with the following components:
\begin{enumerate}
    \item \textbf{Drift function:} $\mathfrak{D}_{\pi}^{\textup{PPU}}(\bar{\pi} | s) = \mathbb{E}_{a \sim \pi} \left[ \operatorname{ReLU} \left( \left[ r(\bar{\pi}) - \operatorname{clip}(r(\bar{\pi}), 1 \pm \epsilon) \right] A_{\pi}(s, a) \right) \right]$,
    where $r(\bar{\pi}) = \frac{\bar{\pi}(a|s)}{\pi(a|s)}$.
    \item \textbf{Neighbourhood operator:} $\mathcal{N} = \Pi$ (trivial neighbourhood).
    \item \textbf{Sampling distribution:} $\beta_{\pi} = \bar{\rho}_{\pi}$, the normalized marginal discounted state distribution.
\end{enumerate}
\end{lemma}

\begin{proof}
We derive the drift function by reformulating the PPU clipping objective. Starting from the PPU update rule:
\begin{equation}
    \pi_{\text{new}}^{\text{PPU}} = \underset{\bar{\pi} \in \Pi}{\arg\max}
    \mathbb{E}_{s \sim \rho_{\pi},a \sim \pi} \left[ \min \left( r(\bar{\pi}) A_{\pi}(s, a), \operatorname{clip}(r(\bar{\pi}), 1 \pm \epsilon) A_{\pi}(s, a) \right) \right].
\end{equation}

The complete derivation proceeds as follows:
\begin{equation}
\begin{aligned}
    \pi_{\text{new}}^{\text{PPU}}
    &= \underset{\bar{\pi} \in \Pi}{\arg\max}
    \mathbb{E}_{s \sim \rho_{\pi},a \sim \pi} \left[ \min \left( r(\bar{\pi}) A_{\pi}(s, a), \operatorname{clip}(r(\bar{\pi}), 1 \pm \epsilon) A_{\pi}(s, a) \right) \right] \\
    &= \underset{\bar{\pi} \in \Pi}{\arg\max} \mathbb{E}_{s \sim \rho_{\pi}} 
    \left[ 
    \mathbb{E}_{a \sim \pi} \left[ \min \left( r(\bar{\pi}) A_{\pi}(s, a), \operatorname{clip}(r(\bar{\pi}), 1 \pm \epsilon) A_{\pi}(s, a) \right) \right] \right. \\
    &\quad \left. - \mathbb{E}_{a \sim \bar{\pi}} \left[A_{\pi}(s, a) \right]
     + \mathbb{E}_{a \sim \bar{\pi}} \left[A_{\pi}(s, a) \right]
    \right] \\
    &= \underset{\bar{\pi} \in \Pi}{\arg\max} \mathbb{E}_{s \sim \rho_{\pi}} 
    \left[ 
    \mathbb{E}_{a \sim \pi} \left[ \min \left( r(\bar{\pi}) A_{\pi}(s, a), \operatorname{clip}(r(\bar{\pi}), 1 \pm \epsilon) A_{\pi}(s, a) \right) \right] \right. \\
    &\quad \left. - \mathbb{E}_{a \sim \pi} \left[r(\bar{\pi}) A_{\pi}(s, a) \right]
     + \mathbb{E}_{a \sim \bar{\pi}} \left[A_{\pi}(s, a) \right]
    \right] \\
    &= \underset{\bar{\pi} \in \Pi}{\arg\max} \mathbb{E}_{s \sim \rho_{\pi}} 
    \left[ 
    \mathbb{E}_{a \sim \bar{\pi}} \left[A_{\pi}(s, a) \right]
     - \left( \mathbb{E}_{a \sim \pi} \left[r(\bar{\pi}) A_{\pi}(s, a) \right] \right. \right. \\
    &\quad \left. \left. - \mathbb{E}_{a \sim \pi} \left[ \min \left( r(\bar{\pi}) A_{\pi}(s, a), \operatorname{clip}(r(\bar{\pi}), 1 \pm \epsilon) A_{\pi}(s, a) \right) \right] \right)
    \right] \\
    &= \underset{\bar{\pi} \in \Pi}{\arg\max} \mathbb{E}_{s \sim \rho_{\pi}} 
    \left[ 
    \mathbb{E}_{a \sim \bar{\pi}} \left[A_{\pi}(s, a) \right]
     - \mathbb{E}_{a \sim \pi} \left( r(\bar{\pi}) A_{\pi}(s, a) \right. \right. \\
    &\quad \left. \left. - \min \left( r(\bar{\pi}) A_{\pi}(s, a), \operatorname{clip}(r(\bar{\pi}), 1 \pm \epsilon) A_{\pi}(s, a) \right) \right)
    \right] \\
    &= \underset{\bar{\pi} \in \Pi}{\arg\max} \mathbb{E}_{s \sim \rho_{\pi}} 
    \left[ 
    \mathbb{E}_{a \sim \bar{\pi}} \left[A_{\pi}(s, a) \right]
     - \mathbb{E}_{a \sim \pi} \left( \max \left( 0, r(\bar{\pi}) A_{\pi}(s, a) \right. \right. \right. \\
    &\quad \left. \left. \left. - \operatorname{clip}(r(\bar{\pi}), 1 \pm \epsilon) A_{\pi}(s, a) \right) \right) 
    \right] \\
    &= \underset{\bar{\pi} \in \Pi}{\arg\max} \mathbb{E}_{s \sim \rho_{\pi}} 
    \left[ 
    \mathbb{E}_{a \sim \bar{\pi}} \left[A_{\pi}(s, a) \right]
     - \mathbb{E}_{a \sim \pi} \left( \max \left( 0, \left[ r(\bar{\pi}) - \operatorname{clip}(r(\bar{\pi}), 1 \pm \epsilon) \right] A_{\pi}(s, a) \right) \right) 
    \right] \\
    &= \underset{\bar{\pi} \in \Pi}{\arg\max} \mathbb{E}_{s \sim \rho_{\pi}} 
    \left[ 
    \mathbb{E}_{a \sim \bar{\pi}} \left[A_{\pi}(s, a) \right]
     - \mathbb{E}_{a \sim \pi} \left[ \operatorname{ReLU} \left( \left[ r(\bar{\pi}) - \operatorname{clip}(r(\bar{\pi}), 1 \pm \epsilon) \right] A_{\pi}(s, a) \right) \right] 
    \right].
\end{aligned}
\label{eq:PPO-drift-derivation-detailed}
\end{equation}

Comparing with the Mirror Learning update (Equation~\ref{eq:mirror learning}), we identify the drift function as:
\begin{equation}
\mathfrak{D}_{\pi}^{\text{PPU}}(\bar{\pi} | s) = \mathbb{E}_{a \sim \pi} \left[ \operatorname{ReLU} \left( \left[ r(\bar{\pi}) - \operatorname{clip}(r(\bar{\pi}), 1 \pm \epsilon) \right] A_{\pi}(s, a) \right) \right].
\label{eq:PPO-drift-detailed}
\end{equation}

We now verify that $\mathfrak{D}_{\pi}^{\text{PPU}}$ satisfies the required properties:

\textbf{Nonnegativity:} Since $\operatorname{ReLU}(x) = \max(0, x) \geq 0$ for all $x \in \mathbb{R}$:
\begin{equation}
\mathfrak{D}_{\pi}^{\text{PPU}}(\bar{\pi} | s) = \mathbb{E}_{a \sim \pi} \left[ \operatorname{ReLU}(\cdot) \right] \geq 0.
\end{equation}
When $\bar{\pi} = \pi$, we have $r(\bar{\pi}) = 1$ and $\operatorname{clip}(1, 1 \pm \epsilon) = 1$, so:
\begin{equation}
\mathfrak{D}_{\pi}^{\text{PPU}}(\pi | s) = \mathbb{E}_{a \sim \pi} \left[ \operatorname{ReLU}(0) \right] = 0.
\end{equation}

\textbf{Zero gradient:} When $\bar{\pi} = \pi$, the argument of ReLU is zero:
\begin{equation}
\left[ r(\bar{\pi}) - \operatorname{clip}(r(\bar{\pi}), 1 \pm \epsilon) \right] A_{\pi}(s, a) \Big\vert_{\bar{\pi}=\pi} = [1 - 1] A_{\pi}(s, a) = 0.
\end{equation}
Since $\operatorname{ReLU}(0) = 0$ and the ReLU function has subgradient $0$ at $x = 0$:
\begin{equation}
\nabla_{\bar{\pi}(\cdot|s)} \mathfrak{D}_{\pi}^{\text{PPU}}(\bar{\pi}|s) \big\vert_{\bar{\pi}=\pi} = 0.
\end{equation}

The trivial neighbourhood operator $\mathcal{N} = \Pi$ satisfies continuity, compactness, and the closed ball property. Hence, PPU is a valid instance of Mirror Learning.
\end{proof}

\subsection{Convergence Theorem}

\begin{theorem}{Convergence of GAE-PPU}{gae_ppo_global}
\label{thm:gae_ppo_global}
Consider an MDP with $|r(s,a)| \leq R_{\max}$ for all $(s,a)$. Let $\pi_\theta$ be a parameterized policy over a compact parameter space $\Theta$. Suppose the policy is updated by PPU using the GAE advantage estimator with a perfect value function approximation ($V_\phi = V^\pi$). Then:
\begin{enumerate}
    \item (\textbf{Advantage unbiasedness}) $\mathbb{E}[\hat{A}_t^{\textup{GAE}} \mid s_t, a_t] = A^\pi(s_t, a_t)$.
    \item (\textbf{Monotonic improvement}) $J(\pi_{k+1}) \geq J(\pi_k), \quad \forall k \in \mathbb{N}$.
    \item (\textbf{Convergence}) $\lim_{k \to \infty} J(\pi_k) = J^*$.
\end{enumerate}
\end{theorem}

\begin{proof}
\textbf{Step 1 (Advantage Unbiasedness).} By Theorem 3, the bias of GAE is bounded by:
\begin{equation}
|\text{Bias}(s_t, a_t; \lambda)| \leq \frac{1 + \gamma - 2\gamma\lambda}{1 - \gamma\lambda} \cdot \epsilon_{\max},
\end{equation}
where $\epsilon_{\max} = \max_s |V_\phi(s) - V^\pi(s)|$. Under perfect value function approximation ($V_\phi = V^\pi$), we have $\epsilon_{\max} = 0$, thus:
\begin{equation}
\mathbb{E}[\hat{A}_t^{\textup{GAE}} \mid s_t, a_t] = A^\pi(s_t, a_t).
\end{equation}

\textbf{Step 2 (Monotonic Improvement and Convergence).} By Lemma 3, PPU is an instance of Mirror Learning with drift function $\mathfrak{D}_{\pi}^{\text{PPU}}$ and trivial neighbourhood $\mathcal{N} = \Pi$. The drift function satisfies nonnegativity ($\mathfrak{D}_{\pi}^{\text{PPU}}(\bar{\pi}|s) \geq 0$ with equality when $\bar{\pi} = \pi$) and zero gradient ($\nabla_{\bar{\pi}} \mathfrak{D}_{\pi}^{\text{PPU}}(\bar{\pi}|s)|_{\bar{\pi}=\pi} = 0$). By the Fundamental Theorem of Mirror Learning~\citep{Mirror}, the policy sequence $\{\pi_k\}_{k=0}^{\infty}$ satisfies:
\begin{equation}
    J(\pi_{k+1}) \geq J(\pi_k), \quad \forall k \in \mathbb{N},
\end{equation}
and
\begin{equation}
    \lim_{k \to \infty} J(\pi_k) = J^*.
\end{equation}
\end{proof}

\section{Detailed Proofs for GRAE-PPU Convergence Analysis}
\label{app:grae_ppo_convergence}

This appendix provides detailed proofs for Theorem 7 regarding the failure of GRAE-PPU to guarantee convergence in general MDPs. The key insight is that in MDPs (as opposed to contextual bandits), GRAE introduces a structural bias that breaks the drift function properties required by the Mirror Learning framework~\citep{Mirror}.

\subsection{GRAE-PPU in the Mirror Learning Framework}

We first analyze how GRAE-PPU fits within the Mirror Learning framework, and show that the structural bias in GRAE leads to a malformed drift function.

\begin{lemma}{Structural Bias of GRAE in MDPs}{structural_bias}
\label{lem:structural_bias}
In general MDPs, GRAE introduces a structural bias $\Delta(s_t) = V(s_t) - V(s_0)$ that does not vanish as the number of samples increases:
\begin{equation}
\lim_{N \to \infty} \mathbb{E}[\hat{A}_{\textup{GRAE}}(s_t, a_t) - A_{\textup{true}}(s_t, a_t) \mid s_t, a_t] = V(s_t) - V(s_0) \neq 0.
\end{equation}
\end{lemma}

\begin{proof}
Consider a response $\tau^{(j)}$ containing state-action pair $(s_t, a_t)$ with reward $R^{(j)}$. In the large sample limit $N \to \infty$, we have $\bar{R} \to V(s_0)$ by the law of large numbers. The bias decomposes as:
\begin{align}
\Delta(s_t, a_t) &= \hat{A}_{\text{GRAE}}(s_t, a_t) - A_{\text{true}}(s_t, a_t) \\
&= (R^{(j)} - V(s_0)) - (Q(s_t, a_t) - V(s_t)) \\
&= \underbrace{(R^{(j)} - Q(s_t, a_t))}_{\epsilon_t} + \underbrace{(V(s_t) - V(s_0))}_{B_t}.
\end{align}

The term $\epsilon_t$ satisfies $\mathbb{E}[\epsilon_t \mid s_t, a_t] = 0$ by definition of $Q$, and vanishes under averaging. However, $B_t = V(s_t) - V(s_0)$ is deterministic given $s_t$ and independent of sample size. Thus:
\begin{equation}
\lim_{N \to \infty} \mathbb{E}[\Delta \mid s_t, a_t] = B_t \neq 0.
\end{equation}
\end{proof}

\begin{lemma}{GRAE-PPU Drift Function}{grae_ppo_drift}
\label{lem:grae_ppo_drift}
GRAE-PPU induces a drift function of the form:
\begin{equation}
\mathfrak{D}_{\pi}^{\textup{GRAE-PPU}}(\bar{\pi} | s_t) = \mathbb{E}_{a_t \sim \pi} \left[ \operatorname{ReLU}\left( (r(\bar{\pi}) - \operatorname{clip}(r(\bar{\pi}), 1 \pm \epsilon)) (A_{\textup{true}}(s_t, a_t) + \Delta(s_t)) \right) \right] - \Delta(s_t),
\label{eq:GRAE-PPO-drift}
\end{equation}
where $\Delta(s_t) = V(s_t) - V(s_0)$ is the structural bias from Lemma 4.
\end{lemma}

\begin{proof}
The GRAE-PPU objective substitutes the true advantage $A_{\text{true}}(s_t, a_t)$ with the biased estimator $\hat{A}_{\text{GRAE}}(s_t, a_t) = A_{\text{true}}(s_t, a_t) + \Delta(s_t)$:
\begin{equation}
\pi_{\text{new}} = \underset{\bar{\pi} \in \Pi}{\arg\max} \mathbb{E}_{s_t \sim \rho_{\pi}, a_t \sim \pi} \left[ \min \left( r(\bar{\pi}) \hat{A}_{\text{GRAE}}, \operatorname{clip}(r(\bar{\pi}), 1 \pm \epsilon) \hat{A}_{\text{GRAE}} \right) \right],
\end{equation}
where $r(\bar{\pi}) = \frac{\bar{\pi}(a_t|s_t)}{\pi(a_t|s_t)}$. To derive the drift function, we must express this objective in the Mirror Learning form: $\mathbb{E}_{a \sim \bar{\pi}}[A_{\text{true}}(s_t, a)] - \mathfrak{D}(\bar{\pi}|s_t)$. The key insight is that the Mirror Learning framework requires using the \emph{true} advantage function $A_{\text{true}}$, not the biased estimate.

Following the derivation in Lemma 3, the GRAE-PPU clipped objective can be rewritten as:
\begin{equation}
\begin{aligned}
&\mathbb{E}_{a_t \sim \pi} \left[ \min \left( r(\bar{\pi}) \hat{A}_{\text{GRAE}}, \operatorname{clip}(r(\bar{\pi}), 1 \pm \epsilon) \hat{A}_{\text{GRAE}} \right) \right] \\
&= \mathbb{E}_{a_t \sim \bar{\pi}} \left[\hat{A}_{\text{GRAE}}\right] - \mathbb{E}_{a_t \sim \pi} \left[ \operatorname{ReLU} \left( \left[ r(\bar{\pi}) - \operatorname{clip}(r(\bar{\pi}), 1 \pm \epsilon) \right] \hat{A}_{\text{GRAE}} \right) \right].
\end{aligned}
\end{equation}

Now we substitute $\hat{A}_{\text{GRAE}} = A_{\text{true}} + \Delta(s_t)$ and expand:
\begin{equation}
\begin{aligned}
&\mathbb{E}_{a_t \sim \bar{\pi}} \left[A_{\text{true}} + \Delta(s_t)\right] - \mathbb{E}_{a_t \sim \pi} \left[ \operatorname{ReLU} \left( \left[ r(\bar{\pi}) - \operatorname{clip}(r(\bar{\pi}), 1 \pm \epsilon) \right] (A_{\text{true}} + \Delta(s_t)) \right) \right] \\
&= \mathbb{E}_{a_t \sim \bar{\pi}} \left[A_{\text{true}}\right] + \Delta(s_t) - \mathbb{E}_{a_t \sim \pi} \left[ \operatorname{ReLU} \left( \left[ r(\bar{\pi}) - \operatorname{clip}(r(\bar{\pi}), 1 \pm \epsilon) \right] (A_{\text{true}} + \Delta(s_t)) \right) \right],
\end{aligned}
\end{equation}
where the second equality uses the fact that $\mathbb{E}_{a_t \sim \bar{\pi}}[\Delta(s_t)] = \Delta(s_t)$ since $\Delta(s_t)$ is a constant with respect to $a_t$.

To express this in Mirror Learning form $\mathbb{E}_{a \sim \bar{\pi}}[A_{\text{true}}] - \mathfrak{D}(\bar{\pi}|s_t)$, we identify the drift function as:
\begin{equation}
\begin{aligned}
\mathfrak{D}_{\pi}^{\text{GRAE-PPU}}(\bar{\pi} | s_t) &= \mathbb{E}_{a_t \sim \pi} \left[ \operatorname{ReLU} \left( \left[ r(\bar{\pi}) - \operatorname{clip}(r(\bar{\pi}), 1 \pm \epsilon) \right] (A_{\text{true}} + \Delta(s_t)) \right) \right] - \Delta(s_t).
\end{aligned}
\end{equation}

\end{proof}

\subsection{Convergence Failure Theorem}

\begin{theorem}{Convergence Failure of GRAE-PPU}{grae_ppo_convergence}
\label{thm:grae_ppo_convergence}
Consider a general MDP where there exist states $s_t$ with $V(s_t) \neq V(s_0)$. Let GRAE-PPU update the policy using the drift function from Lemma 5. Then:
\begin{enumerate}
    \item (\textbf{Drift zero-at-origin violation}) $\exists\, s_t: \mathfrak{D}_{\pi}^{\textup{GRAE-PPU}}(\pi | s_t) \neq 0$.
    \item (\textbf{Drift nonnegativity violation}) $\exists\, s_t, \bar{\pi}: \mathfrak{D}_{\pi}^{\textup{GRAE-PPU}}(\bar{\pi} | s_t) < 0$.
    \item (\textbf{Policy degradation}) $\exists\, \pi_{\textup{old}}: J(\pi_{\textup{new}}) < J(\pi_{\textup{old}})$.
\end{enumerate}
\end{theorem}

\begin{proof}
\textbf{1. Violation of Zero at Origin:}

Mirror Learning requires that $\mathfrak{D}_{\pi_{\text{old}}}^{\text{GRAE-PPU}}(\pi_{\text{old}} | s_t) = 0$ for all states $s_t$. Substituting $\bar{\pi} = \pi_{\text{old}}$ into Equation~\ref{eq:GRAE-PPO-drift}, we have $r(\pi_{\text{old}}) = 1$ and $\operatorname{clip}(r(\pi_{\text{old}}), 1 \pm \epsilon) = 1$. The ReLU term becomes:
\begin{equation}
\operatorname{ReLU}\left( (1 - 1) (A_{\text{true}}(s_t, a_t) + \Delta(s_t)) \right) = \operatorname{ReLU}(0) = 0.
\end{equation}
Thus:
\begin{equation}
\mathfrak{D}_{\pi_{\text{old}}}^{\text{GRAE-PPU}}(\pi_{\text{old}} | s_t) = 0 - \Delta(s_t) = -\Delta(s_t).
\end{equation}
Since $\Delta(s_t) = V(s_t) - V(s_0) \neq 0$ generally (by Lemma 4), the drift function is non-zero at the origin, breaking the fundamental distance metric property required by Mirror Learning.

\textbf{2. Violation of Non-negativity:}

Mirror Learning requires that $\mathfrak{D}_{\pi_{\text{old}}}^{\text{GRAE-PPU}}(\bar{\pi} | s_t) \geq 0$ for all states $s_t$, policies $\pi_{\text{old}}$, and $\bar{\pi} \in \Pi$. However, consider a scenario where:
\begin{itemize}
    \item $\Delta(s_t)$ is large and positive (e.g., when the current state value $V(s_t)$ is much higher than the initial value $V(s_0)$),
    \item The policy $\bar{\pi}$ is close to $\pi_{\text{old}}$ (i.e., $r(\bar{\pi}) \approx 1$), making the ReLU term small or zero.
\end{itemize}
In this case, $\mathfrak{D}_{\pi_{\text{old}}}^{\text{GRAE-PPU}}(\bar{\pi} | s_t) \approx -\Delta(s_t) < 0$.

\textbf{3. Failure of Monotonic Improvement:}

The negative drift term fundamentally breaks the Mirror Learning convergence proof. Recall that in the proof of monotonic improvement (Lemma 3.3 in~\citep{Mirror}), the key step establishes:
\begin{equation}
V_{\bar{\pi}}(s) - V_{\pi_{\text{old}}}(s) \geq \gamma \inf_{s'}[V_{\bar{\pi}}(s') - V_{\pi_{\text{old}}}(s')] + \frac{\nu(s)}{\beta(s)} \mathfrak{D}_{\pi_{\text{old}}}(\bar{\pi}|s).
\end{equation}
Taking infimum over $s$ and rearranging yields:
\begin{equation}
\inf_s[V_{\bar{\pi}}(s) - V_{\pi_{\text{old}}}(s)] \geq \frac{1}{1-\gamma} \inf_s\left[\frac{\nu(s)}{\beta(s)} \mathfrak{D}_{\pi_{\text{old}}}(\bar{\pi}|s)\right].
\end{equation}
When $\mathfrak{D}_{\pi_{\text{old}}}^{\text{GRAE-PPU}}(\bar{\pi}|s) \geq 0$, the right-hand side is non-negative, guaranteeing $V_{\bar{\pi}}(s) \geq V_{\pi_{\text{old}}}(s)$ for all $s$. However, when $\mathfrak{D}_{\pi_{\text{old}}}^{\text{GRAE-PPU}}(\bar{\pi}|s) < 0$, the right-hand side becomes negative, and the proof can only establish that the value difference is lower-bounded by some negative quantity---this provides no guarantee that $V_{\bar{\pi}}(s) \geq V_{\pi_{\text{old}}}(s)$. Thus, the violation of drift nonnegativity causes the proof chain to break down entirely, invalidating the monotonic improvement guarantee and potentially leading to policy degradation.

\textbf{4. Existence of Policy Degradation:}

We now provide a concrete example demonstrating how GRAE-PPU can lead to policy degradation. Consider a simple 2-state MDP where $V(s_0) = 0$ and $V(s_1) = 10$, yielding a structural bias $\Delta(s_1) = V(s_1) - V(s_0) = 10$. In state $s_1$, suppose there are two actions: $a_{\text{good}}$ with true advantage $A_{\text{true}}(s_1, a_{\text{good}}) = 2$ and $a_{\text{bad}}$ with true advantage $A_{\text{true}}(s_1, a_{\text{bad}}) = -5$.

By Lemma 4, GRAE introduces a state-dependent bias $\Delta(s_t) = V(s_t) - V(s_0)$ that applies uniformly to all actions at the same state. Therefore, the GRAE advantage estimates become:
\begin{align}
\hat{A}_{\text{GRAE}}(s_1, a_{\text{good}}) &= A_{\text{true}}(s_1, a_{\text{good}}) + \Delta(s_1) = 2 + 10 = 12, \\
\hat{A}_{\text{GRAE}}(s_1, a_{\text{bad}}) &= A_{\text{true}}(s_1, a_{\text{bad}}) + \Delta(s_1) = -5 + 10 = 5.
\end{align}

\textbf{Why standard policy gradients are unaffected.} Before analyzing PPU, we note an important subtlety: standard (vanilla) policy gradients are \emph{invariant} to state-dependent baselines (Theorem~\ref{thm:grae_bias_gradient}). For softmax policies, the total gradient sums over all actions, and since $\sum_a \nabla_\theta \pi_\theta(a|s) = 0$, constant shifts in advantages cancel out. Concretely, for the two-action case at $s_1$:
\begin{align}
g_{\text{true}} &\propto \nabla \log \pi_{\text{good}} \cdot (2) + \nabla \log \pi_{\text{bad}} \cdot (-5), \\
g_{\text{GRAE}} &\propto \nabla \log \pi_{\text{good}} \cdot (12) + \nabla \log \pi_{\text{bad}} \cdot (5).
\end{align}
Using the fact that $\nabla \log \pi_{\text{good}}$ and $\nabla \log \pi_{\text{bad}}$ point in roughly opposite directions for binary actions, both gradients yield the same direction (favoring $a_{\text{good}}$). Thus, GRAE does not cause degradation for vanilla policy gradients.

\textbf{Why PPU's clipping breaks baseline invariance.} However, the PPU clipped objective is \emph{non-linear} with respect to the advantage estimate:
\begin{equation}
L^{\text{CLIP}}(a) = \min\left(r(\theta)\hat{A}, \text{clip}(r(\theta), 1-\epsilon, 1+\epsilon)\hat{A}\right),
\end{equation}
where $r(\theta) = \pi_\theta(a|s)/\pi_{\text{old}}(a|s)$. This non-linearity fundamentally breaks baseline invariance because the clipping behavior depends on the \emph{sign} of $\hat{A}$. Let us analyze the PPU objective for $a_{\text{bad}}$ under both scenarios with $\epsilon = 0.2$:

\textbf{True scenario} ($A_{\text{true}}(s_1, a_{\text{bad}}) = -5$):
\begin{equation}
L^{\text{CLIP}}_{\text{true}}(a_{\text{bad}}) = \min\left(-5r, \text{clip}(r, 0.8, 1.2) \cdot (-5)\right) = \begin{cases}
-5r & \text{if } r \geq 0.8 \\
-4 & \text{if } r < 0.8
\end{cases}
\end{equation}
Maximizing this objective pushes $r \to 0.8$ (the lower bound), i.e., \emph{decreasing} $\pi(a_{\text{bad}})$ by up to 20\%.

\textbf{GRAE scenario} ($\hat{A}_{\text{GRAE}}(s_1, a_{\text{bad}}) = +5$):
\begin{equation}
L^{\text{CLIP}}_{\text{GRAE}}(a_{\text{bad}}) = \min\left(+5r, \text{clip}(r, 0.8, 1.2) \cdot (+5)\right) = \begin{cases}
+6 & \text{if } r > 1.2 \\
+5r & \text{if } r \leq 1.2
\end{cases}
\end{equation}
Maximizing this objective pushes $r \to 1.2$ (the upper bound), i.e., \emph{increasing} $\pi(a_{\text{bad}})$ by up to 20\%.

\textbf{The clipping directions are completely reversed.} The true objective wants to decrease $\pi(a_{\text{bad}})$ and clips at $1-\epsilon$; the GRAE objective wants to increase $\pi(a_{\text{bad}})$ and clips at $1+\epsilon$. Although GRAE preserves the relative ranking ($\hat{A}_{\text{GRAE}}(a_{\text{good}}) = 12 > 5 = \hat{A}_{\text{GRAE}}(a_{\text{bad}})$), the PPU update for $a_{\text{bad}}$ is fundamentally wrong:
\begin{enumerate}
    \item Under true advantages: PPU allows $\pi(a_{\text{bad}})$ to decrease from $\pi_{\text{old}}(a_{\text{bad}})$ to $0.8 \cdot \pi_{\text{old}}(a_{\text{bad}})$.
    \item Under GRAE advantages: PPU allows $\pi(a_{\text{bad}})$ to \emph{increase} from $\pi_{\text{old}}(a_{\text{bad}})$ to $1.2 \cdot \pi_{\text{old}}(a_{\text{bad}})$.
\end{enumerate}

Consequently, after the PPU update, we have:
\begin{equation}
\bar{\pi}(a_{\text{bad}} \mid s_1) \approx 1.2 \cdot \pi_{\text{old}}(a_{\text{bad}} \mid s_1) > \pi_{\text{old}}(a_{\text{bad}} \mid s_1),
\end{equation}
whereas the optimal update should yield $\bar{\pi}(a_{\text{bad}} \mid s_1) \approx 0.8 \cdot \pi_{\text{old}}(a_{\text{bad}} \mid s_1)$. Since $a_{\text{bad}}$ has true negative advantage ($-5$), increasing its probability degrades expected return:
\begin{equation}
J(\bar{\pi}) < J(\pi_{\text{old}}).
\end{equation}

Therefore, GRAE-PPU violates the fundamental properties required for Mirror Learning convergence guarantees---not because the gradient direction is wrong (it is correct for vanilla PG), but because PPU's non-linear clipping mechanism cannot tolerate the sign flip in advantage estimates, leading to trust region constraints being enforced in the completely wrong direction.
\end{proof}

\section{GRAE-PPU Convergence in Contextual Bandit Settings and Variance Normalization Issues}
\label{app:grae_bandit_gspo}

This appendix provides detailed proofs for Theorem 8 regarding GRAE unbiasedness and convergence in Contextual Bandit settings, and demonstrates why variance normalization operations (as used in GSPO) break convergence guarantees. The key insight is that in Contextual Bandits, the structural bias of GRAE vanishes because there is only one state ($s_0$), making GRAE an unbiased estimator.

\subsection{GRAE-PPU as an Instance of Mirror Learning in Contextual Bandits}

We first establish that GRAE-PPU fits within the Mirror Learning framework when applied to Contextual Bandits.

\begin{lemma}{GRAE Unbiasedness in Contextual Bandits}{grae_bandit_unbiased}
\label{lem:grae_bandit_unbiased}
In the Contextual Bandit setting, GRAE provides unbiased advantage estimates:
\begin{equation}
\mathbb{E}[\hat{A}_{\textup{GRAE}}(\boldsymbol{s}, \boldsymbol{a}) \mid \boldsymbol{s}, \boldsymbol{a}] = A_{\boldsymbol{\pi}}^{\textup{Bandit}}(\boldsymbol{s}, \boldsymbol{a}).
\end{equation}
\end{lemma}

\begin{proof}
In the Contextual Bandit setting, GRAE estimates the advantage as:
\begin{equation}
\hat{A}_{\text{GRAE}}(\boldsymbol{s}, \boldsymbol{a}) = \boldsymbol{r}(\boldsymbol{s}, \boldsymbol{a}) - \bar{R},
\end{equation}
where $\bar{R} = \frac{1}{N}\sum_{i=1}^{N} \boldsymbol{r}(\boldsymbol{s}, \boldsymbol{a}_i)$ is the group mean reward for responses sampled from the same query $\boldsymbol{s}$.

As $N \to \infty$, by the law of large numbers:
\begin{equation}
\bar{R} \xrightarrow{p} \mathbb{E}_{\boldsymbol{a} \sim \boldsymbol{\pi}(\cdot|\boldsymbol{s})}[\boldsymbol{r}(\boldsymbol{s},\boldsymbol{a})] = V_{\boldsymbol{\pi}}^{\text{Bandit}}(\boldsymbol{s}).
\end{equation}

The true state-action value function in the Contextual Bandit setting is:
\begin{equation}
Q_{\boldsymbol{\pi}}^{\text{Bandit}}(\boldsymbol{s}, \boldsymbol{a}) = \mathbb{E}[\boldsymbol{r}(\boldsymbol{s}, \boldsymbol{a}) \mid \boldsymbol{s}, \boldsymbol{a}] = \boldsymbol{r}(\boldsymbol{s}, \boldsymbol{a}),
\end{equation}
since rewards are deterministic given the state-action pair. The true value function is:
\begin{equation}
V_{\boldsymbol{\pi}}^{\text{Bandit}}(\boldsymbol{s}) = \mathbb{E}_{\boldsymbol{a} \sim \boldsymbol{\pi}(\cdot|\boldsymbol{s})}[Q_{\boldsymbol{\pi}}^{\text{Bandit}}(\boldsymbol{s}, \boldsymbol{a})] = \mathbb{E}_{\boldsymbol{a} \sim \boldsymbol{\pi}(\cdot|\boldsymbol{s})}[\boldsymbol{r}(\boldsymbol{s}, \boldsymbol{a})].
\end{equation}

The true advantage function is:
\begin{equation}
A_{\boldsymbol{\pi}}^{\text{Bandit}}(\boldsymbol{s},\boldsymbol{a}) = Q_{\boldsymbol{\pi}}^{\text{Bandit}}(\boldsymbol{s}, \boldsymbol{a}) - V_{\boldsymbol{\pi}}^{\text{Bandit}}(\boldsymbol{s}) = \boldsymbol{r}(\boldsymbol{s},\boldsymbol{a}) - V_{\boldsymbol{\pi}}^{\text{Bandit}}(\boldsymbol{s}).
\end{equation}

Therefore:
\begin{align}
\mathbb{E}[\hat{A}_{\text{GRAE}}(\boldsymbol{s}, \boldsymbol{a}) \mid \boldsymbol{s}, \boldsymbol{a}] &= \boldsymbol{r}(\boldsymbol{s}, \boldsymbol{a}) - V_{\boldsymbol{\pi}}^{\text{Bandit}}(\boldsymbol{s}) \\
&= Q_{\boldsymbol{\pi}}^{\text{Bandit}}(\boldsymbol{s}, \boldsymbol{a}) - V_{\boldsymbol{\pi}}^{\text{Bandit}}(\boldsymbol{s}) \\
&= A_{\boldsymbol{\pi}}^{\text{Bandit}}(\boldsymbol{s},\boldsymbol{a}).
\end{align}
\end{proof}

\begin{lemma}{GRAE-PPU as Mirror Learning in Contextual Bandits}{grae_ppo_bandit_mirror}
\label{lem:grae_ppo_bandit_mirror}
In the Contextual Bandit setting, GRAE-PPU is an instance of Mirror Learning with:
\begin{enumerate}
    \item \textbf{Drift function:} $\mathfrak{D}_{\boldsymbol{\pi}}^{\textup{GRAE-PPU}}(\bar{\boldsymbol{\pi}} | \boldsymbol{s}) = \mathbb{E}_{\boldsymbol{a} \sim \boldsymbol{\pi}} \left[ \operatorname{ReLU} \left( \left[ r(\bar{\boldsymbol{\pi}}) - \operatorname{clip}(r(\bar{\boldsymbol{\pi}}), 1 \pm \epsilon) \right] A_{\boldsymbol{\pi}}^{\textup{Bandit}}(\boldsymbol{s}, \boldsymbol{a}) \right) \right]$,
    which is identical to the standard PPU drift function.
    \item \textbf{Neighbourhood operator:} $\mathcal{N} = \boldsymbol{\Pi}$ (trivial neighbourhood).
    \item \textbf{Sampling distribution:} $\beta_{\boldsymbol{\pi}} = \rho_{\boldsymbol{\pi}}$.
\end{enumerate}
\end{lemma}

\begin{proof}
By Lemma 6, GRAE provides unbiased advantage estimates in the Contextual Bandit setting. Therefore, we can substitute the true advantage function $A_{\boldsymbol{\pi}}^{\text{Bandit}}(\boldsymbol{s}, \boldsymbol{a})$ for $\hat{A}_{\text{GRAE}}(\boldsymbol{s}, \boldsymbol{a})$ in the GRAE-PPU objective:
\begin{equation}
\boldsymbol{\pi}_{\text{new}} = \underset{\bar{\boldsymbol{\pi}} \in \boldsymbol{\Pi}}{\arg\max} \mathbb{E}_{\boldsymbol{s} \sim \rho_{\boldsymbol{\pi}}, \boldsymbol{a} \sim \boldsymbol{\pi}} \left[ \min \left( r(\bar{\boldsymbol{\pi}}) A_{\boldsymbol{\pi}}^{\text{Bandit}}(\boldsymbol{s}, \boldsymbol{a}), \operatorname{clip}(r(\bar{\boldsymbol{\pi}}), 1 \pm \epsilon) A_{\boldsymbol{\pi}}^{\text{Bandit}}(\boldsymbol{s}, \boldsymbol{a}) \right) \right],
\end{equation}
where $r(\bar{\boldsymbol{\pi}}) = \frac{\bar{\boldsymbol{\pi}}(\boldsymbol{a}|\boldsymbol{s})}{\boldsymbol{\pi}(\boldsymbol{a}|\boldsymbol{s})}$.

Following the same derivation as in Lemma 3, the drift function is:
\begin{equation}
\mathfrak{D}_{\boldsymbol{\pi}}^{\text{GRAE-PPU}}(\bar{\boldsymbol{\pi}} | \boldsymbol{s}) = \mathbb{E}_{\boldsymbol{a} \sim \boldsymbol{\pi}} \left[ \operatorname{ReLU} \left( \left[ r(\bar{\boldsymbol{\pi}}) - \operatorname{clip}(r(\bar{\boldsymbol{\pi}}), 1 \pm \epsilon) \right] A_{\boldsymbol{\pi}}^{\text{Bandit}}(\boldsymbol{s}, \boldsymbol{a}) \right) \right],
\end{equation}
which is identical to the standard PPU drift function (Equation~\ref{eq:PPO-drift-detailed}). Since this drift function satisfies nonnegativity and zero gradient properties (as established in Lemma 3), GRAE-PPU is a valid instance of Mirror Learning in the Contextual Bandit setting.
\end{proof}

\subsection{Convergence Theorem}

\begin{theorem}{Convergence of GRAE-PPU in Contextual Bandits}{grae_ppo_bandit}
\label{thm:grae_ppo_bandit}
Consider a Contextual Bandit with $|\boldsymbol{r}(\boldsymbol{s},\boldsymbol{a})| \leq R_{\max}$ for all $(\boldsymbol{s},\boldsymbol{a})$. Let $\boldsymbol{\pi}_\theta$ be a parameterized policy over a compact parameter space $\Theta$. Suppose the policy is updated by GRAE-PPU. Then:
\begin{enumerate}
    \item (\textbf{Advantage unbiasedness}) $\mathbb{E}[\hat{A}_{\textup{GRAE}}(\boldsymbol{s}, \boldsymbol{a}) \mid \boldsymbol{s}, \boldsymbol{a}] = A_{\boldsymbol{\pi}}^{\textup{Bandit}}(\boldsymbol{s}, \boldsymbol{a})$.
    \item (\textbf{Monotonic improvement}) $J(\boldsymbol{\pi}_{k+1}) \geq J(\boldsymbol{\pi}_k), \quad \forall k \in \mathbb{N}$.
    \item (\textbf{Convergence}) $\lim_{k \to \infty} J(\boldsymbol{\pi}_k) = J^*$.
\end{enumerate}
\end{theorem}

\begin{proof}
\textbf{Step 1 (Advantage Unbiasedness).} By Lemma 6, GRAE provides unbiased advantage estimates in the Contextual Bandit setting:
\begin{equation}
\mathbb{E}[\hat{A}_{\text{GRAE}}(\boldsymbol{s}, \boldsymbol{a}) \mid \boldsymbol{s}, \boldsymbol{a}] = A_{\boldsymbol{\pi}}^{\text{Bandit}}(\boldsymbol{s}, \boldsymbol{a}).
\end{equation}

\textbf{Step 2 (Monotonic Improvement and Convergence).} By Lemma 7, GRAE-PPU is an instance of Mirror Learning with drift function identical to the standard PPU drift function. The drift function satisfies nonnegativity ($\mathfrak{D}_{\boldsymbol{\pi}}^{\text{GRAE-PPU}}(\bar{\boldsymbol{\pi}}|\boldsymbol{s}) \geq 0$ with equality when $\bar{\boldsymbol{\pi}} = \boldsymbol{\pi}$) and zero gradient ($\nabla_{\bar{\boldsymbol{\pi}}} \mathfrak{D}_{\boldsymbol{\pi}}^{\text{GRAE-PPU}}(\bar{\boldsymbol{\pi}}|\boldsymbol{s})|_{\bar{\boldsymbol{\pi}}=\boldsymbol{\pi}} = 0$). By the Fundamental Theorem of Mirror Learning~\citep{Mirror}, the policy sequence $\{\boldsymbol{\pi}_k\}_{k=0}^{\infty}$ satisfies:
\begin{equation}
    J(\boldsymbol{\pi}_{k+1}) \geq J(\boldsymbol{\pi}_k), \quad \forall k \in \mathbb{N},
\end{equation}
and
\begin{equation}
    \lim_{k \to \infty} J(\boldsymbol{\pi}_k) = J^*.
\end{equation}
\end{proof}

\subsection{Variance Normalization Breaking Convergence: The Case of GSPO}
\label{app:variance_normalization}

This subsection demonstrates why variance normalization operations, as used in GSPO (and GRPO), break convergence guarantees. We use GSPO as a concrete example to illustrate the fundamental issues with variance normalization. It is worth noting that~\citep{hureinforce++} has shown that the original GRPO's advantage estimation is biased when variance normalization is applied.

\subsubsection{GSPO Drift Function Derivation and Properties}

In single-turn scenarios, GSPO applies variance normalization to the advantage estimates. Specifically, GSPO's advantage estimation can be written as:
\begin{equation}
A_{\boldsymbol{\pi}_{\text{old}}}^{\text{GSPO}}(\boldsymbol{s}, \boldsymbol{a}) = \frac{A_{\boldsymbol{\pi}}^{\text{Bandit}}(\boldsymbol{s},\boldsymbol{a})}{\delta(\boldsymbol{s})},
\end{equation}
where $A_{\boldsymbol{\pi}}^{\text{Bandit}}(\boldsymbol{s},\boldsymbol{a}) = Q_{\boldsymbol{\pi}}^{\text{Bandit}}(\boldsymbol{s},\boldsymbol{a}) - V_{\boldsymbol{\pi}}^{\text{Bandit}}(\boldsymbol{s})$ is the true Bandit advantage function, and $\delta(\boldsymbol{s})$ is the standard deviation of rewards within the group of responses sampled from the same query $\boldsymbol{s}$. 

The key observation is that $\delta(\boldsymbol{s})$ depends only on the state $\boldsymbol{s}$ (the query) and not on the specific action $\boldsymbol{a}$, since it is computed from the empirical distribution of rewards across all responses to the same query. This state-dependent normalization factor $\delta(\boldsymbol{s})$ plays a crucial role in the convergence analysis, as we will demonstrate that variance normalization breaks convergence guarantees unless $\delta(\boldsymbol{s}) \equiv 1$ identically for all states.

Following the same derivation procedure as in Section~\ref{app:gae_ppo_global}, we can identify the drift function as:
\begin{equation}
\mathfrak{D}_{\boldsymbol{\pi}_{\text{old}}}^{\text{GSPO}}(\bar{\boldsymbol{\pi}} | \boldsymbol{s}) = \mathbb{E}_{\boldsymbol{a} \sim \boldsymbol{\pi}_{\text{old}}} \left[ r(\bar{\boldsymbol{\pi}}) A_{\boldsymbol{\pi}_{\text{old}}}^{\text{Bandit}}(\boldsymbol{s}, \boldsymbol{a}) - \min \left( r(\bar{\boldsymbol{\pi}}) \frac{A_{\boldsymbol{\pi}_{\text{old}}}^{\text{Bandit}}(\boldsymbol{s}, \boldsymbol{a})}{\delta(\boldsymbol{s})}, \operatorname{clip}(r(\bar{\boldsymbol{\pi}}), 1 \pm \epsilon) \frac{A_{\boldsymbol{\pi}_{\text{old}}}^{\text{Bandit}}(\boldsymbol{s}, \boldsymbol{a})}{\delta(\boldsymbol{s})} \right) \right].
\label{eq:GSPO-drift-detailed}
\end{equation}

\begin{theorem}{GSPO Drift Function Properties}{gspo_drift_properties}
\label{thm:gspo_drift_properties}
The GSPO drift function $\mathfrak{D}_{\boldsymbol{\pi}_{\text{old}}}^{\text{GSPO}}(\bar{\boldsymbol{\pi}} | \boldsymbol{s})$ satisfies both the nonnegativity property and the zero gradient property required for mirror learning convergence if and only if $\delta(\boldsymbol{s}) \equiv 1$.
\end{theorem}

To prove this theorem, we establish two auxiliary lemmas and then combine them. The proof proceeds in three steps: (1) establish Lemma 8 showing that $\delta(\boldsymbol{s}) \equiv 1$ is necessary and sufficient for nonnegativity, (2) establish Lemma 9 showing that $\delta(\boldsymbol{s}) \equiv 1$ implies zero gradient, and (3) combine these lemmas to prove the main theorem.

\textbf{Step (1): Nonnegativity Analysis.}

To analyze the drift function, we define the integrand:
\begin{equation}
g(A, r) = rA - \min\left(\frac{rA}{\delta(\boldsymbol{s})}, \frac{\operatorname{clip}(r, 1 \pm \epsilon) A}{\delta(\boldsymbol{s})}\right),
\end{equation}
where $r = r(\bar{\boldsymbol{\pi}}) = \frac{\bar{\boldsymbol{\pi}}(\boldsymbol{a}|\boldsymbol{s})}{\boldsymbol{\pi}_{\text{old}}(\boldsymbol{a}|\boldsymbol{s})} > 0$, and $A = A_{\boldsymbol{\pi}_{\text{old}}}^{\text{Bandit}}(\boldsymbol{s}, \boldsymbol{a})$. The drift function is $\mathfrak{D}_{\boldsymbol{\pi}_{\text{old}}}^{\text{GSPO}}(\bar{\boldsymbol{\pi}} | \boldsymbol{s}) = \mathbb{E}_{\boldsymbol{a} \sim \boldsymbol{\pi}_{\text{old}}} [g(A, r)]$.

We now state and prove the first auxiliary lemma.

\begin{lemma}{Necessary and Sufficient Condition for Nonnegativity}{gspo_nonnegativity}
\label{lem:gspo_nonnegativity}
$\delta(\boldsymbol{s}) \equiv 1$ if and only if $\mathfrak{D}_{\boldsymbol{\pi}_{\text{old}}}^{\text{GSPO}}(\bar{\boldsymbol{\pi}} | \boldsymbol{s}) \geq 0$ for all states $\boldsymbol{s}$, policies $\boldsymbol{\pi}_{\text{old}}$, and $\bar{\boldsymbol{\pi}} \in \boldsymbol{\Pi}$.
\end{lemma}

\begin{proof}
We analyze $g(A, r)$ by cases based on the sign of $A$ and the relationship between $r$ and $[1-\epsilon, 1+\epsilon]$.

\textbf{Case 1: $A \geq 0$}

For $r \in [1-\epsilon, 1+\epsilon]$: $g(A, r) = rA(1 - 1/\delta(\boldsymbol{s}))$, requiring $\delta(\boldsymbol{s}) \geq 1$.

For $r > 1+\epsilon$: $g(A, r) = A(r - (1+\epsilon)/\delta(\boldsymbol{s}))$. The critical constraint comes from $r \in [1-\epsilon, 1+\epsilon]$.

For $r < 1-\epsilon$: $g(A, r) = rA(1 - 1/\delta(\boldsymbol{s}))$, requiring $\delta(\boldsymbol{s}) \geq 1$.

\textbf{Case 2: $A < 0$}

For $r \in [1-\epsilon, 1+\epsilon]$: $g(A, r) = rA(1 - 1/\delta(\boldsymbol{s}))$, requiring $\delta(\boldsymbol{s}) \leq 1$.

For $r > 1+\epsilon$: $g(A, r) = rA(1 - 1/\delta(\boldsymbol{s}))$, requiring $\delta(\boldsymbol{s}) \leq 1$.

For $r < 1-\epsilon$: $g(A, r) = A(r - (1-\epsilon)/\delta(\boldsymbol{s}))$. As $r \to (1-\epsilon)^-$, this requires $\delta(\boldsymbol{s}) \leq 1$.

\textbf{Necessity:} Case 1 requires $\delta(\boldsymbol{s}) \geq 1$; Case 2 requires $\delta(\boldsymbol{s}) \leq 1$. Therefore, $\delta(\boldsymbol{s}) = 1$ is necessary.

\textbf{Sufficiency:} When $\delta(\boldsymbol{s}) = 1$, the drift function reduces to the PPU form, which satisfies nonnegativity.
\end{proof}

\textbf{Step (2): Zero Gradient Analysis.}

We now state and prove the second auxiliary lemma.

\begin{lemma}{Zero Gradient Property}{gspo_zero_gradient}
\label{lem:gspo_zero_gradient}
If $\delta(\boldsymbol{s}) = 1$, then $\nabla_{\bar{\boldsymbol{\pi}}(\cdot|\boldsymbol{s})} \mathfrak{D}_{\boldsymbol{\pi}_{\text{old}}}^{\text{GSPO}}(\bar{\boldsymbol{\pi}}|\boldsymbol{s}) \big\vert_{\bar{\boldsymbol{\pi}}=\boldsymbol{\pi}_{\text{old}}} = 0$.
\end{lemma}

\begin{proof}
When $\bar{\boldsymbol{\pi}} = \boldsymbol{\pi}_{\text{old}}$, we have $r = 1$. With $\delta(\boldsymbol{s}) = 1$, the drift function becomes:
\begin{equation}
\mathfrak{D}_{\boldsymbol{\pi}_{\text{old}}}^{\text{GSPO}}(\boldsymbol{\pi}_{\text{old}} | \boldsymbol{s}) = \mathbb{E}_{\boldsymbol{a} \sim \boldsymbol{\pi}_{\text{old}}} \left[ A_{\boldsymbol{\pi}_{\text{old}}}^{\text{Bandit}}(\boldsymbol{s}, \boldsymbol{a}) - A_{\boldsymbol{\pi}_{\text{old}}}^{\text{Bandit}}(\boldsymbol{s}, \boldsymbol{a}) \right] = 0,
\end{equation}
and the gradient is zero.
\end{proof}

\textbf{Step (3): Combining the Lemmas.}

We are now ready to prove the main theorem by combining the two lemmas.

\begin{proof}
By Lemma 8, we have $\delta(\boldsymbol{s}) \equiv 1 \Leftrightarrow$ Nonnegativity. By Lemma 9, we have $\delta(\boldsymbol{s}) \equiv 1 \Rightarrow$ Zero Gradient.

\textbf{Forward direction:} $\delta(\boldsymbol{s}) \equiv 1 \Rightarrow$ (Nonnegativity $\land$ Zero Gradient) follows directly from the two lemmas.

\textbf{Backward direction:} If both properties hold, then Nonnegativity holds in particular. By Lemma 8, this implies $\delta(\boldsymbol{s}) \equiv 1$.

Therefore, $\delta(\boldsymbol{s}) \equiv 1 \Leftrightarrow$ (Nonnegativity $\land$ Zero Gradient).
\end{proof}

\subsubsection{Why Variance Normalization Breaks Convergence}

As established in Theorem 9, variance normalization breaks convergence guarantees unless $\delta(\boldsymbol{s}) \equiv 1$ identically for all states. However, in practice, $\delta(\boldsymbol{s})$ is computed from empirical reward distributions and varies across different queries based on their reward distributions. This variation leads to violations of the fundamental properties required for monotonic improvement guarantees.

The key issue is that when $\delta(\boldsymbol{s}) \neq 1$, the drift function can become negative for certain combinations of advantages and policy ratios, violating the nonnegativity property required for Mirror Learning convergence. Recall that in the proof of monotonic improvement (Lemma 3.3 in~\citep{Mirror}), the critical inequality is:
\begin{equation}
\inf_s[V_{\bar{\pi}}(s) - V_{\pi_{\text{old}}}(s)] \geq \frac{1}{1-\gamma} \inf_s\left[\frac{\nu(s)}{\beta(s)} \mathfrak{D}_{\pi_{\text{old}}}(\bar{\pi}|s)\right].
\end{equation}
When $\mathfrak{D}_{\pi_{\text{old}}}^{\text{GSPO}}(\bar{\pi}|s) < 0$, the right-hand side becomes negative, and the proof can only establish that the value difference is lower-bounded by some negative quantity. This means the proof chain breaks down entirely---it no longer guarantees that $V_{\bar{\pi}}(s) \geq V_{\pi_{\text{old}}}(s)$ for all states. This demonstrates why variance normalization operations like those used in GSPO fundamentally break the convergence guarantees provided by Mirror Learning, even in the simplified Contextual Bandit setting.

\subsubsection{Batch-Level Normalization Preserves Convergence}
\label{app:batch_normalization}

In contrast to group-level variance normalization (as used in GSPO), batch-level normalization applies a single normalization to all advantage estimates within a batch. We now show that batch-level normalization preserves convergence guarantees because it does not change the argmax of the optimization problem, and therefore the drift functional remains completely unchanged.

Consider a batch $\mathcal{B} = \{(\boldsymbol{s}_0^{(i)}, \boldsymbol{a}^{1:T,(i)}, \hat{A}^{(i)})\}_{i=1}^{B}$ of $B$ advantage estimates. Batch-level normalization computes normalized advantages as:
\begin{equation}
\tilde{A}^{(i)} = \frac{\hat{A}^{(i)} - \mu_{\mathcal{B}}}{\sigma_{\mathcal{B}}},
\end{equation}
where $\mu_{\mathcal{B}} = \frac{1}{B}\sum_{j=1}^{B}\hat{A}^{(j)}$ is the batch mean and $\sigma_{\mathcal{B}} = \sqrt{\frac{1}{B}\sum_{j=1}^{B}(\hat{A}^{(j)} - \mu_{\mathcal{B}})^2}$ is the batch standard deviation. During a single policy update iteration $k$, the batch $\mathcal{B}$ is fixed and sampled from the current policy $\hat{\boldsymbol{\pi}}_k$. Therefore, both $\mu_{\mathcal{B}}$ and $\sigma_{\mathcal{B}}$ are deterministic constants that do not depend on the candidate policy $\bar{\boldsymbol{\pi}}$ being optimized.

The PPU policy update with batch-normalized advantages solves:
\begin{equation}
\bar{\boldsymbol{\pi}} = \argmax_{\boldsymbol{\pi}} \mathbb{E}_{(\boldsymbol{s}, \boldsymbol{a}) \sim \mathcal{B}}\left[\min\left(r(\boldsymbol{\pi}) \tilde{A}, \operatorname{clip}(r(\boldsymbol{\pi}), 1 \pm \epsilon) \tilde{A}\right)\right].
\end{equation}
Substituting the definition of $\tilde{A}$, the objective function becomes $\frac{1}{\sigma_{\mathcal{B}}}$ times the original objective minus a constant term $\frac{\mu_{\mathcal{B}}}{\sigma_{\mathcal{B}}}$. Since the argmax is invariant to both (i) adding/subtracting constants and (ii) multiplying by positive constants, we have:
\begin{equation}
\argmax_{\boldsymbol{\pi}} \mathbb{E}\left[\min\left(r(\boldsymbol{\pi}) \tilde{A}, \operatorname{clip}(r(\boldsymbol{\pi})) \tilde{A}\right)\right] = \argmax_{\boldsymbol{\pi}} \mathbb{E}\left[\min\left(r(\boldsymbol{\pi}) \hat{A}, \operatorname{clip}(r(\boldsymbol{\pi})) \hat{A}\right)\right].
\end{equation}
Therefore, batch-level normalization yields the same optimal policy as no normalization, and the drift functional $\mathfrak{D}_{\pi}^{\text{PPU}}(\bar{\pi}|s)$ remains completely unchanged.

This conclusion extends straightforwardly to the multi-agent HAML framework: since each agent's sequential update also takes the form of an argmax over an objective function, the same invariance properties hold, and the heterogeneous-agent drift functional remains unchanged under batch-level normalization.

The key distinction from group-level normalization (as in GSPO) is that batch-level normalization applies constants ($\mu_{\mathcal{B}}$ and $\sigma_{\mathcal{B}}$) that are identical for all samples in the batch, whereas group-level normalization applies state-dependent factors $\delta(\boldsymbol{s})$ that vary across queries, breaking the drift functional properties as established in Theorem 9.